\documentclass{article}
\usepackage{arxiv,times}

\usepackage{amsmath,amsfonts,bm}

\def\eqref#1{equation~\ref{#1}}

\def\1{\bm{1}}

\DeclareMathAlphabet{\mathsfit}{\encodingdefault}{\sfdefault}{m}{sl}
\SetMathAlphabet{\mathsfit}{bold}{\encodingdefault}{\sfdefault}{bx}{n}

\usepackage{hyperref}
\usepackage{url}
\usepackage{graphicx} 
\usepackage[dvipsnames]{xcolor}
\usepackage{wrapfig}
\usepackage{subcaption}
\usepackage{amsthm}
\usepackage{mathastext}
\usepackage{enumitem}
\usepackage{amsmath}
\usepackage{thm-restate}

\usepackage{algorithm}
\usepackage[noend]{algpseudocode}
\algrenewcommand\algorithmicrequire{\textbf{Input:}}
\algrenewcommand\algorithmicensure{\textbf{Output:}}

\newtheorem{definition}{Definition}
\newtheorem{property}{Property}

\newcommand{\green}[1]{\textcolor{ForestGreen}{#1}}

\newcommand{\blue}[1]{\textcolor{Orange}{#1}}

\newcommand{\dblue}[1]{\textbf{\textcolor{Maroon}{#1}}}

\newcommand{\gray}[1]{\textbf{\textcolor{Gray}{#1}}}

\title{Interpretable Probability Estimation with LLMs via Shapley Reconstruction}

\date{} 

\author{
Yang Nan, Qihao Wen, Jiahao Wang, Pengfei He, Ravi Tandon, Yong Ge, Han Xu.\\
\texttt{\{yangnan, qihaowen, jiahaow, tandonr,yongge, xuhan2}\}@arizona.edu\\
\texttt{hepengf1@msu.edu}
}

\begin{document}

\maketitle

\begin{abstract}

Large Language Models (LLMs) demonstrate potential to estimate the probability of uncertain events, by leveraging their extensive knowledge and reasoning capabilities. This ability can be applied to support intelligent decision-making across diverse fields, such as financial forecasting and preventive healthcare. However, directly prompting LLMs for probability estimation faces significant challenges: their outputs are often noisy, and the underlying predicting process is opaque. In this paper, we propose \textbf{PRISM: Probability Reconstruction via Shapley Measures}, a framework that brings transparency and precision to LLM-based probability estimation. PRISM decomposes an LLM’s prediction by quantifying the marginal contribution of each input factor using Shapley values. These factor-level contributions are then aggregated to reconstruct a calibrated final estimate. In our experiments, we demonstrate PRISM improves predictive accuracy over direct prompting and other baselines, across multiple domains including finance, healthcare, and agriculture. Beyond performance, PRISM provides a transparent prediction pipeline: our case studies visualize how individual factors shape the final estimate, helping build trust in LLM-based decision support systems. 
\end{abstract}

\vspace{-0.3cm}
\section{Introduction}
\vspace{-0.2cm}

Estimating the probability of uncertain events~\citep{berger2013statistical, winkler2019probability} is a critical task for intelligent decision makings in various domains such as financial investment~\citep{lathief2024quantifying}, healthcare~\citep{rajkomar2019machine}, and emergency management~\citep{rostami2025hierarchical}. However, in many real-world scenarios, either high-quality datasets are unavailable or mature Machine Learning (ML) techniques are lacking. Besides, there could also appear temporarily unexpected factors which are not considered when people building datasets~\citep{chowdhury2021covid}. As an example, in business analytics~\citep{raghupathi2021contemporary}, people may encounter a variety of diverse estimation tasks, such as determining whether the price, production or market demand of a certain product will increase or not~\citep{fildes2022retail}. It could be difficult for them to collect sufficient task-specific data and build reliable predictive models promptly. In this scenario, the applicability of traditional ML approaches are severely limited.  

As a potential solution, Large Language Models (LLMs) offer a promising alternative to address this challenge \citep{feng2024bird, sui2024table, chung2024large}. They incorporate extensive world knowledge and exhibit strong reasoning capabilities \citep{wei2022chain}, making them particularly valuable in data or model scarce settings. However, directly prompting LLMs to probability estimation still faces tremendous challenges: (i) LLMs typically produce noisy probability estimates that lack accuracy and stability. As an instance, according to the study~\citep{nafar2025extracting}, when LLMs are asked to predict the same event in different forms, i.e., “whether it will happen” and “whether it will not happen”, LLMs could provide conflicting answers. (ii) The ``black-box'' generative process provides little transparency regarding how each individual factor contributes to the final prediction, making the results difficult to interpret. As illustrated in Figure~\ref{fig:introduction}, when asked to predict the likelihood of a person having a certain disease, LLMs will not explicitly show how much weight each factor contributes to final prediction, but outputs a single score (with a partial explanation). It makes the final outcome difficult to interpret and less trustworthy.

To overcome the challenges, we propose a novel framework \textbf{Probability Reconstruction via
Shapley Measures (PRISM)}, inspired by Shapley Value for ML explanation~\citep{lundberg2017unified}. Specifically, we decompose the estimation task into quantifying the marginal contribution of each factor independently. 
As illustrated in  Figure~\ref{fig:introduction}, consider the task of predicting whether a person will experience a stroke. For a given factor such as ``Age=$79$'', we compare the LLM’s estimated probability when the model has access to this factor versus when it does not have access to it, paired with a random subset $S$ of the remaining factors (see Section~\ref{sec:method} for details). The \textit{marginal contribution of a factor}, referred to as its \textit{Shapley value}, is computed by averaging over different subsets $S$. These values capture both the direction and the intensity of the features' contribution to the prediction outcome. However, unlike traditional Shapley methods which solely focus on model output attribution~\citep{lundberg2017unified}, PRISM goes further: we reconstruct a new probability estimate by aggregating the factor contributions. For example, in Figure~\ref{fig:introduction}, the aggregated contributions yield a probability of 0.61, and we can clearly interpret this estimate originates from key risk-increasing factors such as advanced age and having hypertension. 
This process can make the prediction process transparent and allow human users to diagnose and interpret the model’s reasoning.

In our experiments, we validate the effectiveness of our proposed method under various settings. In Section~\ref{sec:exp1}, we compare PRISM with other LLM-based probability estimation methods, on binary classification tasks under benchmark tabular datasets (such as Adult Census Income~\citep{uci_adult_1996}, Heart Disease Prediction~\citep{chicco2020machine}, Stroke Prediction~\citep{stroke_kaggle} and Loan Default Prediction~\citep{lendingclub_kaggle}). We find that PRISM demonstrates higher prediction performance than directly LLM prompting and other representative baselines.  Furthermore, in Section~\ref{sec:exp2}, we evaluate PRISM on real-world prediction tasks involving textual and numerical inputs. For instance, we predict whether the a patient will be re-admitted to hospital, based on the doctor notes, price of an agricultural commodity will rise in the next year using the previous year's annual report, and we forecast the outcomes of English football matches based on pre-match reports. Together, these studies demonstrate PRISM’s ability to handle heterogeneous, context-rich factors—well beyond simple tabular data. Overall, the \textit{experiments highlight PRISM’s potential as a practical and reliable tool} for a wide range of real-world prediction problems.

\begin{figure}
    \centering
    \includegraphics[width=0.7\linewidth, trim={0cm, 3.2cm, 0cm, 0cm}, clip]{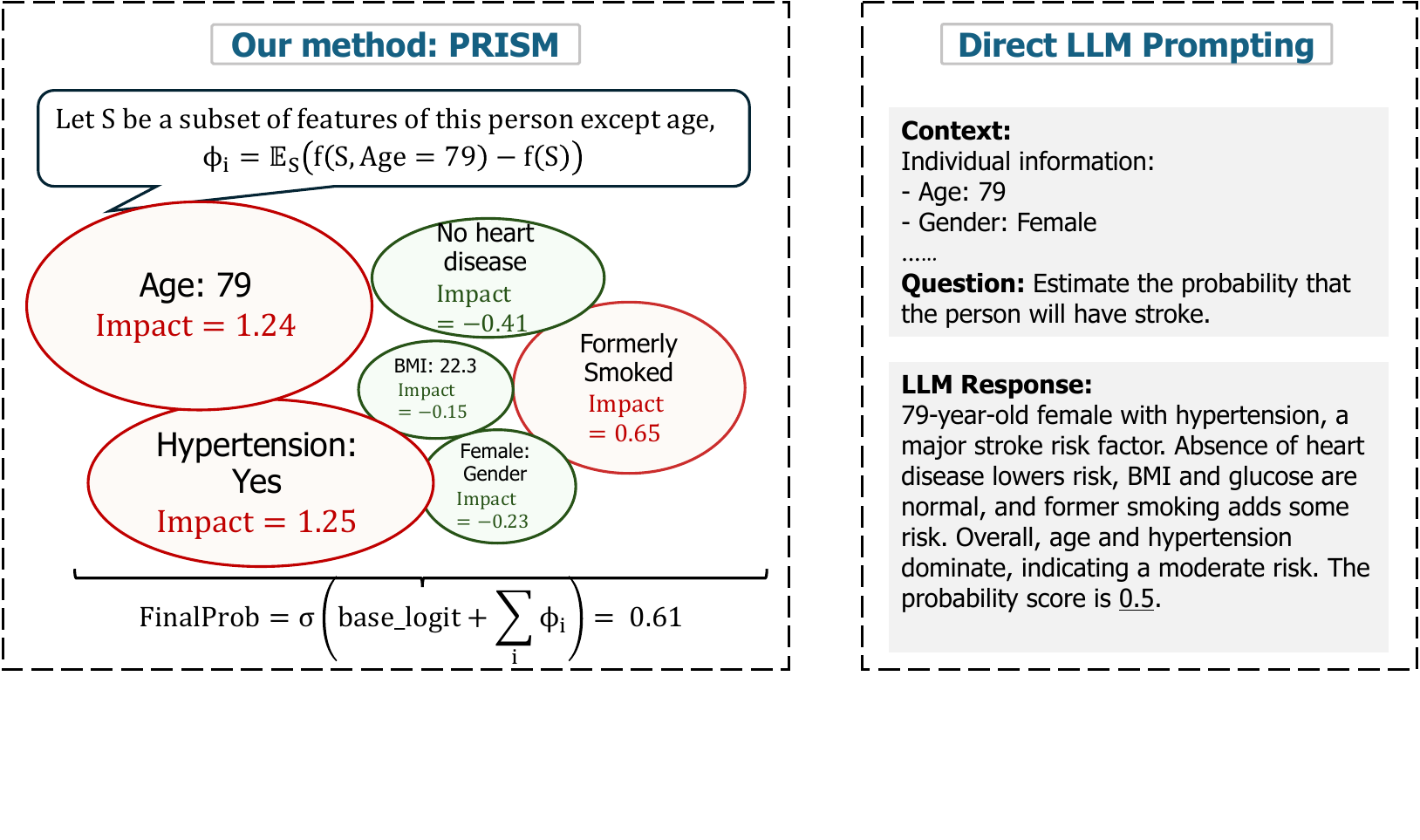}
    \caption{Illustration of direct LLM prompting and PRISM. PRISM first estimates Shapley values (factor contributions) and aggregates them to reconstruct  final probability. We use \textcolor{BrickRed}{red} to represent the factors found by PRISM to have positive contribution to the positive outcome (have a stroke), \textcolor{ForestGreen}{green} are factors found to be negative. The size reflects the contribution's absolute value. $f(\cdot)$ is from LLM prediction and $S$ is a background set, $\sigma(\cdot)$ is the sigmoid function (see details in Section~\ref{sec:method}).} 
    \vspace{-0.2cm}
    \label{fig:introduction}
\end{figure}

\vspace{-0.3cm}
\section{Related Works}
\vspace{-0.2cm}

\subsection{Probability Estimation via Traditional ML and LLMs}
\vspace{-0.2cm}

Traditional machine learning models such as Bayesian inference~\citep{berger2013statistical}, multilayer perceptrons (MLPs)~\citep{rumelhart1986learning} and decision trees~\citep{quinlan1986induction} have long been widely used for probability estimation, typically cast as binary classification tasks on tabular data. These methods often rely on large amount of data and well-designed models. To overcome this limitation, recent works have explored using Large Language Models (LLMs) for probability estimation by leveraging their rich prior knowledge. In this paper, we focus on LLM-based probability estimation \textbf{under zero-shot setting} which makes estimations only based on the LLM's own knowledge. Under this setting, previous  studies validate the potential especially for tabular format datasets~\citep{sui2024table, chung2024large, Ren2025PredictingLM, xie2024finben}. However, although many of these methods show decent estimation precision, their estimation process usually relies on direct LLM prompting. Thus, it can be difficult for the users to exactly measure the contribution of each factor to the final prediction, different from the case in traditional ML models. As a more relevant paper to our work,  BIRD \citep{feng2024bird} is proposed to integrate Bayesian networks~\citep{berger2013statistical} to explicitly calculate the probability of the prediction outcome conditioning on each factor. However, it assumes the factors are independent to each other, and it has to transform the factors into categorical values. Notably, beyond the zero-shot setting, there are related studies focusing on few-shot prediction with LLMs~\citep{hegselmann2023tabllm, brown2020language}. However, in such cases, the interpretation can become more complicated, as it should disentangle whether a prediction arises from the own knowledge of LLMs or from the provided demonstrations. Therefore, we leave the exploration of few-shot prediction tasks to future work.

\vspace{-0.2cm}
\subsection{Explainable probability estimation}
\vspace{-0.2cm}

For traditional ML explanation, Shapley-value~\citep{shapley1953value, lundberg2017unified,vstrumbelj2014explaining} is widely used to attribute predicted probabilities to input factors. In general, it quantifies each factor’s marginal contribution to the final prediction outcome. 
Building on this line, recent works adapt Shapley values to explain the behavior of LLMs. For instance,  TokenSHAP~\citep{goldshmidt2024tokenshap} estimates token-level contributions to LLM outputs in general question answering tasks. SyntaxSHAP~\citep{amara2024syntaxshap} moves beyond token-level attributions by capturing the contribution of higher-level syntactic units.
Unlike these studies, our work specifically targets on probability estimation tasks, where we quantify the contributions of each ``influencing factor'' in the probability estimation. 
Besides,, our approach goes beyond conventional LLM or ML interpretation tools. We reconstruct new probability estimates through a principled approach of factor aggregation, rather than only interpreting the original model outputs.

\vspace{-0.3cm}
\section{Methodology}\label{sec:method}
\vspace{-0.3cm}

In this section, we  formally introduce and describe our proposed method Probability Reconstruction via Shapley Measures (PRISM). In 
 Section 3.1, we first introduce the necessary notations and definitions, with particular emphasis on the Shapley value and its importance for ML model explanations. Section 3.2 then provides a full description of the PRISM algorithm, clarifying how factor contributions are combined to reconstruct probabilities. Finally, Section 3.3 introduces an efficient variant of PRISM specifically designed for tabular tasks, enabling faster computation in such scenarios.


\vspace{-0.2cm}
\subsection{Definitions and Notations}
\vspace{-0.2cm}

\textbf{Notations.} In our study, for a probability estimation task, we denote an \textbf{instance} by $x=(x_1,\ldots,x_m)$ and its true outcome by $y_{true}\in\{0,1\}$, where each $x_i$ is a factor that can influence the outcome.
We let $f(\cdot)$ denote a general model output, which may come from either a traditional ML binary classifier or from LLM-based predictions. 
Depending on the setting, $f(\cdot)$ can represent either the predicted probability or the model logit (before sigmoid function). In general, it is desirable to obtain the estimation $f(x)$ to be close or aligned to $y_{true}$.


In our paper, an important concept is Shapley value~\citep{shapley1953value} that applied in ML prediction explanation~\citep{lundberg2017unified}. In the next, we first introduce the basic definition of Shapley value (for ML models) in general, and we discuss one important property of it. 

\begin{definition}[Shapley value]
Let $\mathcal{I} = \{1,2,\dots,m\}$ denote the index set of factors, and let $x = (x_1, \dots, x_m)$.  
The Shapley value of factor $i \in \mathcal{I}$ for instance $x$ with respect to model $f$ is:
\begin{align}
\phi_i(f,x) = \sum_{S \subseteq \mathcal{I} \setminus \{i\}}
\frac{|S|!\,(m-|S|-1)!}{m!}\,
\Big[ f(x_{S \cup \{i\}}) - f(x_S) \Big],
\label{eq:shap}
\end{align}
where each $S$ is a subset of factors $\mathcal{I}$ excluding $i$, and the sum is taken over all $S\subseteq \mathcal{I}\setminus\{i\}$, $f(x_S)$ is the model output when only the factors in $S$ are used 
(and $f(x_{S \cup \{i\}})$ is defined analogously). In our paper, we name each $S$ as a ``\textbf{background set}''. 
\end{definition}
\vspace{-0.1cm}

Intuitively, Shapley value $\phi_i(f,x)$ measures the \emph{marginal contribution} of factor $i$ to the model output, by comparing the model's output when factor $i$ is ``included in the background set $S$'' versus when it is ``not included in the background set $S$'', across all possible cases of $S$. In this way, one can interpret the model’s prediction by separately examining the contribution of each individual factor. In practice, for traditional ML, the difference 
$\big(f(x_{S \cup \{i\}}) - f(x_S)\big)$ 
can be computed either by retraining models on different subsets of factors~\citep{lipovetsky2001analysis}, 
or by approximating the effect of missing factors using training data~\citep{vstrumbelj2014explaining}. Next, we discuss one important property of Shapley value in traditional MLs which inspires our method.

\begin{property}[Additivity~\citep{lundberg2017unified}]\label{property1}
Let $\phi_0 = f(x_\emptyset)$ be the expected model output when no factor information is available. Then, for models with deterministic scalar outputs, the sum of $\phi_0$ and the Shapley value of the factors 
$\phi_i(f, x)$ must \textbf{reconstruct} the model prediction itself.  
\begin{align}\label{eq:localacc}
f(x) = \phi_0 + \sum_{i=1}^m \phi_i(f,x).
\end{align}
That is, the explanation exactly matches the original model output.
\end{property}

It suggests for traditional ML models, the output $f(x)$ can be decomposed to the sum of Shapley values and the base logit $\phi_0$. However, Property~\ref{property1} does not extend easily to Large Language Models (LLMs). This is because, when LLMs make predictions for uncertain events, they express their probability estimates through generated text (see Figure~\ref{fig:introduction} (right)). Such expressed probabilities differ from the direct model outputs, such as the token probabilities produced by the LLM architectures.


\vspace{-0.3cm}
\subsection{PRISM: Probability Reconstruction via Shapley Measures}\label{sec:prism}
\vspace{-0.2cm}

Nevertheless, inspired by Property~\ref{property1}, we devise our new estimation framework: we first obtain the Shapley values for each factor and then aggregate them to reconstruct a new prediction outcome. In the next, we elaborate two important steps to obtain the Shapley values through LLMs.

\begin{wrapfigure}{r}{0.38\textwidth} 
  \centering
  \vspace{-0.4cm}
\includegraphics[width=0.36\textwidth]{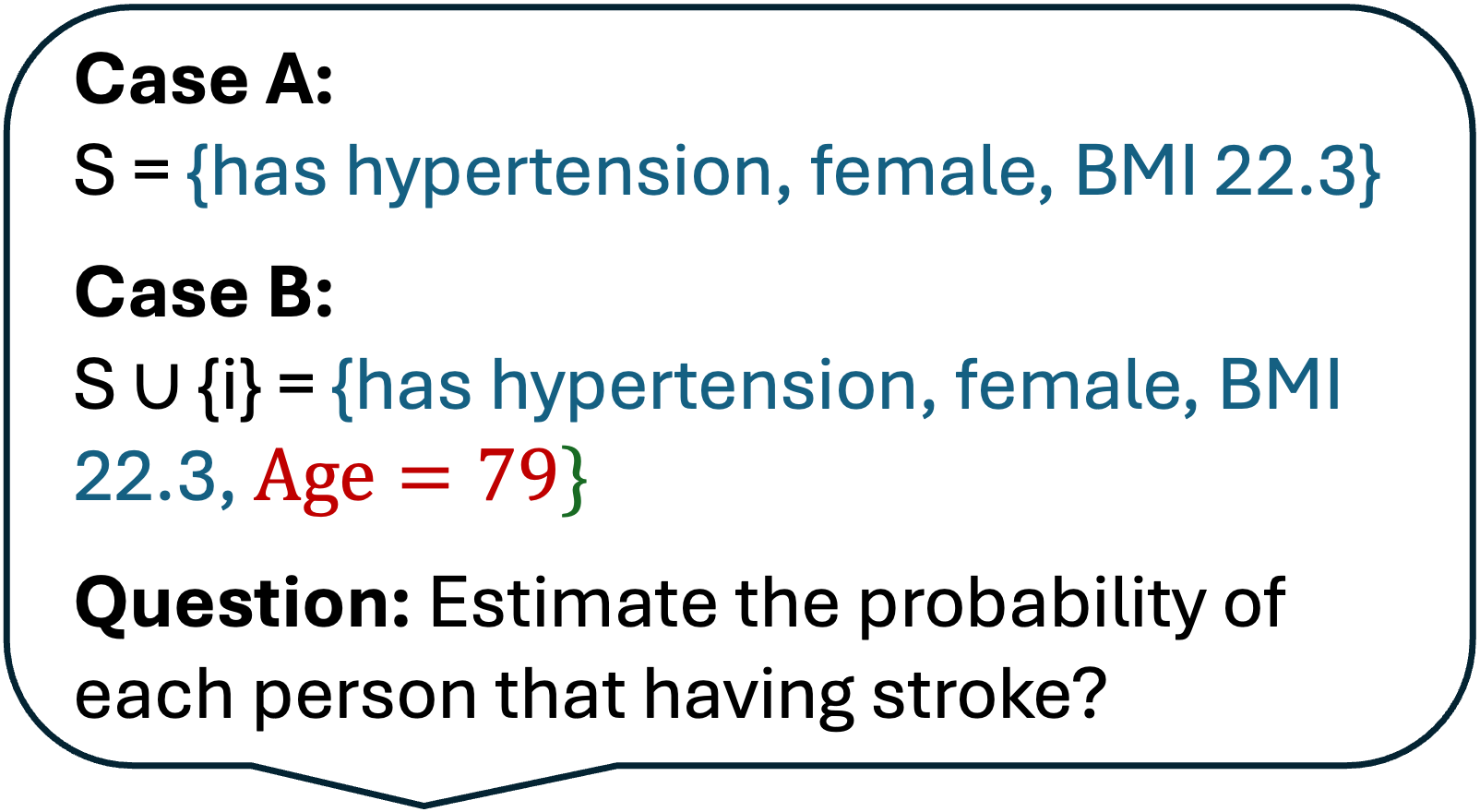} 
\vspace{-0.3cm}
  \caption{\small Comparative Prompting.}
    \vspace{-0.4cm}
  \label{fig:compare}
\end{wrapfigure}
\textbf{Comparative LLM Prompting.}
Given a factor of interest, such as ``$i=\text{Age is 79}$'' in the stroke prediction task,  we first sample multiple background sets $S$ from $\mathcal{I}\setminus\{i\}$ 
(following the permutation sampling strategy in the next part). Each $S$ contains information of several factors that different from $i$, such as ``hypertension, female and BMI=22.3'' illustrated in Figure~\ref{fig:compare}. 
Next, for a given background set $S$, we prompt the LLM to evaluate both $x_{S\cup\{i\}}$ and $x_{S}$ in the same query. 
We denote $p(\cdot)$ as the LLM estimated probability for a given instance (e.g., by responding to a question as in Figure~\ref{fig:compare}).
Then, if we let $f(\cdot) = \sigma^{-1}(p(\cdot))$, where $\sigma^{-1}(\cdot)$ is the inverse sigmoid function, we have:
\begin{align}\label{eq:diff}
    f(x_{S\cup\{i\}}) - f(x_{S}) = \sigma^{-1}\big(p(x_{S\cup\{i\}})\big) - \sigma^{-1}\big(p(x_S)\big),
\end{align}
This calculates the difference term in Shapley value defined in Eq.(\ref{eq:shap}).
We argue that: the single estimates derived by directly LLM prompting may not be exact, but the relative relation between the two cases can be well-captured. Recent studies also find that LLMs tend to perform more reliably on ranking or comparison tasks 
than absolute probability estimation~\citep{qin2023large, liu2024aligning}, 
which supports the validity of our design.

\textbf{Permutation Sampling.} In PRISM, at each iteration we sample a background set $S$ 
using the permutation sampling rule~\citep{shapley1953value, vstrumbelj2014explaining}.  Concretely, we generate a random permutation of all factors in $\mathcal{I}$ including $i$, 
and choose $S$ as the set of factors that precede $i$ in this order. Then, we directly average the pair-wise differences (from Eq.(\ref{eq:diff})) across multiple (e.g., $K$ times) such samplings to approximate the Shapley value:
\begin{align}
\phi_i = \frac{1}{K} \sum_{k=1}^{K} 
\Big[ \sigma^{-1}\big(p(x_{S^{(k)}\cup\{i\}})\big) - \sigma^{-1}\big(p(x_{S^{(k)}})\big) \Big],
\label{eq:perm-shap}
\end{align}
The equivalence of Eq.(\ref{eq:perm-shap}) and Eq.(\ref{eq:shap}) is shown in~\citep{shapley1953value}.
In practice, it is necessary to sample $S$ for multiple times, ensuring each background factor is considered with a non-negligible probability. 
This will lead to a more comprehensive estimation by accounting for possible factor interactions (see Section~\ref{sec:ablation} for further discussion).

Finally, we set $\phi_0$ as the base logit, either from the population average 
(e.g., $\phi_0 = \sigma^{-1}$(Average stroke prevalence), or simply set to $0$ when no prior knowledge is available in other tasks. We then reconstruct the probability estimation as:
\begin{align}\label{eq:prism}
p_{PRISM}(x)=\sigma (\phi_0 + \sum_{i=1}^m \phi_i).
\end{align}
We provide the detailed algorithm sketch in Appendix~\ref{app:algo}.  

\vspace{-0.2cm}
\subsection{Tabular-PRISM}\label{sec:tabular-prism}
\vspace{-0.2cm}

In this subsection, we discuss a variant of PRISM, when the factors can be represented as single values or categories and be integrated in tabular datasets. For such tasks, instead of calculating the difference term in Eq.(\ref{eq:diff}) for each $S$ in a separate LLM query, we can instead consider multiple background sets in \textbf{one single query}. Refer to Figure~\ref{fig:missed}, we can present $x_{S}$ and $x_{S\cup\{i\}}$ for different $S$ in the same table (the adjacent rows are from the same $S$). Then, we prompt LLMs to evaluate each of them in one query. In this way,  the query and token efficiency can be significantly improved. However, a challenge is that: if we leave those factors that are not presented as blank, ``unknown'' or ``not provided'', it will introduce bias to the LLM judgment. For example, LLM can over-estimate a person's risk of certain disease only because some information is ``unknown''.

\begin{wrapfigure}{r}{0.5\textwidth} 
  \centering
  \vspace{-0.4cm}
  \begin{subfigure}{0.48\linewidth}
\includegraphics[width=\linewidth]{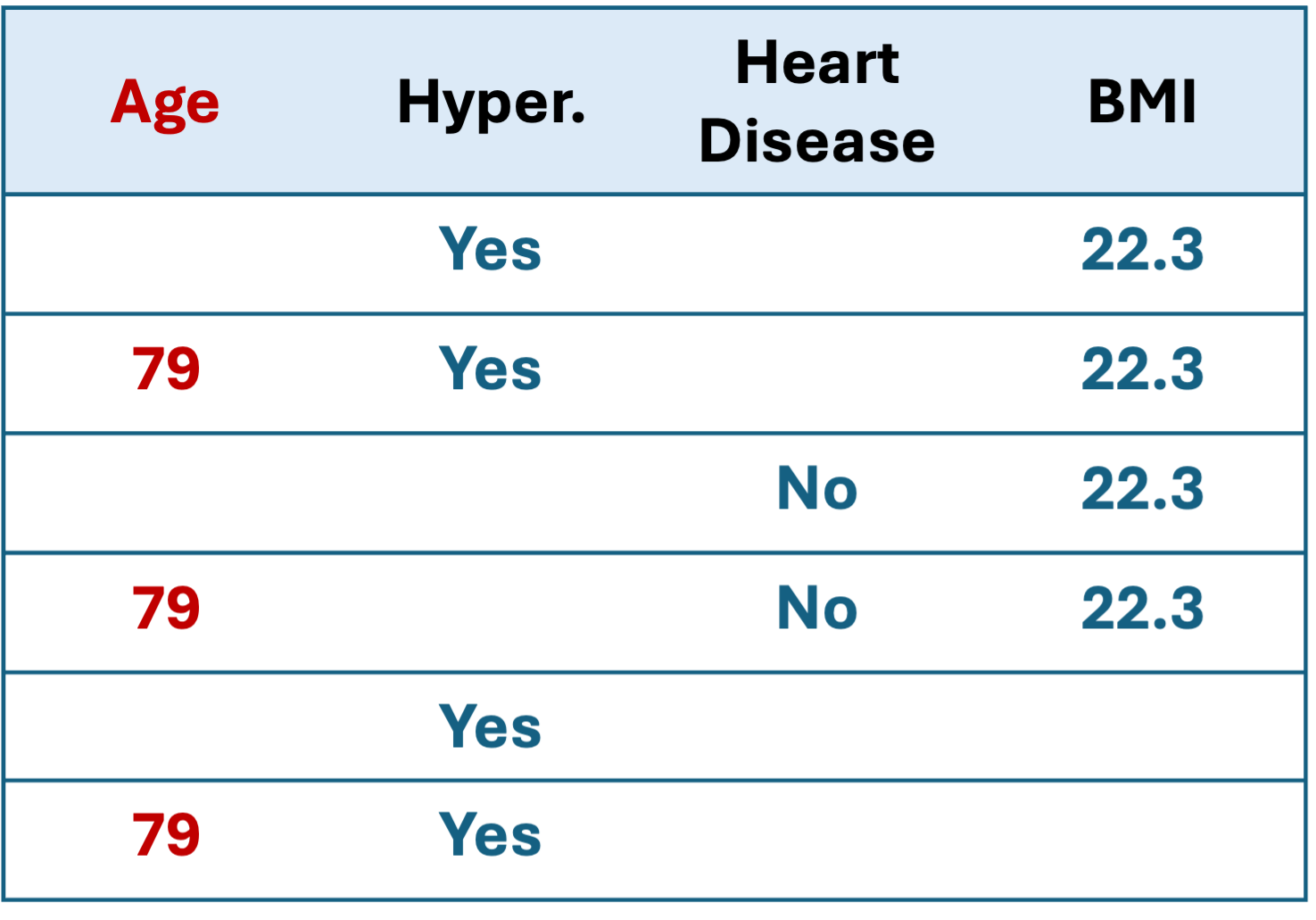}
  \vspace{-0.5cm}
    \caption{\small With blanks.}
    \label{fig:missed}
  \end{subfigure}
  \hfill
  \begin{subfigure}{0.48\linewidth}
\includegraphics[width=\linewidth]{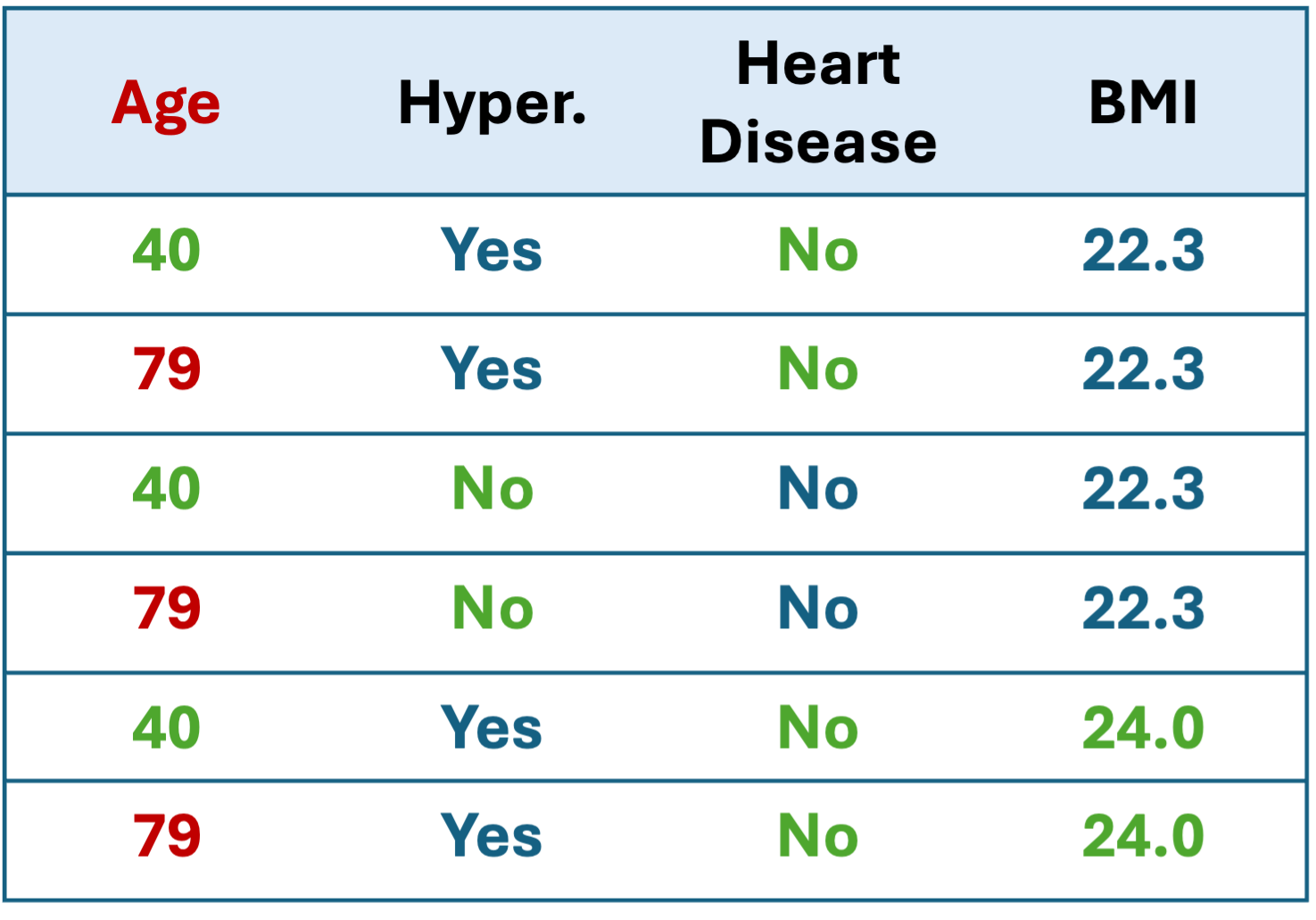}  
  \vspace{-0.5cm}
    \caption{\small With reference.}
    \label{fig:imputed}
  \end{subfigure}
  \vspace{-0.2cm}
  \caption{\small When calculating Shapley value of the factor ``Age=79'', we put multiple $S$ in one table. Factor values from reference instances are noted in \green{green}.}
  \label{fig:wrapsubfig}
  \vspace{-0.4cm}
\end{wrapfigure}
To solve this issue, we leverage the strategy to introduce a \textbf{reference instance}, which is also used in Shapley ML explanation~\citep{vstrumbelj2014explaining}. Specifically, we let 
$r$ be a reference instance with each factor to be a fixed value (typically obtained from population average or majority). Then, we use the information of this reference instance to impute missing factors, which is illustrated in Figure~\ref{fig:imputed}, and  we calculate the Shapley values (with a new definition) based on this new table. 

\begin{definition}\label{def:shap2}
For background set $S$, we
define the model output of the imputed sample as
$
v_r(S) = f\big([x_S,\; r_{\bar S}]\big)$, where $\bar S=\mathcal{I}\setminus S$, which means the factors in $S$ are provided by $x$, and the remaining are provided by $r$. Then, for each factor $x_i$, the (reference-specific) Shapley value is
\begin{align}\label{eq:table-shap}
\phi_i^{(r)} =
\sum_{S\subseteq\mathcal{I}\setminus\{i\}}
\frac{|S|!\,(m-|S|-1)!}{m!}\,
\Big(v_r(S\cup\{i\})-v_r(S)\Big). 
\end{align}
\end{definition}
Essentially, this new definition has a different interpretation compared to the original definition in Eq.(\ref{eq:shap}). It emphasizes the comparison to the reference sample. For example, a large positive $\phi_i^{(r)}$ for ``Age = 79'' suggests: compared with ``Age = 40'' (from reference sample), the factor ``Age = 79''  increases the risk of having stroke greatly. Despite different interpretations, the following proposition can still allow us to reconstruct predictions from these Shapley values.  

\begin{restatable}{proposition}{mainprop}\label{thm:proposition1}
Fix an instance $x$ and a single reference sample $r$.
Let $\phi_i^{(r)}$ be the Shapley value in Definition~\ref{def:shap2}.
Then, for models with deterministic scalar outputs, with $\phi_0^{(r)}=v_r(\emptyset)$, we still have:
\begin{align}
  f(x) = v_r(\mathcal{I})  = \phi_0^{(r)} + \sum_{i=1}^m \phi_i^{(r)}   
\end{align}
\vspace{-0.4cm}
\end{restatable}

Similar to Property~\ref{property1}, this proposition suggests that: in LLMs, we can also reconstruct a new estimation by aggregating Shapley values, following the rule: $p_{PRISM}(x)=\sigma (\phi^{(r)}_0 + \sum_{i=1}^m \phi^{(r)}_i).$
We defer both the detailed algorithm sketch and the proof in Appendix~\ref{app:algo} and Appendix~\ref{app:thm}.

\vspace{-0.2cm}
\subsection{Factor Extraction} 
\vspace{-0.2cm}

In PRISM, it is essential to determine which factors should be considered during prediction. We adopt two criteria for selecting these factors: (1) Non-overlap, which avoids redundancy by ensuring that the same information is not repeatedly considered or recalculated, and
(2) Completeness, which ensures that the union of all factors covers the necessary information used by the LLM. In practice, for tabular tasks, factors are directly chosen from the provided tabular features after removing duplicates. For unstructured tasks, we use an automated pipeline. We first query an LLM to propose the minimal set of aspects required to complete the task, with the constraints that these aspects must be non-overlapping and complete. For each aspect, we then extract a summary from the given context, which serves as a “factor’’ within the PRISM framework.

\vspace{-0.2cm}
\section{Experiment and Discussion}
\vspace{-0.2cm}

In this section, we conduct experiments to validate the effectiveness of PRISM. In Section~\ref{sec:exp1}, we compare the estimation performance of PRISM and baselines on benchmark (tabular) datasets. Meanwhile, we provide case analysis to visualize the interpretation outcome of PRISM.
In Section~\ref{sec:exp2}, we demonstrate that PRISM can also be applied in real-world estimation tasks which cannot be easily formed into tabular format. Finally, Section~\ref{sec:ablation} provides ablation studies about whether PRISM can handle feature interactions~\citep{hall1999correlation}, and discuss its computational efficiency. 

\vspace{-0.2cm}
\subsection{Probability Estimation on Benchmark Tabular Datasets}\label{sec:exp1}
\vspace{-0.2cm}

We first evaluate our method focused on benchmark \textbf{tabular} datasets. This is because probability estimation tasks are more frequently built in tabular format, for traditional ML studies. Therefore, we have various datasets with sufficient samples for fair and comprehensive comparisons. We later discuss scenarios beyond tabular format in Section~\ref{sec:exp2}.

\textbf{Experiment setup.} We involve four representative tabular datasets ranging from disease prediction, income prediction, and loan credit estimation. In details,  Adult Census Income~\citep{uci_adult_1996} is to predict whether a person has annual income over \$50K, based on their occupation, education and family information.
Stroke~\citep{stroke_kaggle} and Heart Disease~\citep{chicco2020machine} predict whether a certain disease will appear based on the patients' health status. Lending~\citep{lendingclub_kaggle} predicts whether a loan default will happen, based on the loan application records and applicants' personal information. More details and pre-process procedures are in Appendix~\ref{sec:addl-datasets}.

In our study, we majorly consider the zero-shot settings, where the estimations are solely relied on the LLM's own knowledge.
For our method, we implement Tabular-PRISM (in Section~\ref{sec:tabular-prism}) and we sample the background set $S$ for 10 times (for each Shapley value) \footnote{We set the factor values in the reference instance to be around the population average (for continuous values), and population majority (for categorical values). We query an LLM for the base logit $\phi_0$. In practice, one can choose $\phi_0$ from more reliable sources such as historical statistics or human experts, and the selection of $\phi_0$ will not impact the later evaluation metrics, including AUROC, AUPRC and F1.}.
We also consider the baselines:
\begin{itemize}[itemsep=0em, leftmargin=2em]
    \item Directly prompting LLMs to predict likelihood \textbf{levels}. For example, we ask LLM to choose one in the options from ``very unlikely'' to ``very likely''. We also try multiple shots to obtain the self-consistency result~\citep{wang2022self} by making votes. They are denoted as ``1shot\_level, 5shot\_level, 10shot\_level'' in Table~\ref{tab:comparison}.
    \item Directly prompting LLMs to predict likelihood \textbf{scores}, which are probability scores between [0-1], and we obtain the self-consistency result by taking their average. They are denoted as ``1shot\_score, 5shot\_score, 10shot\_score'' in Table~\ref{tab:comparison}
    \item \textit{Contrast}~(motivated by \citet{nafar2025extracting}) asks about the likelihood in the positive and the negated question form, and then unify the answers to get the final estimation.
    \item \textit{BIRD}~\citep{feng2024bird} builds Bayesian networks to evaluate the probability of the estimation outcome conditioning on various (categorized) factors. 
    \item We also add In-Context Learning~(ICL, \cite{brown2020language}), which is beyond zero-shot setting. We randomly select 5 positive and 5 negative samples as demonstrations for each instance (or 10 positive and 10 negative respectively).
\end{itemize} 
For each setting and method, we conduct experiments on GPT-4.1-mini \citep{gpt41mini_docs} and Gemini-2.5-Pro \citep{gemini25pro_card}. 
All runs use temperature 1.0 under default settings.

\textbf{Experiment result.} In our result shown in Table~\ref{tab:comparison}, for each dataset, we randomly choose 300 test samples (150 positive and 150 negative), and we report AUROC (shown as ``ROC'' in Table~\ref{tab:comparison}), AUPRC (shown as ``PRC''), and the best F1 (when select the threshold for maximized TPR + TNR). From the result, we can see the PRISM consistently demonstrates reliable estimation performance. Specifically, PRISM presents highest performance or it is close to the strongest baselines in most datasets, including Adult Census, Heart Disease and Lending, under both LLMs. In these datasets, we see PRISM has more obvious improvement in AUROC and AUPRC than F1-scores. This suggests that the baselines can sometimes effectively separate positive and negative samples when an appropriate threshold is chosen, whereas the high AUROC and AUPRC of PRISM indicate its strong discriminative ability across all possible thresholds. Under Stroke Dataset, our method is comparable to strong baselines or slightly lower than those strong baselines. For example, under Gemini-2.5-pro, ``1shot\_score'' has higher AUROC, AUPRC and F1 than PRISM. We conjecture it may be because the LLM itself has well-educated knowledge in the relevant domain. However, PRISM demonstrates stable performance across various datasets and shows competency against the strongest baselines, if not surpassing them. in Appendix~\ref{sec:addl-results}, we show PRISM's predicted probability also has a good calibration~\citep{bella2010calibration} if a proper base logit is selected.

For the baseline methods, we find ``score'' based LLM prompting are generally better than ``level'' based prompting. The most comparable baseline with PRISM is ``5shot\_score'' and ``10shot\_score''. However, we would like to argue that: the multiple-query strategy (or namely Self-Consistency) is less interpretable than single-query in practice, as they rely on aggregating multiple diverse estimation paths, making it difficult for a uniform assessment of the factor contribution. Besides, BIRD is the only method with explicit interpretable structure among the baselines. Compared to BIRD, PRISM has obviously higher performance across different settings. 

\begin{table}[t]
\centering
\resizebox{0.9\linewidth}{!}{
\begin{tabular}{c|ccc|ccc|ccc|ccc}\hline\hline 
& \multicolumn{3}{c|}{Adult Census} 
& \multicolumn{3}{c|}{Heart Disease}  
& \multicolumn{3}{c|}{Stroke} 
& \multicolumn{3}{c}{Lending} \\ 
Model / Method 
& ROC  & PRC  & F1   
& ROC  & PRC  & F1   
& ROC  & PRC  & F1  
& ROC  & PRC  & F1  \\
\hline
\multicolumn{13}{l}{\textbf{GPT-4.1-mini}}\\
\hline
1shot\_level        & 0.777 & 0.773 & 0.709& 0.722 & 0.706 & 0.615 & 0.767 & 0.694 & 0.726 & 0.629 & 0.606 & 0.478 \\
5shot\_level  & 0.795 & 0.779 & 0.701 & 0.759 & 0.723 & 0.664  & 0.780 & 0.716 & 0.730 & 0.617 & 0.576 & \blue{0.648}  \\
10shot\_level & 0.799 & 0.779 & 0.701  & 0.759 & 0.718 & 0.672  & 0.792 & 0.727 & 0.732  & 0.627 & 0.591 & 0.443  \\
1shot\_score & 0.795 & 0.792 & 0.709 & 0.799 & 0.761 & 0.772 & 0.804 & 0.763 & 0.753 & 0.612 & 0.600 & 0.558 \\

5shot\_score & \blue{0.819} & \blue{0.820} & \blue{0.755} & \blue{0.807} & \blue{0.779} & \blue{0.782} & \dblue{0.816} & \blue{0.783} & 0.780 & 0.629 & 0.621 & 0.612\\ 
10shot\_score & 0.816 & \blue{0.820} & 0.734 & 0.806 & 0.776 & \dblue{0.793} & 0.813 & 0.781 & \blue{0.785} & \blue{0.636}  & \dblue{0.631} & 0.621\\
ICL-5+5      & 0.807 & 0.788 & 0.707 & 0.803 & 0.761 & 0.762  & 0.801 & 0.751 & 0.775 & 0.598 & 0.567 & 0.668  \\
ICL-10+10    & 0.754 & 0.758 & 0.645  & 0.776 & 0.744 & 0.760  & 0.788 & 0.736 & 0.763  & 0.591 & 0.569 & 0.642  \\
Contrast     & 0.790 & 0.813 & 0.697  & 0.769 & 0.729 & 0.738  & 0.790 & \dblue{0.813} & 0.697 & 0.485 & 0.543 & 0.169 \\
BIRD          & 0.813 & 0.804 & 0.730 & 0.777 & 0.748 & 0.712 & 0.778 & 0.740 & 0.729 & 0.610 & 0.564 & 0.533  \\
PRISM (Ours)  & \dblue{0.851} & \dblue{0.874} & \dblue{0.770} & \dblue{0.816} & \dblue{0.799} & \dblue{0.793} & \blue{0.814} & \blue{0.783} & \dblue{0.790} & \dblue{0.655} & \blue{0.626} & \dblue{0.671} \\
\gray{Std. Err. (Ours)} & \gray{0.0215} & \gray{0.0221} & \gray{0.0296} & \gray{0.0254} & \gray{0.0350} & \gray{0.0250} & \gray{0.0311} & \gray{0.0430} & \gray{0.0283} & \gray{0.0313} & \gray{0.0428} & \gray{0.0300} \\
\hline
\multicolumn{13}{l}{\textbf{Gemini-2.5-Pro}}\\
\hline
1shot\_level        & 0.818 & 0.797 & 0.699 & 0.732 & 0.668 & 0.780  & 0.793 & 0.744 & 0.720 & 0.555 & 0.538 & 0.518 \\
5shot\_level  & 0.822 & 0.794 & 0.694  & 0.758 & 0.686 & 0.794 & 0.804 & 0.749 & 0.727 & 0.541 & 0.526 & 0.511 \\
10shot\_level & 0.821 & 0.797 & 0.688 & 0.738 & 0.663 & 0.791 & 0.803 & 0.745 & 0.716 & 0.564 & 0.543 & 0.571  \\
1shot\_score & 0.855 & 0.876 & 0.767 & 0.812 & 0.739 & \blue{0.797} & \dblue{0.836} & \blue{0.812} & 0.802 & 0.536 & 0.538 & 0.465\\
5shot\_score  & \blue{0.864} & \blue{0.879} & \blue{0.823} & 0.816 & 0.749 & 0.795 & 0.834 & 0.810 & \dblue{0.804} & 0.529 &0.534 & 0.533\\
10shot\_score & \blue{0.864} & 0.878 & \dblue{0.826} & 0.815 & 0.753 & \dblue{0.799} & \blue{0.835} & \dblue{0.814} & \blue{0.803} & 0.528 &0.526 & 0.547\\
ICL-5+5      & 0.842 & 0.834 & 0.744 & \blue{0.834} & \blue{0.785} & 0.770 & 0.815 & 0.767 & 0.728  & 0.615 & 0.576 & 0.575 \\
ICL-10+10    & 0.812 & 0.793 & 0.722 & 0.807 & 0.756 & 0.739 & 0.818 & 0.776 & 0.738 & \blue{0.626} & 0.584 & \dblue{0.690} \\
Contrast     & 0.835 & 0.830 & 0.757 & 0.831 & 0.779 & 0.789 & 0.806 & 0.768 & 0.727  & 0.593 & \blue{0.588} & 0.520 \\
BIRD         & 0.799 & 0.794 & 0.754  & 0.810 & 0.767 & 0.762  & 0.803 & 0.779 & 0.735 & 0.555 & 0.528 & 0.625 \\
PRISM (Ours)  & \dblue{0.876} & \dblue{0.893} & \dblue{0.826}  & \dblue{0.844} & \dblue{0.832} & 0.787 & 0.826 & 0.788 & 0.783  & \dblue{0.654} & \dblue{0.614} & \blue{0.685} \\
\gray{Std. Err. (Ours)} & \gray{0.0198} & \gray{0.0194} & \gray{0.0239} & \gray{0.0220} & \gray{0.0298} & \gray{0.0257} & \gray{0.0245} & \gray{0.0369} & \gray{0.0252} & \gray{0.0314} & \gray{0.0426} & \gray{0.0277} \\
\hline\hline
\end{tabular}
}
\vspace{-0.3cm}
\caption{Performance comparison across various tabular datasets. Highest results are in dark red. Second-highest results are in orange.}
\vspace{-0.3cm}
\label{tab:comparison}
\end{table}

\begin{table}[t]
\centering
\setlength{\tabcolsep}{4pt}
\resizebox{0.9\linewidth}{!}{%
\begin{tabular}{c |cc| cc| cc| cc}
\hline
Factor & \multicolumn{2}{|c|}{Case 1} & \multicolumn{2}{c|}{Case 2} & \multicolumn{2}{c|}{Case 3} & \multicolumn{2}{c}{Case 4} \\
 & Value & Shapley & Value & Shapley & Value & Shapley & Value & Shapley \\
\hline
Gender                & Female          & -0.08 & Female          & 0.00 & Male             & 0.00 & Female         & 0.00 \\
Age                   & 82              & 1.17  & 52              & 0.57 & 68               & 1.15 & 29             & -0.77 \\
Hypertension          & Yes             & 1.64  & Yes             & 1.41 & No               & 0.00 & No             & 0.00 \\
Heart Disease         & Yes             & 0.66  & No              & 0.00 & Yes              & 1.17 & No             & 0.00 \\
Marital Status        & Never Married   & 0.00  & Ever Married    & 0.00 & Ever Married     & 0.00 & Ever Married   & 0.00 \\
Work Type             & Government job  & 0.00  & Private sector  & 0.00 & Private sector   & 0.00 & Private sector & 0.00 \\
Residence Type        & Rural           & 0.00  & Urban           & 0.00 & Urban            & 0.00 & Urban          & 0.00 \\
Glucose Level & 84.03           & 0.00  & 94.98           & 0.51 & 223.83           & 0.62 & 116.98         & 0.98 \\
BMI                   & 25.60           & 0.00  & 23.80           & 0.00 & 31.90            & 0.53 & 23.40          & 0.00 \\
Smoking Status        & Smokes          & 0.85  & Never smoked    & 0.00 & Formerly smoked  & 0.60 & Never smoked   & 0.00 \\
\hline
\textit{Sum Shapley}          & \multicolumn{2}{|c|}{4.240} & \multicolumn{2}{c|}{2.490} & \multicolumn{2}{c|}{4.070} & \multicolumn{2}{c}{0.210} \\
\textit{Sum logit}         & \multicolumn{2}{|c|}{-0.355} & \multicolumn{2}{c|}{-2.105} & \multicolumn{2}{c|}{-0.525} & \multicolumn{2}{c}{-4.385} \\
\textit{Predicted prob}    & \multicolumn{2}{|c|}{0.413}  & \multicolumn{2}{c|}{0.109}  & \multicolumn{2}{c|}{0.372}  & \multicolumn{2}{c}{0.012}  \\

\textit{True label}       & \multicolumn{2}{|c|}{Yes}     & \multicolumn{2}{c|}{No}      & \multicolumn{2}{c|}{Yes}     & \multicolumn{2}{c}{No}      \\
\hline
\end{tabular}%
}
\vspace{-0.2cm}
\caption{Shapley values for four instances in Stroke dataset. Reference instance: \textit{\textcolor{gray}{gender=Male; age=40; hypertension=No; heart disease=No; marital status=Never Married;  residence=Rural; average glucose=90.0; BMI=24.0; work type=Private; smoking status=never smoked}}. A single base logit is shared across cases, \textcolor{gray}{\textit{$\phi_0 = \sigma^{-1}$(0.01)=-4.5951}.}}
\vspace{-0.4cm}
\label{tab:stroke-cases-impacts}
\end{table}

\textbf{Examination of Interpretations.} In Table~\ref{tab:stroke-cases-impacts}, we present instances under Stroke 
dataset to help understand the interpretation process of PRISM. In detail, we present the factor values as well as the estimated Shapley values (Eq.(\ref{eq:table-shap})). Notably, the Shapley value here represents the relative contribution of each factor, comparing to the reference instance (see Section~\ref{sec:tabular-prism}). For example, in Case~1, a Shapley value of 0.66 for ``Heart Disease'' suggests: compared with ``No Heart Disease'' (from reference instance), this factor increases the risk of having stroke. Similarly, in Case 2, because the person does not have heart disease which is the same as reference instance, the Shapley value is 0. In Appendix~\ref{sec:addl-results}, we provide the interpretation results similar to Table~\ref{tab:stroke-cases-impacts} of other datasets. 

\vspace{-0.3cm}
\subsection{Probability Estimation on Unstructured Data}\label{sec:exp2}
\vspace{-0.2cm}

Beyond tabular data, there are often cases where the influencing factors cannot be easily transformed into single values to be inputted into tables. For example, the factors can be in form of descriptive facts that are extracted from unstructured sources such as news articles, financial reports, or social media. Probability estimation in this scenario is also of great importance. We conduct multiple studies to demonstrate the applicability of PRISM (the version in Section~\ref{sec:prism}) in such setting.

\begin{figure}[t]
\centering

\begin{minipage}{0.66  
 \textwidth}
    \centering
    \includegraphics[width=\linewidth]{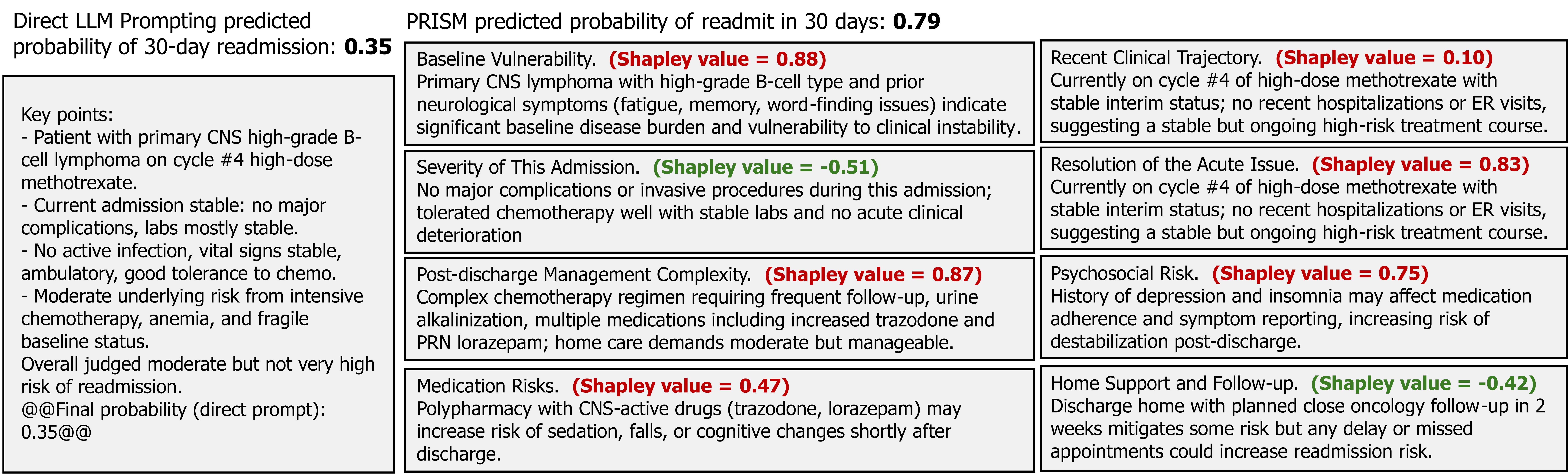}
    \caption{PRISM attribution example from MIMIC-III (True Label=1), illustrating the contribution of each  factor.}
    \label{fig:mimic_example}
\end{minipage}
\hfill
\begin{minipage}{0.32\textwidth}
    \centering
    \scalebox{0.6}{
    \begin{tabular}{c| ccc}
    \hline \hline
    GPT & ROC & PRC & F1 \\
    1shot\_level & 0.621 & 0.573 & 0.573 \\
    5shot\_level & 0.639 & 0.627 & 0.657 \\
    1shot\_score & 0.635 & 0.607 & 0.503 \\
    5shot\_score & 0.641 & 0.616 & 0.632 \\
    PRISM & 0.630 & 0.645 & 0.688\\
    \hline
    Gemini & ROC & PRC & F1 \\
    1shot\_level & 0.642 & 0.634 & 0.603 \\
    5shot\_level & 0.663 & 0.672 & 0.661 \\
    1shot\_score & 0.658 & 0.658 & 0.592 \\
    5shot\_score & 0.672 & 0.675 & 0.667 \\
    PRISM & 0.659 & 0.673 & 0.674\\
    \hline \hline
    \end{tabular}
    }
    \captionof{table}{Performance on the MIMIC-III dataset}
    \label{tab:mimic_results}
\end{minipage}
\vspace{-0.2cm}
\end{figure}

\begin{figure}[t]
  \centering
  \begin{minipage}{0.32\linewidth}
    \includegraphics[width=1.1\linewidth]{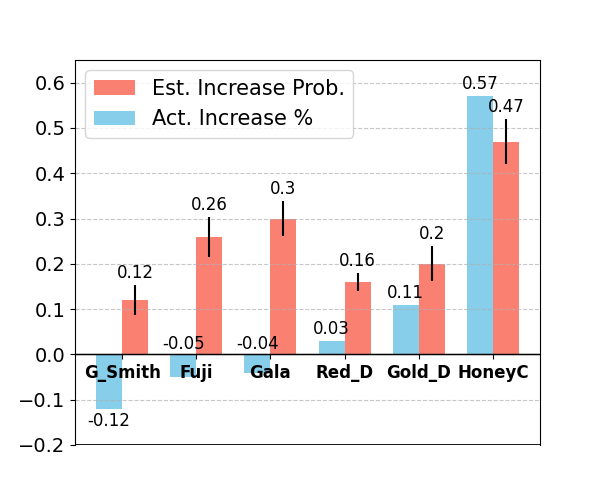}
    \vspace{-0.8cm}
\subcaption{PRISM-10\_shot}
      \label{fig:app1}
  \end{minipage}
  \hfill
  \begin{minipage}{0.32\linewidth}
    \includegraphics[width=1.1\linewidth]{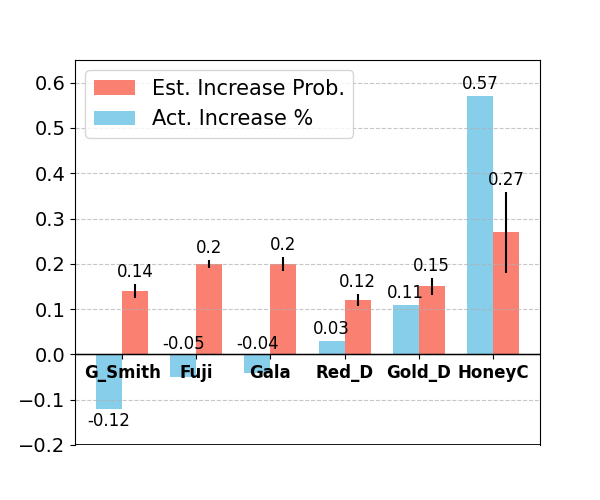}
         \vspace{-0.8cm}
      \subcaption{Extracted-10\_shot}
      \label{fig:app2}
  \end{minipage}
  \hfill
  \begin{minipage}{0.32\linewidth}
\includegraphics[width=1.1\linewidth]{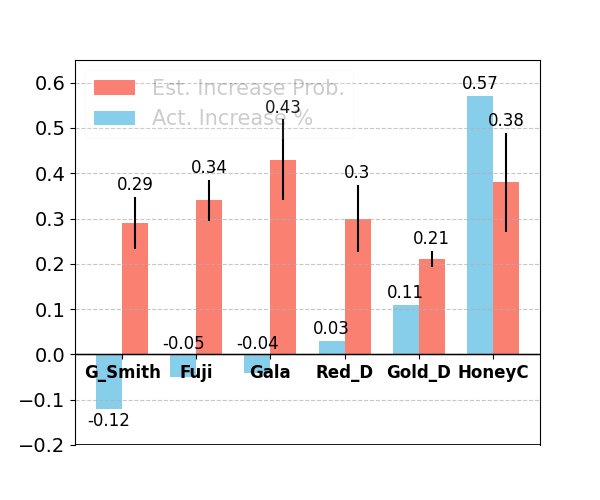}
    \vspace{-0.8cm}
\subcaption{Raw-10\_shot}
      \label{fig:app3}
  \end{minipage}
  \vspace{-0.2cm}
  \caption{Prediction on whether the price of a type of apple will increase. Blue bars are actual price increase rates. Red bars are the estimated increase probability. } 
  \label{fig:apple}
  \vspace{-0.6cm}
\end{figure}

\textbf{Predicting 30-day Readmission after Discharging.} 
We evaluate an unstructured-text prediction task under MIMIC-III dataset \citep{johnson2016mimic}. We aim to use the discharge summaries (doctor notes) to predict whether a patient will be readmitted to the hospitals within 30 days. In PRISM, we consider eight distinct factors, including the patients' overall health status and their recent hospital course and so on. We evaluate 200 samples (100 positive and 100 negative cases) using both GPT-4.1-mini and Gemini-2.5-Pro. We compare PRISM against directly prompting baselines.  Our results in table~\ref{tab:mimic_results} show that PRISM achieves consistently higher performance, particularly on PRAUC and F1. For these direct-prompting baselines, we find that LLMs such as GPT-4.1-mini often provide the same values (0.35 and 0.65) for different samples, reducing their ability to provide more fine-grained ranking. 
Moreover, Figure~\ref{fig:mimic_example} shows a full attribution example on a true readmission case. Direct prompting yields only a coarse estimate (0.35) without explaining the underlying clinical drivers. In contrast, PRISM produces a calibrated prediction (0.79) and distributes contributions across eight clinically meaningful factors, revealing how baseline disease burden, treatment intensity, psychosocial risks, and post-discharge complexity elevate the patient’s risk.

\textbf{Predicting apple price}. Our  task is to determine ``whether the price of apple will increase in 2025 compared to 2024?'', based on ({\href{https://usaa.memberclicks.net/assets/USApple_OutlookReport_2024.pdf}{U.S.~Apple Association annual report 2024}). Such a task may later assist farmers in deciding what type of produce to grow. In this report, it provides descriptive analysis regarding key factors that can influence apple prices. For each type of apple, we first use an LLM to generate summaries across seven aspects: production, demand, storage, imports and exports, policy, cost, and varietal competition. Each summary is then directly treated as a factor in PRISM. Notably, we use GPT-4.1 model for factor extraction and PRISM implementation, as its knowledge-cutoff is June 2024 and guarantees knowledge absent in 2025. 

In Figure~\ref{fig:apple}, for each type of apple, we report the actual price change rate ((price of 2025 - price of 2024) / price of 2024) (\textcolor{blue}{blue bars}), and the estimated probability of ``the price will increase in 2025'' (\textcolor{red}{red bars}). We compare PRISM with direct LLM prompting, which conducts estimation on the extracted factors (Figure~\ref{fig:app2}) and on the raw report (Figure~\ref{fig:app3}). 
Note that the factor extraction from a long report is highly stochastic and it can greatly impact the estimation, we repeat the extraction and estimation process for 10 times and report the average estimation outcome.

\vspace{0.2cm}
From the result, we see PRISM can provide relatively more promising estimation result. In detail, if we focus on ``Honeycrisp (HoneyC)'', which has the largest increase in 2025, both PRISM and ``Extracted'' give it highest estimated increase probability, although ``Extracted'' has a huge variance in its predictions. For ``Granny Smith (G\_Smith)'', which has the largest decrease, only PRISM can give it a lower expectation than other apple varieties. However, it is difficult to have more fine-grained comparison for PRISM and other strategies, e.g, to compare ``Fuji vs Red Delicious (Red\_D)''. In Figure~\ref{fig:honey-crisp} and Figure~\ref{fig:granny-smith} in Appendix~\ref{sec:addl-results}, we provide all the factor details for ``Honeycrisp'' and ``Granny Smith'', showing the prediction of PRISM is reasonable and easy to interpret.

\begin{wrapfigure}{r}{0.5\textwidth}
\vspace{-0.4cm}
  \centering
  \begin{minipage}{0.49\textwidth} 
    \begin{subfigure}{0.48\linewidth}
      \includegraphics[width=1.1\linewidth]{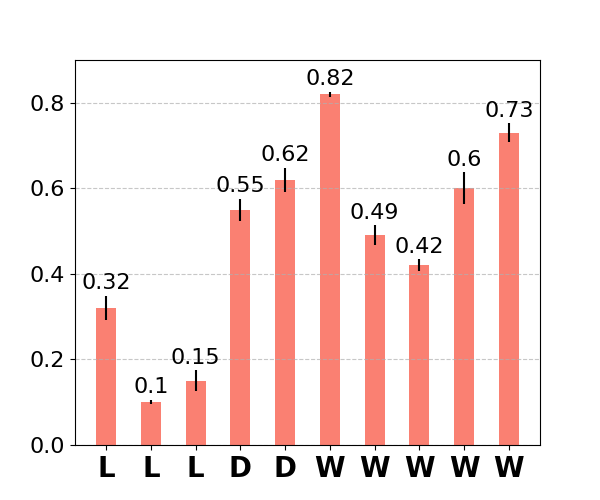}
\vspace{-0.6cm}
      \caption{PRISM}
    \end{subfigure}\hfill
    \begin{subfigure}{0.48\linewidth}
\includegraphics[width=1.1\linewidth]{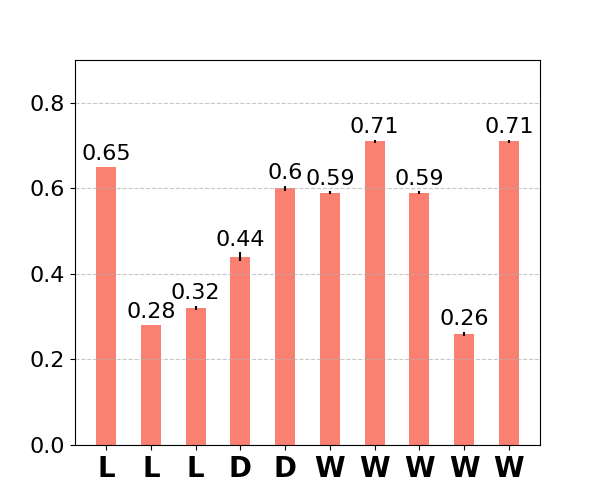}
\vspace{-0.6cm}
      \caption{Raw}
    \end{subfigure}
\vspace{-0.2cm}
    \caption{Football match result prediction.}
    \label{fig:football}
\vspace{-0.3cm}
  \end{minipage}
\end{wrapfigure}

\textbf{Predicting football matches}. In this study, we randomly choose 10 English football matches in 2025, and we leverage PRISM to estimate the probability of: ``the home team will win''. Here, we use GPT-5-mini, which is released before but relatively close to match dates. We collect pre-match reports (\href{https://www.football365.com/match-preview}{sourced from Footbal365}) and extract key features from five aspects: squad quality, head-to-head records, recent form,  player availability and fitness, and  external conditions.
 Figure~\ref{fig:football} shows the estimated winning probability (\textbf{red bars}) and the match results (\textbf{x-axis}), where ``L, D, W'' denote ``Lose, Draw and Win''. Since the reports are relatively short, we only compare PRISM with LLM direct prompting from the raw texts. From the result, we can see PRISM can correctly predict the ``Lose'' cases by giving them a low prediction value, and it gives the two ``Draw'' matches winning probability around 0.5-0.6. For the ``Win'' games, it can tell 2 of out 5 winning matches by giving them scores over 0.7. Interestingly, we find PRISM tends to focus on accounting the factors themselves during prediction, while direct LLM prompting tends to rely on overall impressions, usually assuming the stronger teams are more likely to win (see the example in Figure~\ref{fig:man_t} in Appendix~\ref{sec:addl-results}).

\vspace{-0.4cm}
\subsection{Additional Studies}\label{sec:ablation}
\vspace{-0.3cm}

In this subsection, we further answer three important questions regarding PRISM: (1) Can PRISM generalize to multi-class prediction? (2) Can it handle feature interaction~\citep{hall1999correlation}? (3) How is the computational efficiency of PRISM?

\begin{wraptable}{r}{0.30\textwidth}
\vspace{-10pt}
\centering

\resizebox{0.30\textwidth}{!}{
\begin{tabular}{l c}
\hline
Method & Accuracy \\
\hline
1shot & 0.437 \\
10shot & 0.453 \\
PRISM (Ours) & 0.613 \\
\hline
\end{tabular}
}
\caption{Accuracy on the UCI-Wine.}
\label{tab:multiclass_wine}
\vspace{-10pt}
\end{wraptable}

\textbf{Can PRISM generalize to multi-class prediction?}
We provide a concise strategy to extend PRISM to multi-class prediction tasks. Consider a prediction task with $M$ possible classes, in PRISM, we aim to quantify the marginal contribution of each factor $i$ to each possible class $m$ among all possible classes $m\in\{1,2,...,M\}$. In detail, given a background set $S$ and an interested factor $i$ (see Definition 1 of our paper),
we query the LLM to estimate the probability of $x_S$ and $x_{S\cup\{i\}}$ to belong to each class $m\in\{1,2,...m\}$. We denote these LLM generated estimations as $p^m(x_S)$ and $p^m(x_{S\cup\{i\}})$. Similar to Section~\ref{sec:method}, we can use these estimated probability values to calculate
\(f^m(x_{S\cup\{i\}})-f^m(x_S)
\) which is the logit difference, and then get the Shapley value $\phi^m_i$ to each class $m$. In the experiment, we conduct a study on UCI-Wine dataset (3 classes) which is to predict the category of different types of wines. We compare PRISM with directly prompting, which let LLM to choose the most likely class among the three options.  In the setting, we consider 1 time prompting (1shot) or asking 10 times and get the majority voting (10shot).
The accuracy reported in Table~\ref{tab:multiclass_wine} shows PRISM can greatly outperform both baselines. 
This significant gap may arise from direct prompting cannot well distinguish Class 2 (Grignolino) from Class 3 (Barbera), which are frequently confused due to their highly overlapping feature distributions. In contrast, PRISM introduces pairwise comparisons that amplify these small differences.

\textbf{Can PRISM handle feature interaction?} In ML probability estimation, feature interaction usually happens as the contribution of one factor to the outcome also highly depends on the condition of another factor. Taking this into consideration is necessary for precise and reliable estimation. Figure~\ref{fig:interaction} demonstrates that PRISM can indeed consider feature interactions when calculating the Shapley values and the final estimation. In detail, we focus on the task to predict loan default in Lending dataset under GPT-4.1-mini (see Section~\ref{sec:exp1}), where feature interactions can naturally occur between ``Loan Amount'' and ``Annual Income''. In Figure~\ref{fig:interaction1}, we compute the average Shapley value of Loan Amount across different instances, conditioned on loan amount (vertical axis) and annual income (horizontal axis).  From the result, we can see that the individuals with annual income 120K+ receive a Shapley value of 0.18 for the factor ``having a loan amount over 30K+''. It is lower than the values assigned to people with income in 0–60K or 60K–120K. This indicates that, within PRISM, for individuals earning above 120K, having a loan amount over 30K+ \textbf{is not considered as risky as those with lower incomes}. Similarly, Figure~\ref{fig:interaction2} computes the average Shapley value of Annual Income. We can see PRISM believes that having a factor ``Annual Income over 120K+'' can greatly reduce the risk (-0.65), if their loan amount is over 30K+, and it only moderately reduces the risk (-0.26), if the loan amount is low, e.g., below 10K. It indeed shows a Shapely value in PRISM does not only rely on its factor of interest, but also other factors.

\begin{wrapfigure}{r}{0.5\textwidth}
\vspace{-0.4cm}
  \centering
  \begin{minipage}{0.52\textwidth} 
    \begin{subfigure}{0.48\linewidth}
      \includegraphics[width=1.1\linewidth]{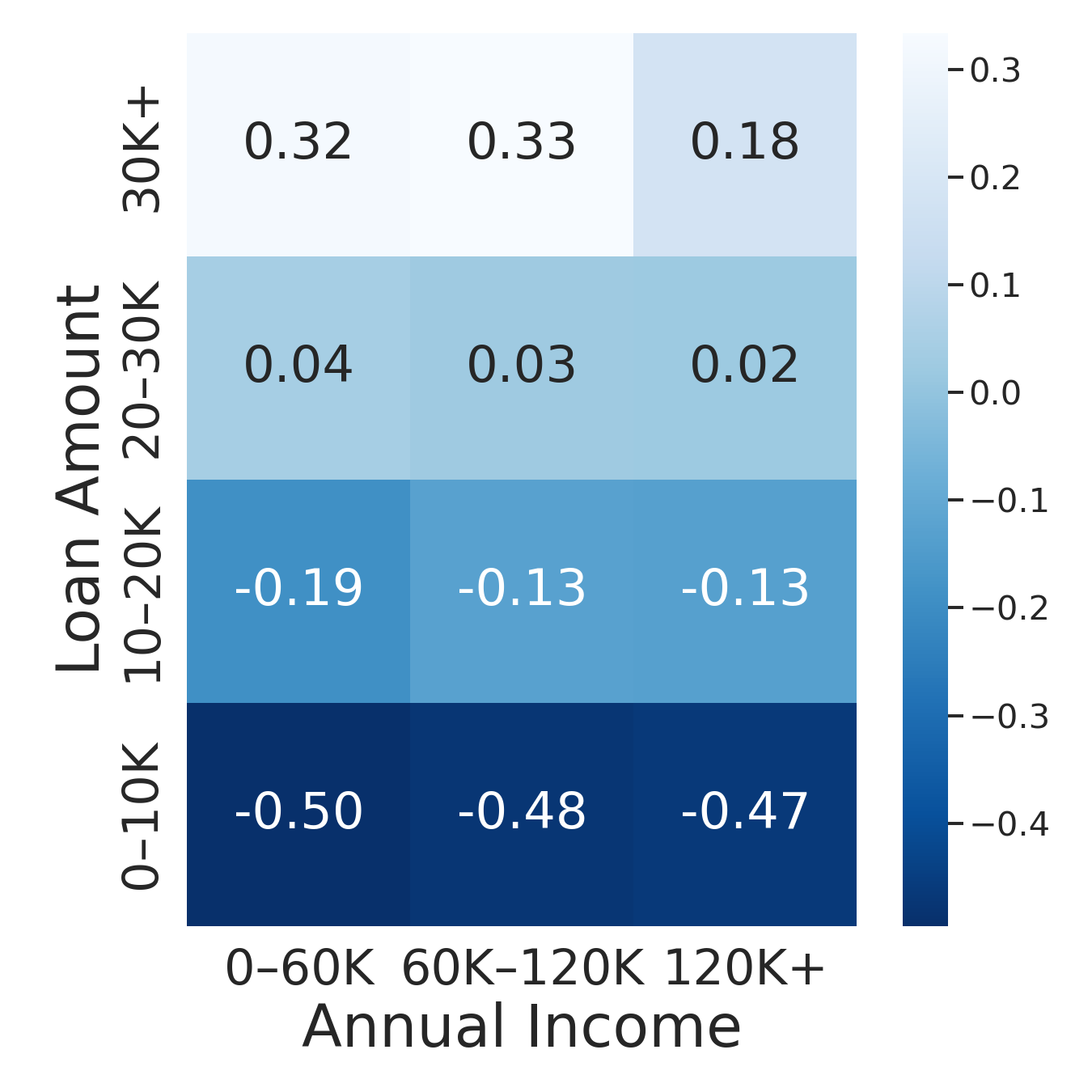}
\vspace{-0.6cm}
      \caption{\small Loan Amount}
      \label{fig:interaction1}
    \end{subfigure}\hfill
    \begin{subfigure}{0.48\linewidth}
\includegraphics[width=1.1\linewidth]{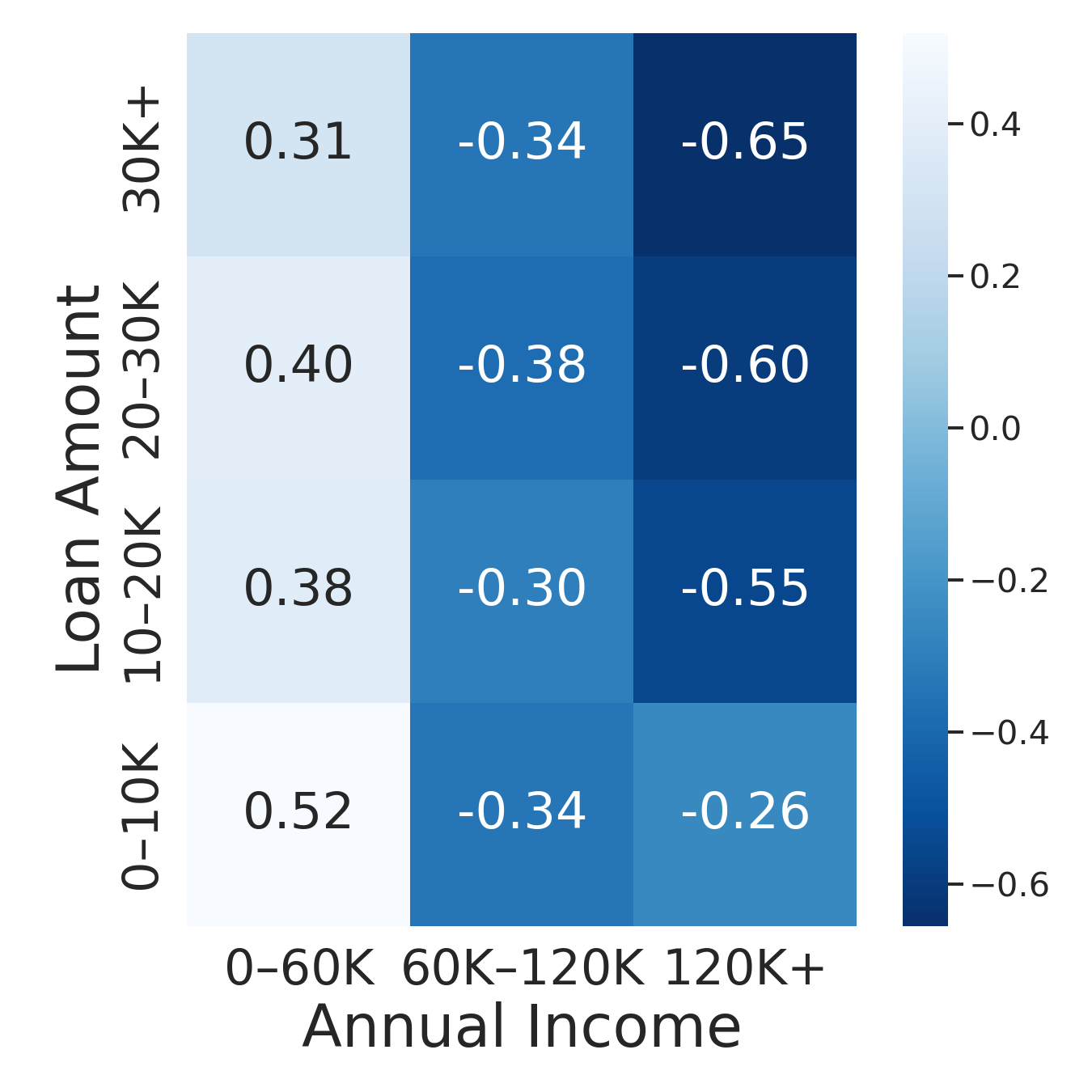}
\vspace{-0.6cm}
      \caption{\small Annual Income}
      \label{fig:interaction2}
    \end{subfigure}
\vspace{-0.2cm}
      \caption{Avg.~Shapley under various conditions}
    \label{fig:interaction}
\vspace{-0.3cm}
  \end{minipage}
\end{wrapfigure}

\textbf{Computational efficiency of PRISM.}  In this part, we analyze the computational efficiency of PRISM.
Table~\ref{tab:effciency} summarizes two notions of complexity: \emph{Query Complexity} counts the number of API requests that each method issues, and \emph{LLM Evaluation Complexity} counts the number of instances that each method need to evaluate. In the table, we compare PRISM with 1-shot direct prompting and $n$-shot direct prompting. For PRISM, it needs to calculate Shapley values for $m$ factors one by one. For each factor, it makes $K$ samplings of background set for comparison. Therefore, it has a query and evaluation complexity $\Theta(mK)$. Tabular-PRISM can input multiple $S$ samplings into one query, so it has a query complexity $\Theta(m)$. In practice, we argue that Tabular-PRISM strategy can greatly reduce the time and token cost, as it saves API calls and evaluate multiple instances in one answer. For the actual time cost, in Stroke dataset (with $m=10, K=10$), Tabular-PRISM takes 92.7s for each instance on average under GPT-4.1-mini non-batched API calling. Apple price prediction using PRISM (GPT-4.1, $m=7, K = 5$) takes around 330s on average, due to large amount of queries, long inputs (each factor is a paragraph) and larger model size. This efficiency can be acceptable if the evaluation size is not large.

\begin{table}[h]
\centering
\vspace{-0.2cm}
\small
\label{tab:stroke-tokens}
\begin{tabular}{ccc}
\hline
\textbf{Metric} & \textbf{Query Complexity} &\textbf{LLM Evaluation Complexity} \\
\hline
1-shot & $\Theta(1)$ & $\Theta(1)$\\ 
$n$-shot & $\Theta(n)$ & $\Theta(n)$\\
Tabular-PRISM & $\Theta(m)$& $\Theta(mK)$\\
PRISM &   $\Theta(mK)$ & $\Theta(mK)$\\
\hline
\end{tabular}
\vspace{-0.2cm}
\caption{Complexity of PRISM and baselines.}
\vspace{-0.3cm}
\label{tab:effciency}
\end{table}

\vspace{-0.3cm}
\section{Conclusion and Limitations}
\vspace{-0.4cm}

In this work, we propose Probability Reconstruction via Shapley Measures (RPSIM) for LLM-based probability estimation tasks. 
In our experiment, we empirically validate its predictive accuracy across multiple benchmark datasets, demonstrating the reliability of the proposed approach.
Compared to direct LLM prompting, PRISM provides enhanced explainability and transparency, thereby enabling more trustworthy use of LLM predictions in high-stakes applications. However, our work has a few limitations. First, it only focuses on the zero-shot setting. In practice, there could be historical records or references available to facilitate the prediction. In such few-shot prediction settings, interpretation can become more challenging, as the system need disentangle whether a prediction arises from its own knowledge or from the provided demonstrations. 
Besides, in Section~\ref{sec:ablation}, we examine the efficiency of PRISM and conclude that its cost can be relatively high, making it practical only when the evaluation size is not too large.

\vspace{-0.4cm}
\section{Reproducibility statement}
\vspace{-0.3cm}
We release an anonymous repository with full source code and our processed datasets (if they are not publicly available) at \href{https://anonymous.4open.science/r/prism-62B5/}{https://anonymous.4open.science/r/prism-62B5/}. In Appendix~\ref{app:algo_and_thm}, we provide the detailed algorithm sketches and theorem proof. 
Appendix~\ref{sec:addl-datasets} and Appendix~\ref{sec:addl-baselines} provide a complete description of our data pre-processing pipelines and the mentioned baselines. The experimental setup, including hyperparameters, model configurations are introduced in the main text. Appendix~\ref{sec:addl-results} contains additional examples, results and interpretations to aid further verification. Appendix~\ref{sec:prompts} lists the exact prompts used for all language model components. 

\bibliography{reference}
\bibliographystyle{unsrtnat}

\appendix

\section{Theory and Algorithm}\label{app:algo_and_thm}
\subsection{Detailed Algorithm}\label{app:algo}
\label{sec:algo_overview}

Guided by the Shapley additivity property (Property~\ref{property1}), PRISM estimates each factor's marginal effect via paired contrasts (realized vs.\ baseline value of that factor), averages these effects over randomly sampled contexts, and then \emph{reconstructs} the model output, finally mapping to a calibrated probability.

\textbf{Setup.}
Let $\mathcal{I}=\{1,\dots,m\}$ index factors, $x=(x_1,\ldots,x_m)$ be the instance,
$b=(b_1,\ldots,b_m)$ be designated baselines ,
$f$ be an evaluation oracle that returns either a \emph{probability} in $[0,1]$ or a \emph{logit} in $\mathbb{R}$ when only a subset of factors is revealed,
and $\Pi_i$ be a distribution over background sets $S\subseteq\mathcal{I}\setminus\{i\}$.
Given $S$, we write $x_S$ for the partial specification that reveals $\{x_j:j\in S\}$.

\begin{algorithm}[h]
\caption{PRISM (Probability Reconstruction via Shapley Measures)}
\label{alg:prism}
\begin{algorithmic}[1]
\Require Instance $x\in\mathcal{X}^m$; oracle $f$; sampling distributions $\{\Pi_i\}_{i=1}^m$; budget $K$.
\Ensure Probability $\hat{p}(x)$ and factor attributions $\{\hat{\phi}_i(x)\}$.
\State $\phi_0 \gets f(x_\emptyset)$
\For{$i \in \mathcal{I}$}
    \State $\Delta \gets 0$
    \For{$k=1$ to $K$}
        \State Sample $S \sim \Pi_i$
        \State $\Delta \gets \Delta + \big(f(x_{S\cup\{i\}}) - f(x_S)\big)$ 
    \EndFor
    \State $\hat{\phi}_i(x) \gets \Delta / K$
\EndFor
\State $\hat{z}(x) \gets \phi_0 + \sum_i \hat{\phi}_i(x)$,\quad $\hat{p}(x) \gets \sigma(\hat{z}(x))$
\State \textbf{return} $\hat{p}(x)$ and $\{\hat{\phi}_i(x)\}$
\end{algorithmic}
\end{algorithm}

\begin{algorithm}[h]
\caption{Tabular-PRISM (Batched realized vs.\ baseline contrasts)}
\label{alg:tabular_prism}
\begin{algorithmic}[1]
\Require Instance $x\in\mathcal{X}^m$; oracle $f$; designated baselines $b$; sampling distributions $\{\Pi_i\}_{i=1}^m$; budget $K$.
\Ensure Probability $\hat{p}(x)$ and factor attributions $\{\hat{\phi}_i(x)\}$.
\State $\phi_0 \gets f(x_\emptyset)$
\For{$i \in \mathcal{I}$}
    \State For $k=1,\dots,K$, sample $S_i^{(k)} \sim \Pi_i$
    \State \textbf{Batch-Query:} evaluate all pairs
    \[
      \big\{ f(x_{S_i^{(k)}\cup\{i\}}|i{:=}x_i),\ f(x_{S_i^{(k)}\cup\{i\}}|i{:=}b_i) \big\}_{k=1}^K
    \]
    \State $\hat{\phi}_i(x)\gets \tfrac{1}{K}\sum_{k=1}^K \Big(f(x_{S_i^{(k)}\cup\{i\}}|i{:=}x_i)-f(x_{S_i^{(k)}\cup\{i\}}|i{:=}b_i)\Big)$
\EndFor
\State $\hat{z}(x) \gets \phi_0 + \sum_i \hat{\phi}_i(x)$,\quad $\hat{p}(x) \gets \sigma(\hat{z}(x))$
\State \textbf{return} $\hat{p}(x)$ and $\{\hat{\phi}_i(x)\}$
\end{algorithmic}
\end{algorithm}

\textbf{Explanation for Alg.~\ref{alg:prism} .}
(1) Query the oracle on the empty specification to obtain the intercept $\phi_0=f(x_{\emptyset})$.
(2--3) Start looping over factors $i\in\mathcal I$ and initialize the accumulator $\Delta\leftarrow 0$ for factor $i$.
(4--6) For $k=1,\ldots,K$, draw a background set $S\sim\Pi_i$ and accumulate the presence/absence contrast
$f(x_{S\cup\{i\}})-f(x_S)$ into $\Delta$.
(7) Average the $K$ contrasts to estimate the contribution of factor $i$: $\hat\phi_i(x)\leftarrow \Delta/K$.
(8) Reconstruct the score by additivity, $\hat z(x)=\phi_0+\sum_i\hat\phi_i(x)$, and map through the logistic link to get the probability $\hat p(x)=\sigma(\hat z(x))$.
(9) Return the probability and the per-factor attributions, i.e., $\hat p(x)$ and $\{\hat\phi_i(x)\}$.

\textbf{Explanation for Alg.~\ref{alg:tabular_prism} .}
(1) Obtain the intercept by querying $f(x_{\emptyset})$ so $\phi_0=f(x_{\emptyset})$.
(2--3) For each factor $i$, sample $K$ background contexts $S_i^{(k)}\sim\Pi_i$ to form the evaluation batches.
(4) In each sampled context, evaluate a realized/baseline pair in batch:
$f(x_{S_i^{(k)}\cup\{i\}}\!\mid i:=x_i)$ and $f(x_{S_i^{(k)}\cup\{i\}}\!\mid i:=b_i)$.
(5) Average the $K$ realized–baseline differences to obtain
$\hat\phi_i(x)=\tfrac{1}{K}\sum_{k=1}^K\big(f(x_{S_i^{(k)}\cup\{i\}}\!\mid i:=x_i)-f(x_{S_i^{(k)}\cup\{i\}}\!\mid i:=b_i)\big)$.
(6) Reconstruct the score $\hat z(x)=\phi_0+\sum_i\hat\phi_i(x)$ and apply the logistic link to produce $\hat p(x)=\sigma(\hat z(x))$.
(7) Return $\hat p(x)$ together with the attributions $\{\hat\phi_i(x)\}$.

\subsection{Theorem proof}\label{app:thm}


\mainprop*

\begin{proof}
We use the permutation form of the Shapley value, which is equivalent to the subset form. Let $\pi$ be a permutation of $\mathcal{I}$, and let $\Pi(\mathcal{I})$ be the set of all $m!$ permutations. For a permutation $\pi$ and an index $i$, let $\mathrm{Pre}_i(\pi)$ denote the set of features that appear before $i$ in $\pi$. The permutation form of the Shapley value of $i$ for the game $v_r(S) = f\big([x_S,\; r_{\bar S}]\big)$ can be written as

\begin{align}\label{eq:permutation-shap}
\phi_i^{(r)}=
\frac{1}{m!}\sum_{\pi\in\Pi(\mathcal{I})}
\Big(
v_r\big(\mathrm{Pre}_i(\pi)\cup\{i\}\big)
-
v_r\big(\mathrm{Pre}_i(\pi)\big)
\Big). 
\end{align}

Taking equation (\ref{eq:permutation-shap}) to the $\sum_{i=1}^m \phi_i^{(r)}$ part yields
\begin{align}\label{eq:permutation-shap-sum}
\sum_{i=1}^m \phi_i^{(r)}
=
\frac{1}{m!}\sum_{\pi\in\Pi(\mathcal{I})}
\sum_{i=1}^m
\Big(
v_r\big(\mathrm{Pre}_i(\pi)\cup\{i\}\big)
-
v_r\big(\mathrm{Pre}_i(\pi)\big)
\Big).
\end{align}

For a fixed permutation $\pi=(\pi_1,\pi_2,\dots,\pi_m)\in\Pi(\mathcal{I})$, we have $\mathrm{Pre}_i(\pi)\cup\{i\} = \mathrm{Pre}_{i+1}(\pi)$ for $1 \leq i \leq m-1$. Therefore
\begin{align}\label{eq:permutation-sum-cancel}
\sum_{i=1}^m
\Big(
v_r\big(\mathrm{Pre}_i(\pi)\cup\{i\}\big)
-
v_r\big(\mathrm{Pre}_i(\pi)\big)
\Big)
=
v_r(\mathcal{I}) - v_r(\emptyset).
\end{align}
Substituting equation (\ref{eq:permutation-sum-cancel}) into (\ref{eq:permutation-shap-sum}) gives
\begin{align}
\sum_{i=1}^m \phi_i^{(r)}
=
\frac{1}{m!}\sum_{\pi\in\Pi(\mathcal{I})}\big(v_r(\mathcal{I})-v_r(\emptyset)\big)
=
v_r(\mathcal{I})-v_r(\emptyset).
\end{align}
Finally, given $\phi_0^{(r)}:=v_r(\emptyset)$ and $v_r(\mathcal{I}) = f\big([x_\mathcal{I},\; r_{\emptyset}]\big) = f(x)$, we have
\begin{align}
f(x) = v_r(\mathcal{I}) = \phi_0^{(r)} + \sum_{i=1}^m \phi_i^{(r)}.
\end{align}
Proof complete.
\end{proof}

\section{Additional Details for Experiments}

\subsection{Datasets}\label{sec:addl-datasets}
\textbf{Stroke}: The stroke prediction dataset contains health, demographic, and lifestyle information for 5,110 patients, with the goal of predicting stroke occurrence. Each record includes variables such as age, hypertension, heart disease, marital status, work type, body mass index (BMI) and so on. 

\textbf{Heart Disease}: The heart disease dataset integrates multiple clinical heart disease datasets and contains 918 patient records with 11 features related to demographics, clinical measurements, and lifestyle factors. The target variable indicates the presence or absence of heart disease.

\textbf{Adult Census}: The adult census contains 48,842 records from the 1994 U.S. Census, with 14 demographic and employment-related features such as age, education, occupation, work hours, and marital status. The target variable indicates whether an individual’s annual income exceeds \$50,000.

\textbf{Lending}: The lending dataset contains peer-to-peer loans issued from 2007--2018, with borrower and loan features such as income, debt-to-income ratio, FICO score, interest rate, loan amount, and purpose. The target variable indicates whether the loan will default.

Before applying above datasets for evaluation, we preprocess the datasets as follows:

(1) In order to make LLMs better understand the datasets, we rename the names of columns of datasets. For example, for \textbf{stroke} dataset, we rename ``avg\_glucose\_level'' as ``average glucose level" and rename ``bmi'' as ``Body Mass Index''.

(2) For some columns, the values may be some abbreviation or with unclear meanings. For example, the values of attribute ``gender'' in dataset \textbf{stroke} are $\{0,1\}$, so we convert ``0'' into ``Female'' and ``1'' into ``Male''. Besides, for attribute ``Chest Pain Type'' in dataset \textbf{heart disease}, the values are abbreviations such as ``ATA'', ``NAP''. We also convert them into their full names, e.g., ``ATA'' $\rightarrow$ ``Atypical Angina'' and ``NAP'' $\rightarrow$ ``Non-Anginal Pain''.

(3) We dropped some attributes of the datasets so that them LLMs can evaluation each data point with more efficiency.

The final data points examples and columns of each dataset are summarized in Figure (\ref{fig:data_example}).

\begin{figure}[htbp]
  \centering
  \begin{subfigure}{0.24\textwidth}
    \centering
    \includegraphics[width=\linewidth]{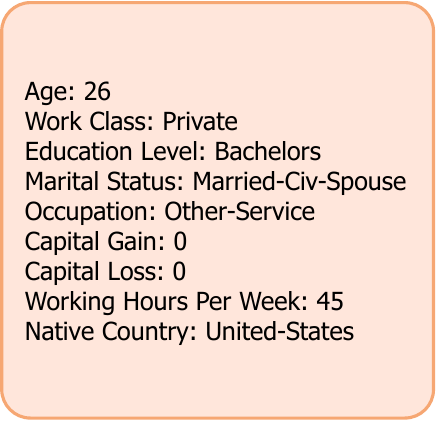}
    \caption{Adult Census}
  \end{subfigure}\hfill
  \begin{subfigure}{0.24\textwidth}
    \centering
    \includegraphics[width=\linewidth]{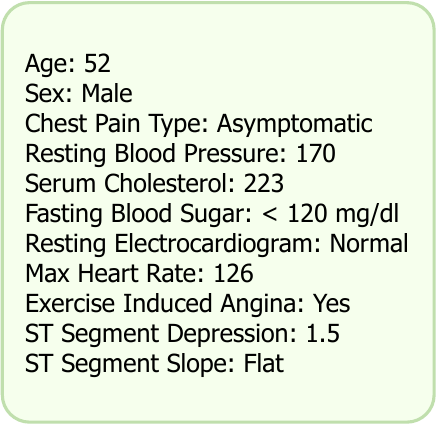}
    \caption{Heart Disease}
  \end{subfigure}\hfill
  \begin{subfigure}{0.24\textwidth}
    \centering
    \includegraphics[width=\linewidth]{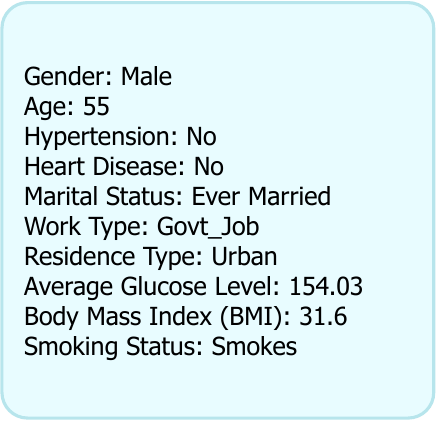}
    \caption{Stroke}
  \end{subfigure}\hfill
  \begin{subfigure}{0.24\textwidth}
    \centering
    \includegraphics[width=\linewidth]{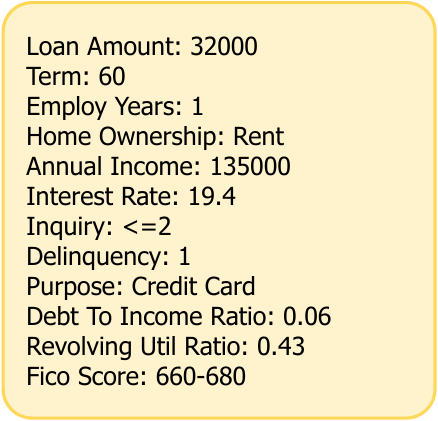}
    \caption{Lending}
  \end{subfigure}
    \caption{Data examples of the datasets.}
    \label{fig:data_example}
\end{figure}

\subsection{Baseline Methods}\label{sec:addl-baselines}

\textbf{1/5/10shot\_level}: We perform multiple shots or multiple trials on directly asking LLM for probability as in prompt Figure (\ref{fig:prompt-nshot_level}). The LLM is instructed to select a linguistic probability description to represent the probability. Before evaluation, we map the descriptions into numerical values as follows: ``very unlikely": 0.05, ``unlikely": 0.2, ``somewhat unlikely": 0.35, ``neutral": 0.5, ``somewhat likely": 0.65, ``likely": 0.8, ``very likely": 0.95. For 1shot\_level, we only perform one trial; for 5shot\_level and 10shot\_level we perform 5 and 10 trials, and select the most common value as output.

\textbf{1/5/10shot\_score}: We perform multiple shots similar to 1/5/10shot\_level. The difference is that we instruct the LLM to output the numerical probability directly, as shown in the prompt Figure (\ref{fig:prompt-nshot_score}). We take the average probability as output.

\textbf{ICL-5+5, ICL-10+10}: We perform in-context learning on the 4 datasets with available training data. There are 5 positive 5 negative data in ICL-5+5 and 10 negative 10 positive data in ICL-10+10. The training data are excluding the evaluation data, with large mutual differences from the arbitrary candidate set. The order of the training data inside the prompt is shuffled randomly to minimize confusion. Detailed prompt example is shown as follows \ref{fig:prompt-ICL}.

\textbf{Contrast}: We perform two queries in opposite directions. The first query apply the same prompt as in 1shot\_level Figure (\ref{fig:prompt-nshot_level}). The second query is trying to ask the question in the opposite way, such as "How likely is this patient to NOT have a stroke?" instead of "How likely is this patient to have a stroke?". Two queries are normalized so that they sum up to 1, and the normalized positive answer will be the output.

\textbf{BIRD}: We follow BIRD's method with slight modification. First, since BIRD could not handle numerical datatype features, those features will be turned into bins before treated as input. In order to better utilize prior knowledge of LLM, binning will base on prior evidence that extract the most characteristic of the stage, therefore the interval of each bin may not have the same length. For example, regarding to the numerical feature ``Resting Blood Pressure" in the Heart Disease dataset, the bins are: Normal `80-120', Pre-hypertension `120-130', Hypertension Stage 1 `130-140', Hypertension Stage 2 `140-180', Hypertensive crisis `180-200'. Second, instead of probability mapping in BIRD \{``$f_j$ supports outcome $i$": 75\%, ``$f_j$ is neutral": 50\%, ``$f_j$ supports opposite outcome $\neg i$": 25\%\}, we are using a denser mapping which is the same as introduced in 1shot\_level. 

In general, the baseline BIRD is performed in 4 steps: (1) Initializing the prediction probability given each independent feature, by querying LLM using similar prompt as in 1shot\_level Figure (\ref{fig:prompt-nshot_level}), only replace ``Person information: ..." by ``Given that [\textit{factor}] = [\textit{value}]". (2) Generating stochastic training data, by randomly choice a bin for each feature, the prompt is the same as in 1shot\_level Figure (\ref{fig:prompt-nshot_level}). (3) Training the BIRD constrained optimization method. (4) Inferring to the evaluation dataset.

\subsection{Reference instances settings}
\label{app:anchor}

\begin{table}[!h]
\centering
\small
\begin{tabular}{ll}
\hline
\multicolumn{2}{c}{\textbf{Stroke}}\\
\hline
Gender & Male \\
Age & 40.0 \\
Hypertension & No \\
Heart disease & No \\
Marital status & Never Married \\
Residence type & Rural \\
Average glucose level & 90.0 \\
Body Mass Index (BMI) & 24.0 \\
Work type & Private \\
Smoking status & never smoked \\
\hline
\multicolumn{2}{c}{\textbf{Adult}}\\
\hline
Age & 40 \\
Workclass & Private \\
Education level & Some-college \\
Marital status & Married-civ-spouse \\
Occupation & Sales \\
Capital gain & 0 \\
Capital loss & 0 \\
Working hours per week & 40 \\
Native country & United-States \\
\hline
\multicolumn{2}{c}{\textbf{Loan}}\\
\hline
Loan amount & 20000 \\
Term & 36 \\
Employ years & 3 \\
Home ownership & OWN \\
Annual income & 60000 \\
Interest rate & 14.0 \\
Purpose & car \\
Debt-to-income ratio & 0.35 \\
Revolving util ratio & 0.30 \\
FICO score & 680--710 \\
Inquiry & $\leq 2$ \\
Delinquency & 0 \\
\hline
\multicolumn{2}{c}{\textbf{Heart Disease}}\\
\hline
Age & 53 \\
Resting Blood Pressure & 133 \\
Serum Cholesterol & 212 \\
Max Heart Rate & 137 \\
ST Segment Depression & 0.8 \\
Sex & Male \\
Chest Pain Type & Asymptomatic \\
Fasting Blood Sugar & $< 120$ mg/dl \\
Resting Electrocardiogram & Normal \\
Exercise Induced Angina & No \\
ST Segment Slope & Flat \\
\hline
\end{tabular}
\caption{Reference instance of each dataset.}
\label{tab:anchor}
\end{table}

We instantiate a per-dataset \emph{reference instance} \(x_{\mathrm{ref}}\) that serves as the baseline context for PRISM’s contrastive evaluations. Reference values are chosen to be representative, i.e., close to the empirical mean for continuous variables and a prevalent category for discrete variables (the mode for multi-class features and the negative category for binary features). The concrete instances used in our experiments are listed in Table~\ref{tab:anchor}.

For each dataset's reference instance,, we use prompt \ref{fig:prompt-nshot_score}, query the model five times, average the predicted probabilities, and convert the result to a base logit via \(\mathrm{logit}(p)=\log\!\big(\frac{p}{1-p}\big)\). The resulting base probabilities and logits are reported in Table~\ref{tab:base-logits}.

\begin{table}[h]
\centering
\small
\begin{tabular}{lcc}
\hline
Dataset & \(p\) (mean over 5 runs) & logit \\
\hline
Stroke & 0.001 & -6.9068 \\
Adult & 0.354 & -0.6015 \\
Heart Disease & 0.410 & -0.3640 \\
Loan & 0.182 & -1.5029 \\
\hline
\end{tabular}
\caption{Base probabilities \(p\) (from prompt \ref{fig:prompt-nshot_score}, averaged over five queries) and corresponding base logits per dataset.}
\label{tab:base-logits}
\end{table}




\section{Additional Results}
\label{sec:addl-results}

\subsection{Examination of Interpretations}
\label{subsec:interp-exam}

We examine factor-level attributions on the tabular datasets in \autoref{tab:adult-cases-impacts}, \autoref{tab:heart-cases-impacts}, and \autoref{tab:loan-cases-impacts}, and on the text scenarios in agriculture (\autoref{fig:honey-crisp}, \autoref{fig:granny-smith}) and soccer (\autoref{fig:man_t}).

\textbf{Adult (\autoref{tab:adult-cases-impacts}).}
The attributions align with well-known socio-economic regularities for predicting income.
Education exerts the largest and most consistent influence: \textit{Masters} yields a strong positive contribution (Case 4: \(+1.31\)), while \textit{HS-grad} is negative (Case 3: \(-0.45\)).
Occupational roles are similarly informative: \textit{Exec-managerial} is positive (Cases 1/4: \(+0.49/+0.99\)), whereas \textit{Farming-fishing} is negative (Case 3: \(-0.84\)).
\textit{Capital gain} is highly predictive when present (Case 1: \(+1.04\)), and \textit{age} contributes moderately with the expected direction (older age increasing odds in Cases 1/3).
Marital status exhibits negative impacts for \textit{Divorced} and \textit{Never-married} (Cases 1/2), consistent with prior findings that marriage correlates with higher income.
Some variables (e.g., \textit{workclass}, \textit{native country}) show near-zero effects in these cases, indicating either proximity to the anchor or low marginal power after conditioning on stronger factors.

\begin{table}[h]
\centering
\setlength{\tabcolsep}{4pt}
\resizebox{0.95\linewidth}{!}{%
\begin{tabular}{c |cc| cc| cc| cc}
\hline
Factor & \multicolumn{2}{|c|}{Case 1} & \multicolumn{2}{c|}{Case 2} & \multicolumn{2}{c|}{Case 3} & \multicolumn{2}{c}{Case 4} \\
 & Value & Shapley & Value & Shapley & Value & Shapley & Value & Shapley \\
\hline
Age              & 51   & 0.17  & 25   & -0.41 & 47   & 0.38  & 38   & 0.00 \\
Workclass        & Private & 0.00 & Private & 0.00 & Self-emp-not-inc & -0.44 & Private & 0.00 \\
Education level  & Bachelors & 0.62 & Bachelors & 0.49 & HS-grad & -0.45 & Masters & 1.31 \\
Marital status   & Divorced & -0.47 & Never-married & -0.41 & Married-civ-spouse & 0.00 & Married-civ-spouse & 0.00 \\
Occupation       & Exec-managerial & 0.49 & Sales & 0.00 & Farming-fishing & -0.84 & Exec-managerial & 0.99 \\
Capital gain     & 10520 & 1.04 & 0 & 0.00 & 0 & 0.00 & 0 & 0.00 \\
Capital loss     & 0 & 0.00 & 1876 & 0.00 & 0 & 0.00 & 0 & 0.00 \\
Working hours/wk & 40 & 0.00 & 40 & 0.00 & 60 & 0.54 & 60 & 0.53 \\
Native country   & US & 0.00 & US & 0.00 & US & 0.00 & US & 0.00 \\
\hline
\textit{Sum Shapley} & \multicolumn{2}{|c|}{1.86} & \multicolumn{2}{c|}{-0.33} & \multicolumn{2}{c|}{-0.81} & \multicolumn{2}{c}{2.83} \\
\textit{Sum logit}   & \multicolumn{2}{|c|}{1.26} & \multicolumn{2}{c|}{-0.93} & \multicolumn{2}{c|}{-1.41} & \multicolumn{2}{c}{2.22} \\
\textit{Pred prob}   & \multicolumn{2}{|c|}{0.778} & \multicolumn{2}{c|}{0.283} & \multicolumn{2}{c|}{0.197} & \multicolumn{2}{c}{0.902} \\
\textit{True label}  & \multicolumn{2}{|c|}{Yes}   & \multicolumn{2}{c|}{No}    & \multicolumn{2}{c|}{No}    & \multicolumn{2}{c}{Yes} \\
\hline
\end{tabular}%
}
\vspace{-0.2cm}
\caption{Shapley values for four instances in Adult dataset. Reference instance: \textit{\textcolor{gray}{age=40; workclass=Private; education=Some-college; marital status=Married-civ-spouse; occupation=Sales; capital gain=0; capital loss=0; working hours=40; native country=US}}. A single base logit is shared across cases, \textcolor{gray}{\textit{$\phi_0=\sigma^{-1}(0.354)=-0.6015$}}.}
\vspace{-0.3cm}
\label{tab:adult-cases-impacts}
\end{table}

\textbf{Heart (\autoref{tab:heart-cases-impacts}).}
For heart disease risk, the factor impacts align with real-world clinical patterns
\textit{Exercise-induced angina} is strongly positive (Cases 1/2: \(+0.89/+0.87\)), and elevated \textit{resting blood pressure} and \textit{serum cholesterol} contribute positively (Case 1: \(+0.43/+0.45\); Case 2: \(+0.38/+0.29\)).
Conversely, higher \textit{max heart rate} and \textit{upsloping ST slope} decrease risk (Case 3: \(-0.45\) and \(-0.48\)), aligning with cardiology practice that better exercise capacity and non-flat slopes are protective.
\textit{Sex=Female} is negative in Case 3 (\(-0.53\)), capturing lower risk in females.
Note that the same feature can change sign across cases (e.g., \textit{ST depression}: \(-0.59\) in Case 2 vs.\ near-zero/positive elsewhere), reflecting interactions and the conditional nature of \(\phi_i(x)\) under different covariate settings.

\begin{table}[t]
\centering
\setlength{\tabcolsep}{4pt}
\resizebox{0.95\linewidth}{!}{%
\begin{tabular}{c |cc| cc| cc| cc}
\hline
Factor & \multicolumn{2}{|c|}{Case 1} & \multicolumn{2}{c|}{Case 2} & \multicolumn{2}{c|}{Case 3} & \multicolumn{2}{c}{Case 4} \\
 & Value & Shapley & Value & Shapley & Value & Shapley & Value & Shapley \\
\hline
Age                    & 52  & 0.00 & 58  & 0.00  & 34  & -0.73 & 48  & -0.12 \\
Sex                    & Male & 0.00 & Male & 0.00  & Female & -0.53 & Male & 0.00 \\
Chest Pain Type        & Asymptomatic & 0.00 & Non-Anginal Pain & 0.70 & Atypical Angina & 0.42 & Asymptomatic & 0.00 \\
Resting Blood Pressure & 170 & 0.43 & 150 & 0.38  & 118 & -0.49 & 132 & 0.00 \\
Serum Cholesterol      & 223 & 0.45 & 219 & 0.29  & 210 & 0.00  & 272 & 0.37 \\
Fasting Blood Sugar    & $<120$ mg/dl & 0.00 & $<120$ mg/dl & 0.00 & $<120$ mg/dl & 0.00 & $<120$ mg/dl & 0.00 \\
Resting ECG            & Normal & 0.00 & ST-T abn. & 0.33 & Normal & 0.00 & ST-T abn. & 0.45 \\
Max Heart Rate         & 126 & 0.19 & 118 & 0.00  & 192 & -0.45 & 139 & 0.00 \\
Exercise Induced Angina& Yes  & 0.89 & Yes  & 0.87  & No  & 0.00  & No  & 0.00 \\
ST Segment Depression  & 1.5  & 0.44 & 0.0  & -0.59 & 0.7 & 0.00  & 0.2 & -0.44 \\
ST Segment Slope       & Flat & 0.00 & Flat & 0.00  & Upsloping & -0.48 & Upsloping & -0.44 \\
\hline
\textit{Sum Shapley}  & \multicolumn{2}{|c|}{2.40} & \multicolumn{2}{c|}{1.99} & \multicolumn{2}{c|}{-2.25} & \multicolumn{2}{c}{-0.63} \\
\textit{Sum logit}    & \multicolumn{2}{|c|}{2.04} & \multicolumn{2}{c|}{1.62} & \multicolumn{2}{c|}{-2.61} & \multicolumn{2}{c}{-0.54} \\
\textit{Pred prob}    & \multicolumn{2}{|c|}{0.885} & \multicolumn{2}{c|}{0.835} & \multicolumn{2}{c|}{0.068} & \multicolumn{2}{c}{0.367} \\
\textit{True label}   & \multicolumn{2}{|c|}{Yes}   & \multicolumn{2}{c|}{Yes}   & \multicolumn{2}{c|}{No}   & \multicolumn{2}{c}{No} \\
\hline
\end{tabular}%
}
\vspace{-0.2cm}
\caption{Shapley values for four instances in Heart dataset. Reference instance: \textit{\textcolor{gray}{Age=53; Resting BP=133; Serum Chol.=212; Max HR=137; ST Depression=0.8; Sex=Male; Chest Pain=Asymptomatic; Fasting Blood Sugar=$<120$ mg/dl; Resting ECG=Normal; Exercise Angina=No; ST Slope=Flat}}. A single base logit is shared across cases, \textcolor{gray}{\textit{$\phi_0=\sigma^{-1}(0.41)=-0.363$}}.}
\vspace{-0.3cm}
\label{tab:heart-cases-impacts}
\end{table}

\textbf{Loan (\autoref{tab:loan-cases-impacts}).}
For default risk, the patterns are intuitive.
High \textit{interest rate} increases risk (Cases 3/4: \(+0.67/+0.71\)), while lower rates reduce it (Case 1: \(-0.50\)).
Past \textit{delinquency} is the single most influential positive factor when present (Case 3: \(+1.15\)).
\textit{Home ownership=RENT} and lower \textit{annual income} tend to raise risk (Case 3: \(+0.44/+0.48\)), while higher \textit{FICO} mitigates risk and lower \textit{FICO} elevates it (Case 1: \(-0.41\) vs.\ Case 3: \(+0.45\)).
Debt burden is captured by \textit{DTI} and \textit{revolving utilization}: lower values reduce risk (Case 1: \(-0.08\) and \(-0.51\)), whereas moderate-to-high levels are less favorable across other cases.

\begin{table}[h]
\centering
\setlength{\tabcolsep}{4pt}
\resizebox{0.95\linewidth}{!}{%
\begin{tabular}{c |cc| cc| cc| cc}
\hline
Factor & \multicolumn{2}{|c|}{Case 1 (ID=1)} & \multicolumn{2}{c|}{Case 2 (ID=2)} & \multicolumn{2}{c|}{Case 3 (ID=29)} & \multicolumn{2}{c}{Case 4 (ID=37)} \\
 & Value & Shapley & Value & Shapley & Value & Shapley & Value & Shapley \\
\hline
Loan amount          & 30000 & 0.45 & 1500  & 0.00 & 10000 & -0.38 & 30000 & 0.48 \\
Term                 & 60    & 0.45 & 36    & 0.00 & 36    & 0.00  & 60    & 0.51 \\
Employ years         & 4     & -0.43 & $<1$ & 0.45 & 6     & -0.43 & 10+   & -0.55 \\
Home ownership       & MORTGAGE & 0.00 & MORTGAGE & 0.00 & RENT & 0.44 & MORTGAGE & 0.42 \\
Annual income        & 65000 & -0.27 & 55000 & 0.00 & 48000 & 0.48 & 42000 & 0.46 \\
Interest rate        & 12.0  & -0.50 & 10.4  & -0.43 & 20.0  & 0.67 & 19.4  & 0.71 \\
Inquiry              & $\leq 2$ & 0.00 & $\leq 2$ & 0.00 & $\leq 2$ & 0.00 & $\leq 2$ & 0.00 \\
Delinquency          & 0     & 0.00 & 0     & 0.00 & 1     & 1.15 & 0     & 0.00 \\
Purpose              & Debt cons. & 0.43 & Home improv. & 0.00 & Debt cons. & 0.44 & Debt cons. & 0.48 \\
Debt-to-income ratio & 0.25  & -0.08 & 0.13  & -0.55 & 0.29  & 0.00 & 0.17  & 0.00 \\
Revolving util ratio & 0.21  & -0.51 & 0.29  & 0.00 & 0.08  & -0.47 & 0.44  & 0.00 \\
FICO score           & 710--740 & -0.41 & 710--740 & -0.77 & 660--680 & 0.45 & 660--680 & 0.42 \\
\hline
\textit{Sum Shapley}  & \multicolumn{2}{|c|}{-0.87} & \multicolumn{2}{c|}{-1.30} & \multicolumn{2}{c|}{2.35} & \multicolumn{2}{c}{2.93} \\
\textit{Sum logit}    & \multicolumn{2}{|c|}{-2.37} & \multicolumn{2}{c|}{-2.81} & \multicolumn{2}{c|}{0.85} & \multicolumn{2}{c}{1.43} \\
\textit{Pred prob}    & \multicolumn{2}{|c|}{0.086} & \multicolumn{2}{c|}{0.057} & \multicolumn{2}{c|}{0.701} & \multicolumn{2}{c}{0.806} \\
\textit{True label}   & \multicolumn{2}{|c|}{No}    & \multicolumn{2}{c|}{No}    & \multicolumn{2}{c|}{Yes}   & \multicolumn{2}{c}{Yes} \\
\hline
\end{tabular}%
}
\vspace{-0.2cm}
\caption{Shapley values for four instances in Loan dataset. Reference instance: \textit{\textcolor{gray}{Loan amount=20000; Term=36; Employ years=3; Home ownership=OWN; Annual income=60000; Interest rate=14.0; Purpose=car; Debt-to-income ratio=0.35; Revolving util ratio=0.30; FICO score=680--710; Inquiry=$\leq 2$; Delinquency=0}}. A single base logit is shared across cases, \textcolor{gray}{\textit{$\phi_0=\sigma^{-1}(0.182)=-1.504$}}.}
\vspace{-0.3cm}
\label{tab:loan-cases-impacts}
\end{table}

Across these tabular datasets, predicted probabilities exhibit a label-consistent ordering with clear separation—for example, the positives cluster at higher values while the negatives fall into intermediate and low ranges—yielding an easily interpretable ranking. Per-instance factor impacts make the drivers of each estimate explicit and support auditability: dominant contributors can be inspected, implausible probabilities traced to specific factors, and questionable cases flagged for review.

\textbf{Agriculture (\autoref{fig:honey-crisp}, \autoref{fig:granny-smith}).}
For Honeycrisp, \textit{Production} contributes positively because a forecast output decline tightens supply, and this effect becomes stronger when \textit{Storage} carryover is elevated and \textit{Market demand} is soft. \textit{Costs} add a small positive push as labor and inputs remain high, while \textit{Varietal competition} subtracts; \textit{Government policy} offers limited support. Together these effects yield a predicted probability of \(0.51\) for a price increase.
For Granny Smith, the signs invert: an expected \textit{Production} increase drives a negative contribution, which becomes larger under \textit{Imports \& exports} pressure and aggressive \textit{Promotional activity}. The small positive \textit{Government policy} term does not offset oversupply, and \textit{Costs} together with \textit{Varietal competition} further weigh on price. In this setting the predicted probability for a price increase is \(0.09\). Across the two varieties, these attributions accord with well-known agricultural price formation regularities in which supply shocks and carryover inventory dominate short-run price movements, moderated by demand conditions, policy, and competing varieties.

\begin{figure}
    \centering
    \includegraphics[width=0.9\linewidth]{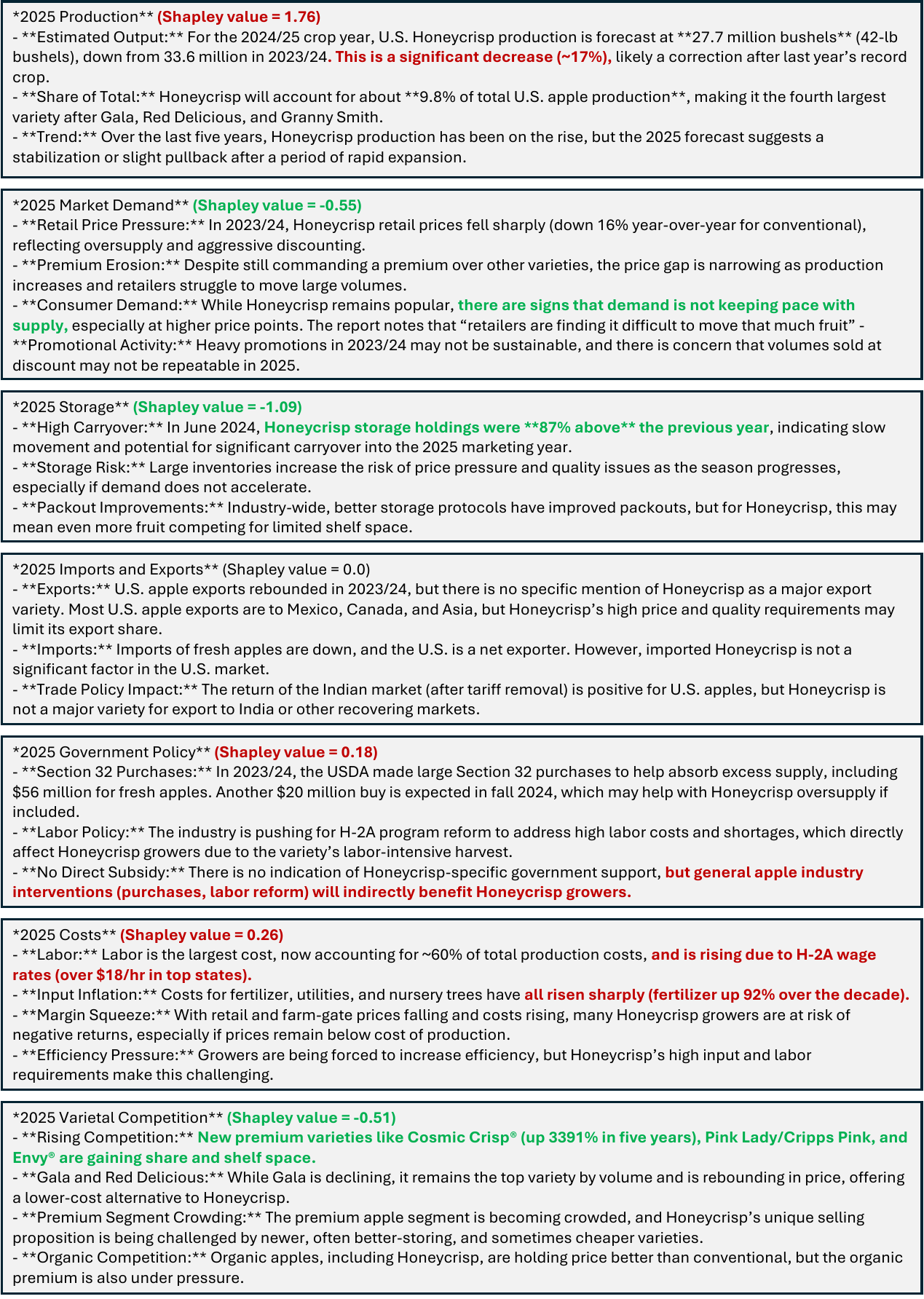}
    \caption{The factors and interpretations of PRISM for predicting whether the price of Honeycrisp apple will increase in 2025. Based on the Shapley values, it is finally predicted to have a chance of 51\% to have price increase. Among the factors, ``Production'' has a large positive indication of price increase, as the production is expected to decrease by 17\%.}
    \label{fig:honey-crisp}
\end{figure}

\begin{figure}
    \centering
    \includegraphics[width=0.89\linewidth]{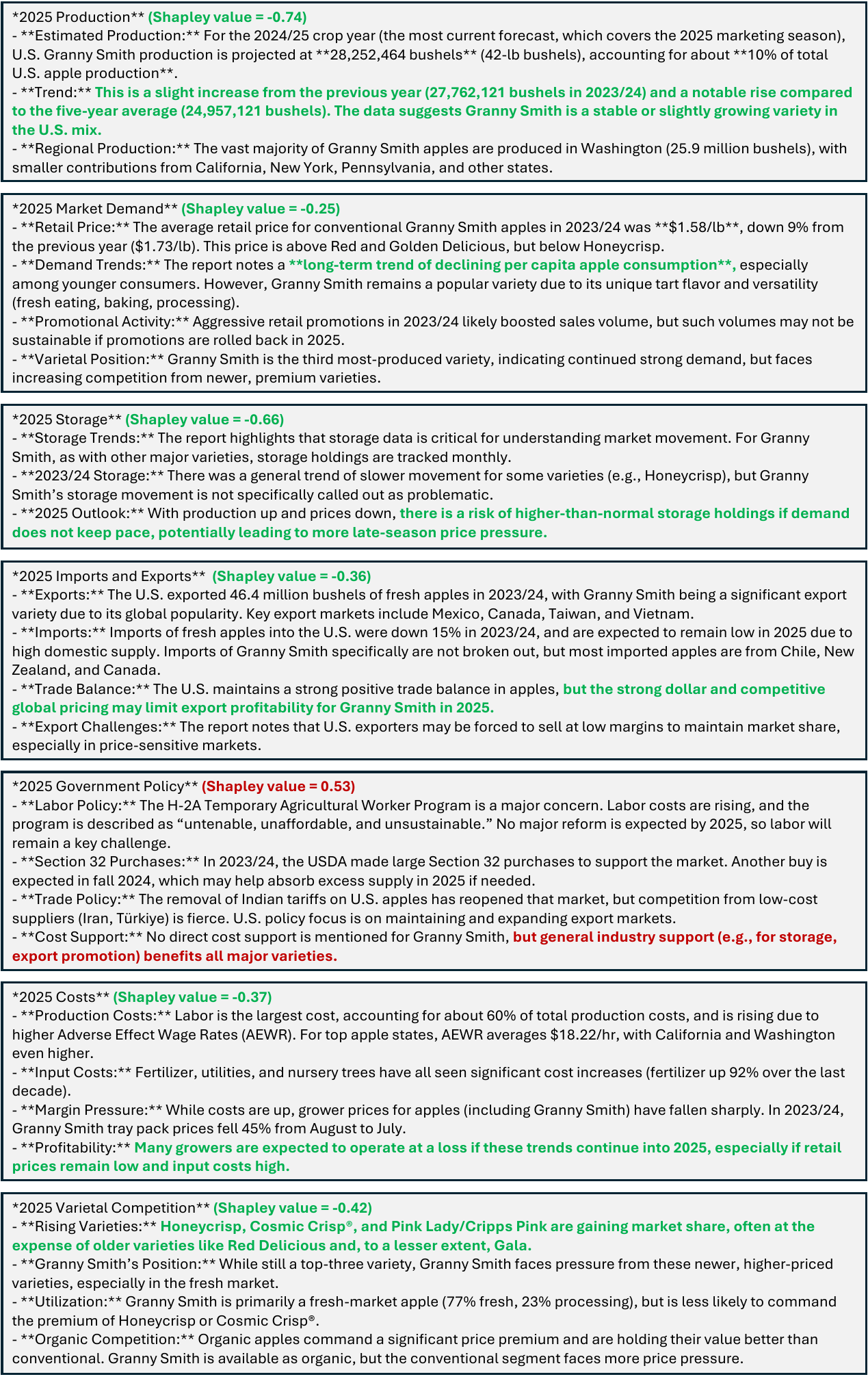}
    \caption{The factors and interpretations of PRISM for predicting whether the price of Granny Smith apple will increase in 2025. Based on the Shapley values, it is finally predicted to have a chance of 9\% to have price increase. Among the factors, only the ``Government policy'' is positive. In the negative factors, the production of Granny Smith is expected to increase, so it lowers the expectation of price increase.}
    \label{fig:granny-smith}
\end{figure}

\begin{figure}
    \centering
    \includegraphics[width=0.85\linewidth]{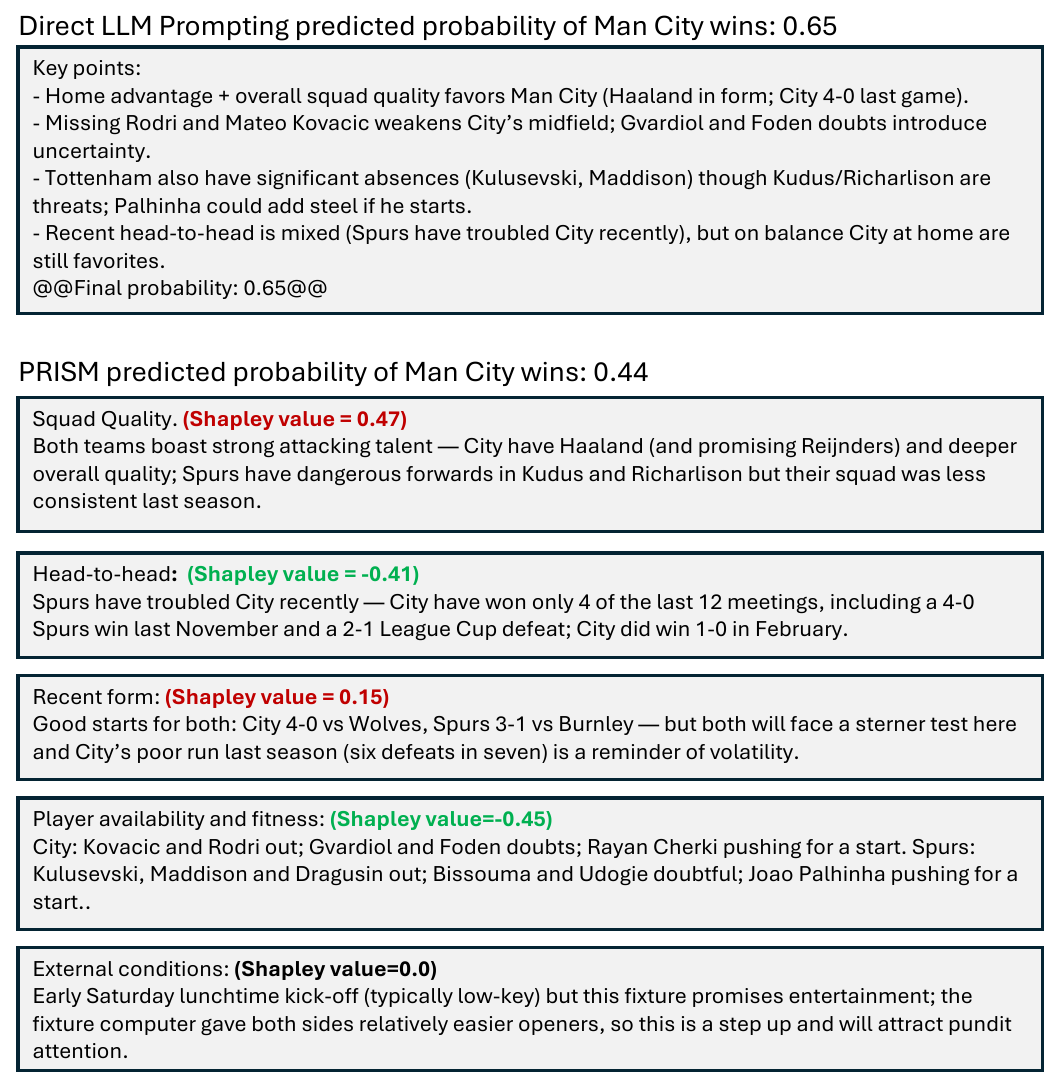}
    \caption{Direct LLM Prompting and PRISM for estimating the probability of ``Man City will beat Tottenham at Man City's Home''. In general, Man City is a stronger team as it has a better squad. Direct LLM prompting yields facts similar to the factors used by PRISM, but concludes that Man City are favorites. This suggests that the LLM may rely on impressions to assign stronger teams higher winning probability.
    For PRISM, it also considers the factor ``Squad Qualify'' which favors Man City, but other factors such as head-to-head records and player availability lead PRISM give a lower winning probability for Man City—0.44. The result of the match is Man City loses.  }
    \label{fig:man_t}
\end{figure}


\textbf{Soccer (\autoref{fig:man_t}).}
PRISM decomposes the match into conditional cues.
\textit{Home advantage} and \textit{Squad quality} contribute positively (the latter is the largest, reflecting City’s deeper roster), while \textit{Head-to-head} contributes negatively (City have been troubled by Spurs) and \textit{Player availability and fitness} is strongly negative given absences in midfield and doubts in key positions; \textit{Recent form} is mildly positive and \textit{External conditions} are neutral.
These effects interact: the benefit of home advantage weakens when midfield anchors are missing, and the head-to-head penalty matters more when overall form is mixed.
Balancing these factors yields a PRISM predicted probability of \(0.44\) for a City win at home (vs.\ \(0.65\) from direct LLM scoring), and the realized outcome—City lost—aligns more closely with PRISM’s assessment.


Across agriculture and soccer, PRISM's factor impacts are conditional rather than global: identical features switch sign or magnitude as the surrounding evidence changes.
This conditionality captures interaction structure among drivers, clarifies why conclusions can differ across otherwise similar factor sets, and yields interpretable, context-aware attributions aligned with domain regularities.

\subsection{Calibration Analysis}\label{sec:calibration}

\begin{figure}[htbp]
    \centering
    
    \begin{subfigure}{0.35\textwidth}
        \centering
        \includegraphics[width=\linewidth]{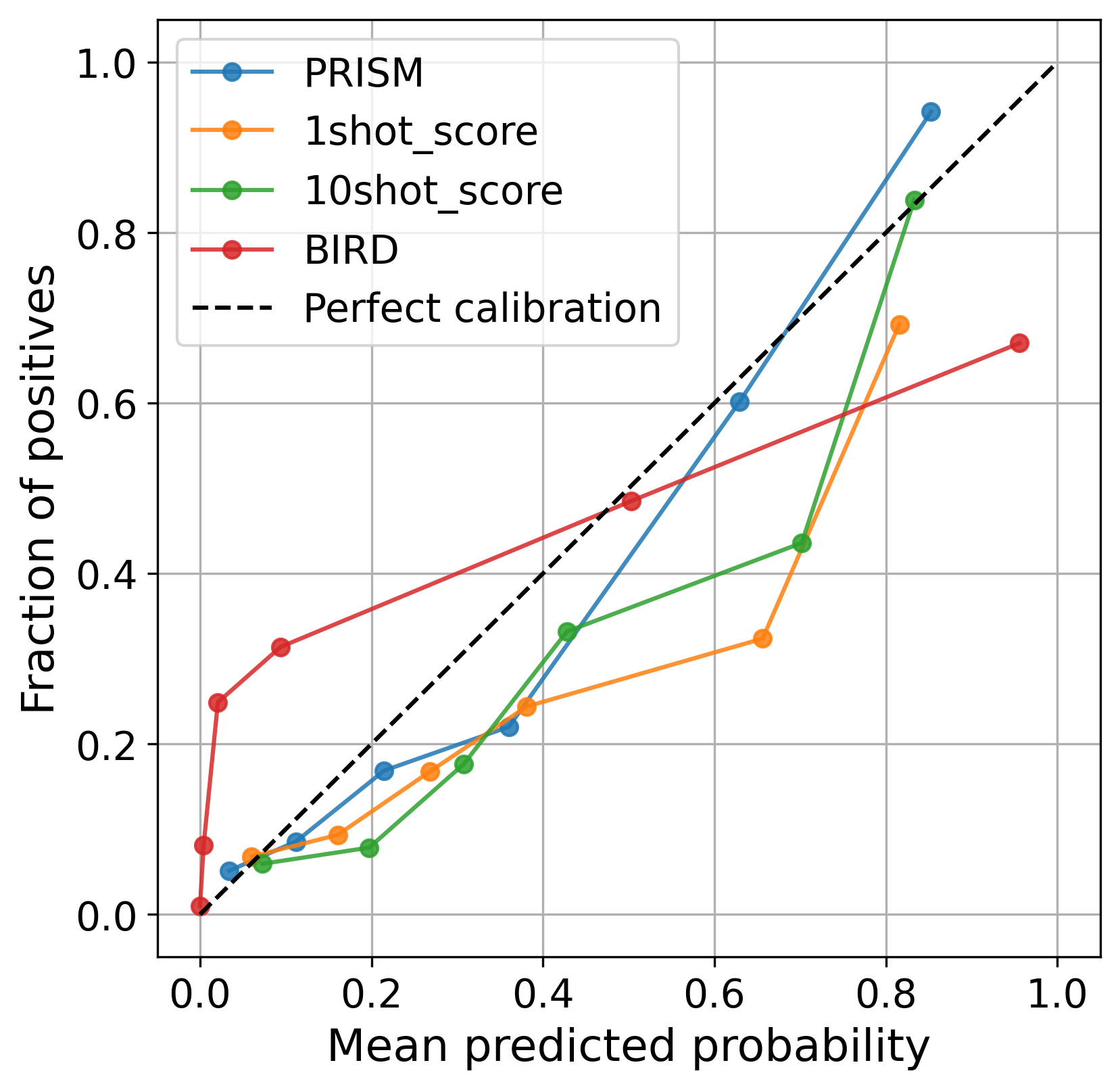}
        \caption{Adult (GPT)}
    \end{subfigure}
    \begin{subfigure}{0.35\textwidth}
        \centering
        \includegraphics[width=\linewidth]{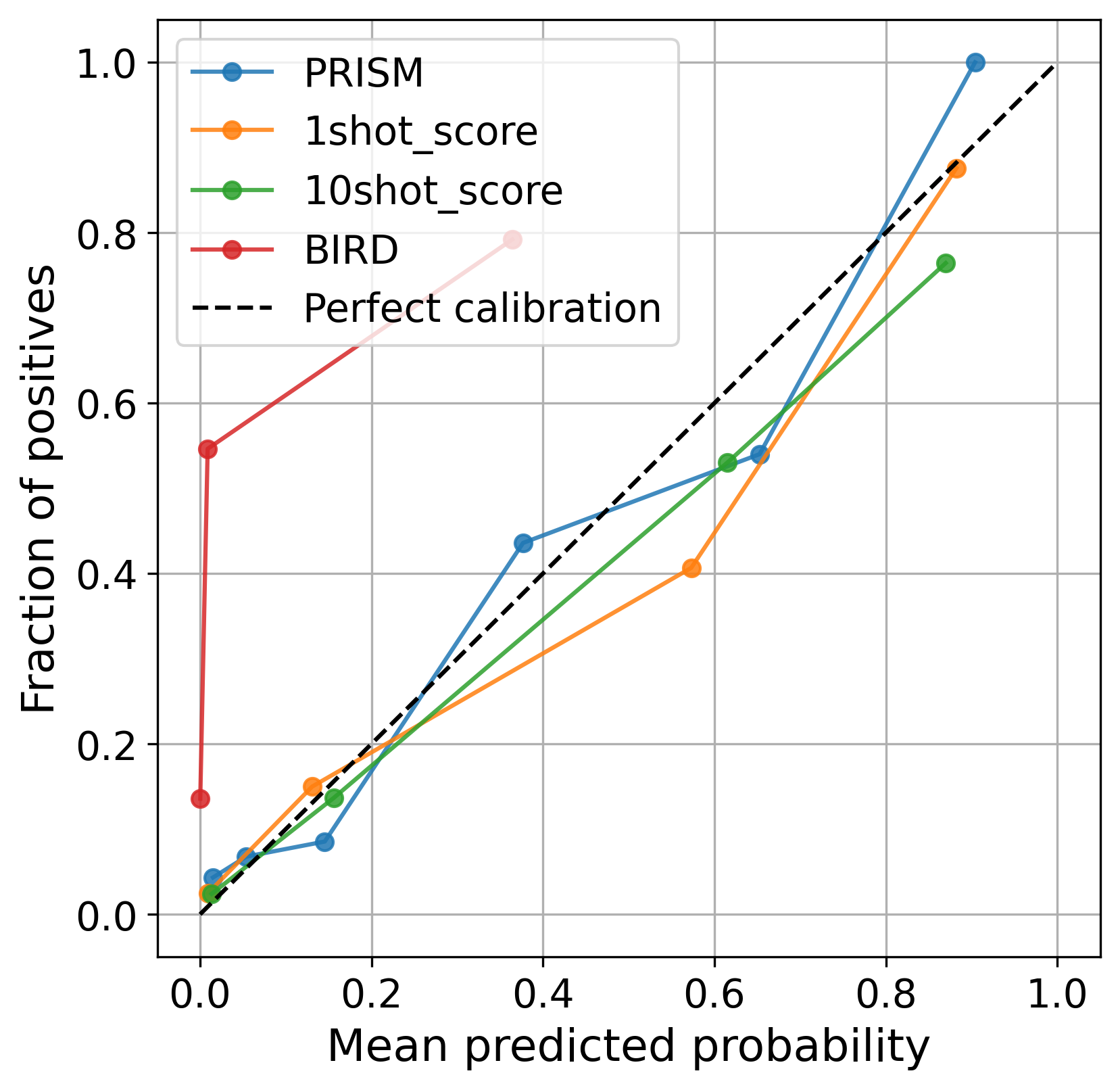}
        \caption{Adult (Gemini)}
    \end{subfigure}
    
    \begin{subfigure}{0.35\textwidth}
        \centering
        \includegraphics[width=\linewidth]{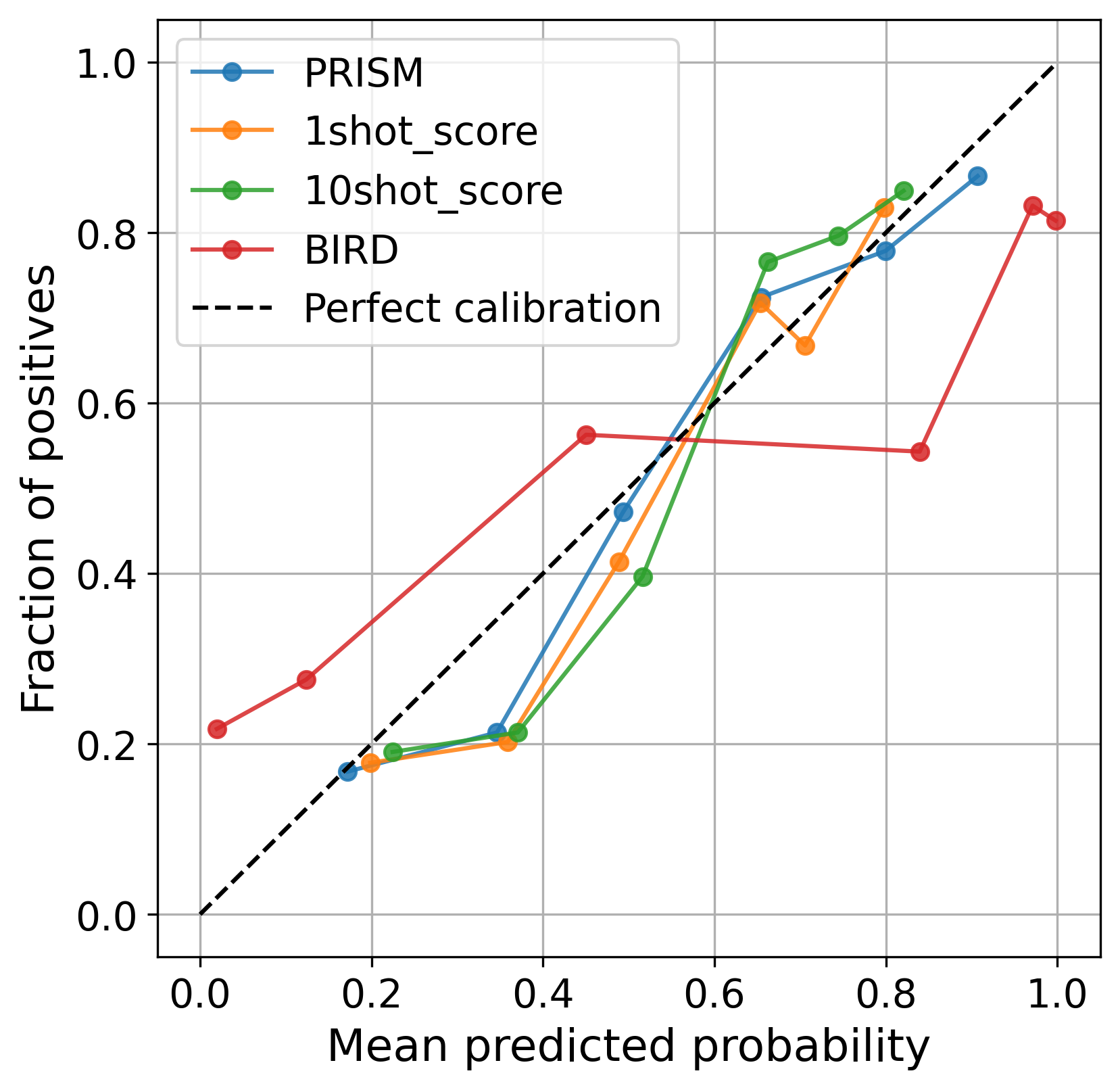}
        \caption{Heart (GPT)}
    \end{subfigure}
    \begin{subfigure}{0.35\textwidth}
        \centering
        \includegraphics[width=\linewidth]{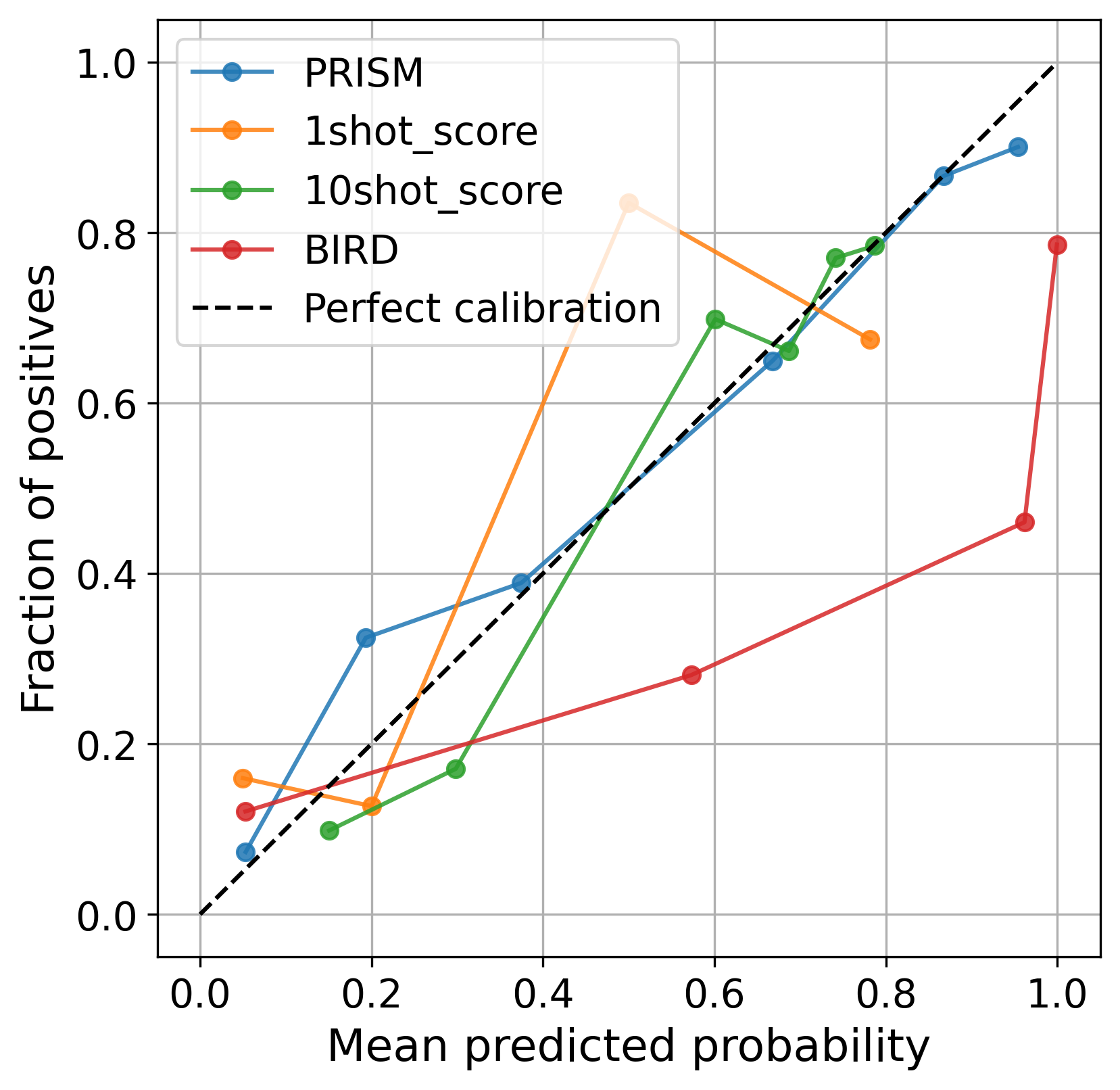}
        \caption{Heart (Gemini)}
    \end{subfigure}
    
    \begin{subfigure}{0.35\textwidth}
        \centering
        \includegraphics[width=\linewidth]{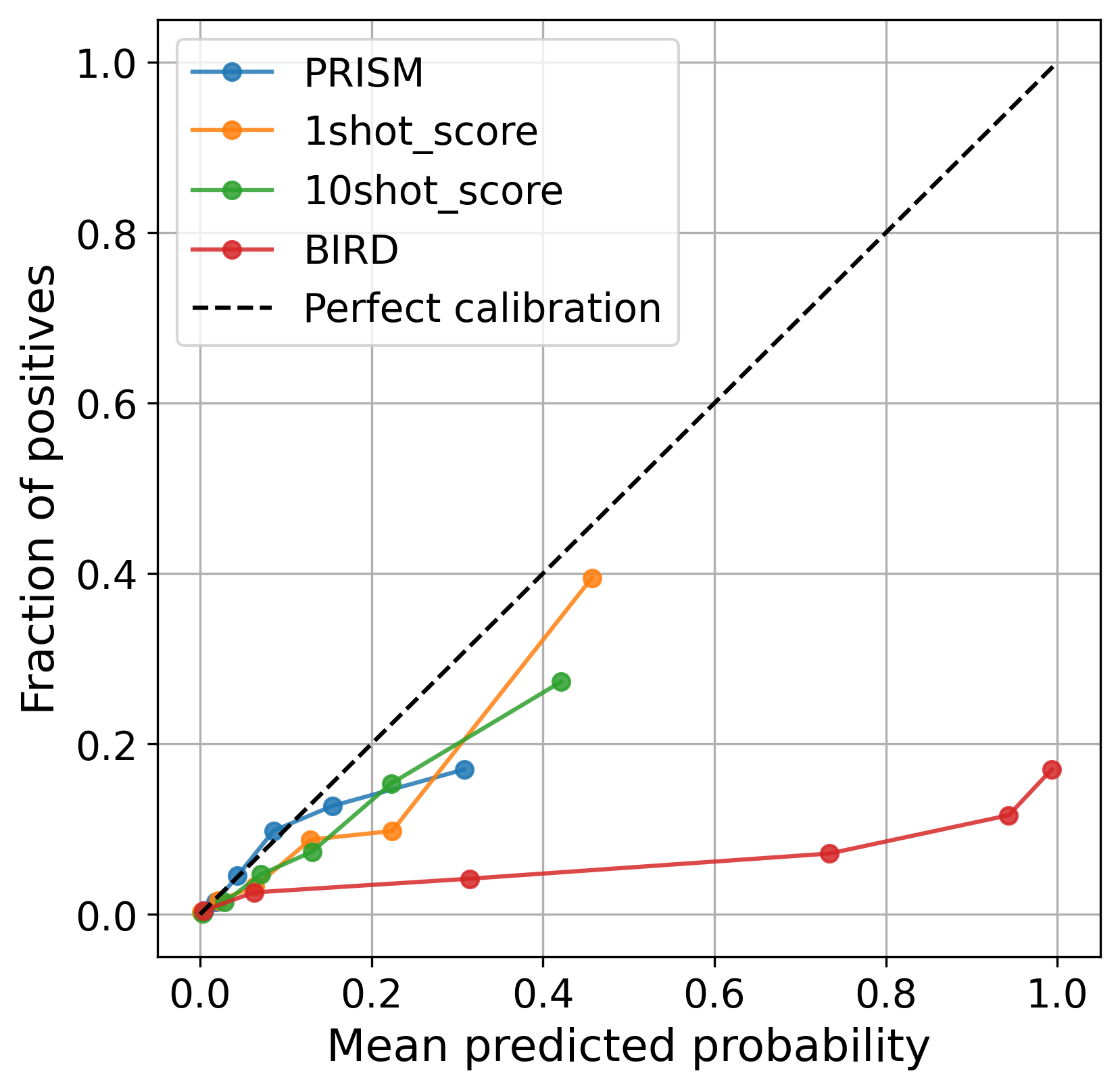}
        \caption{Stroke (GPT)}
    \end{subfigure}
    \begin{subfigure}{0.35\textwidth}
        \centering
        \includegraphics[width=\linewidth]{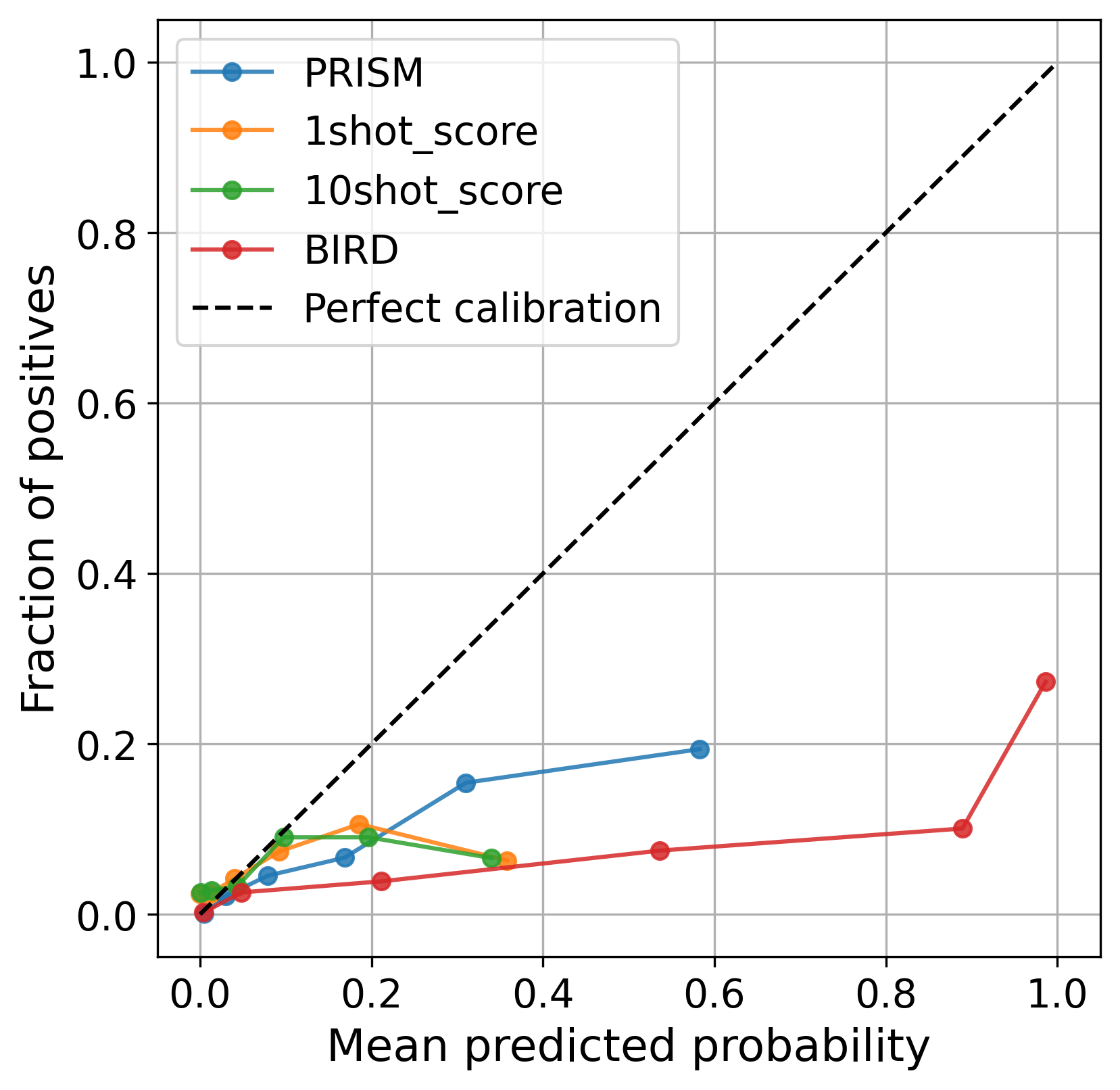}
        \caption{Stroke (Gemini)}
    \end{subfigure}
    
    \begin{subfigure}{0.35\textwidth}
        \centering
        \includegraphics[width=\linewidth]{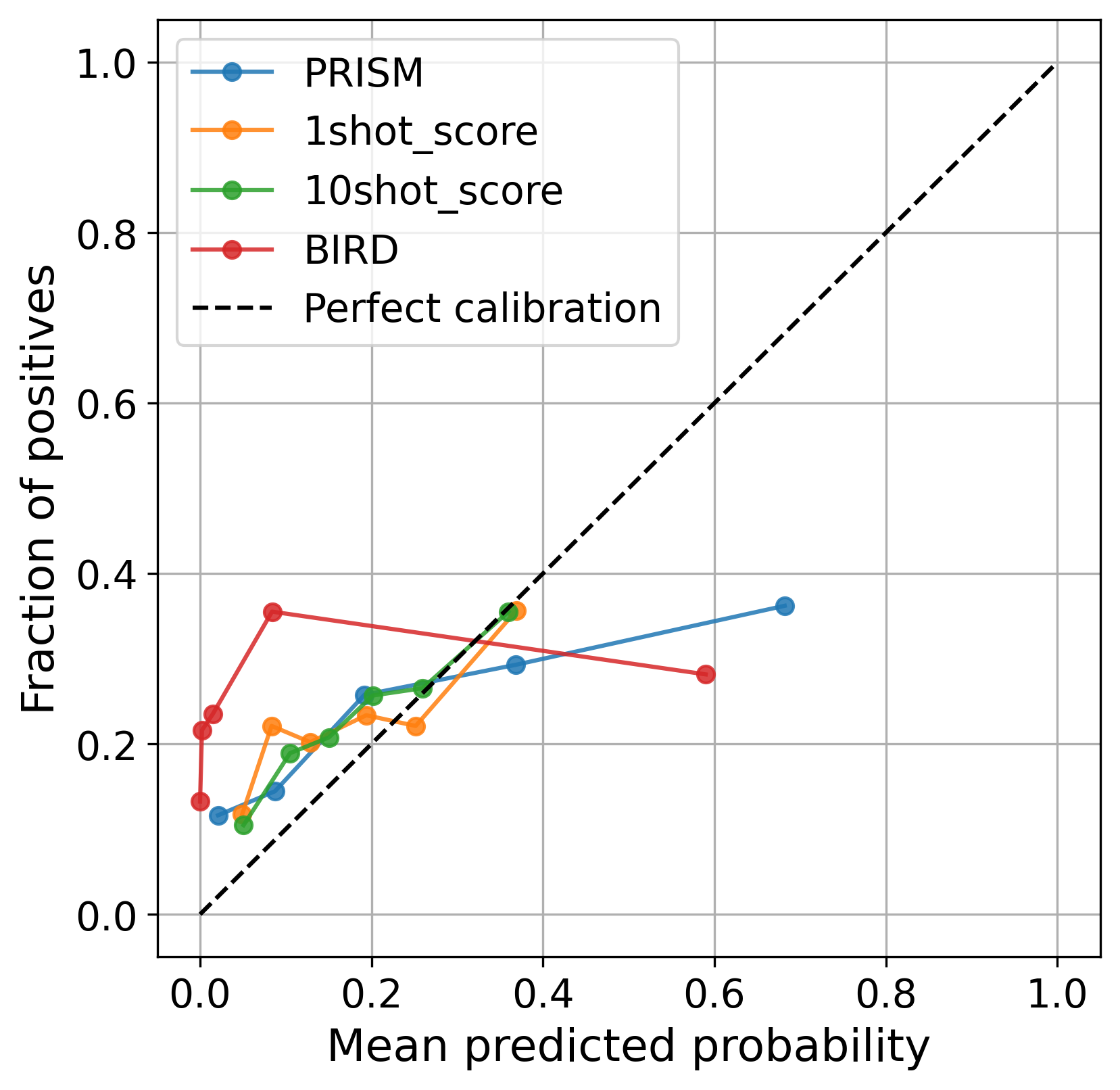}
        \caption{Loan (GPT)}
    \end{subfigure}
    \begin{subfigure}{0.35\textwidth}
        \centering
        \includegraphics[width=\linewidth]{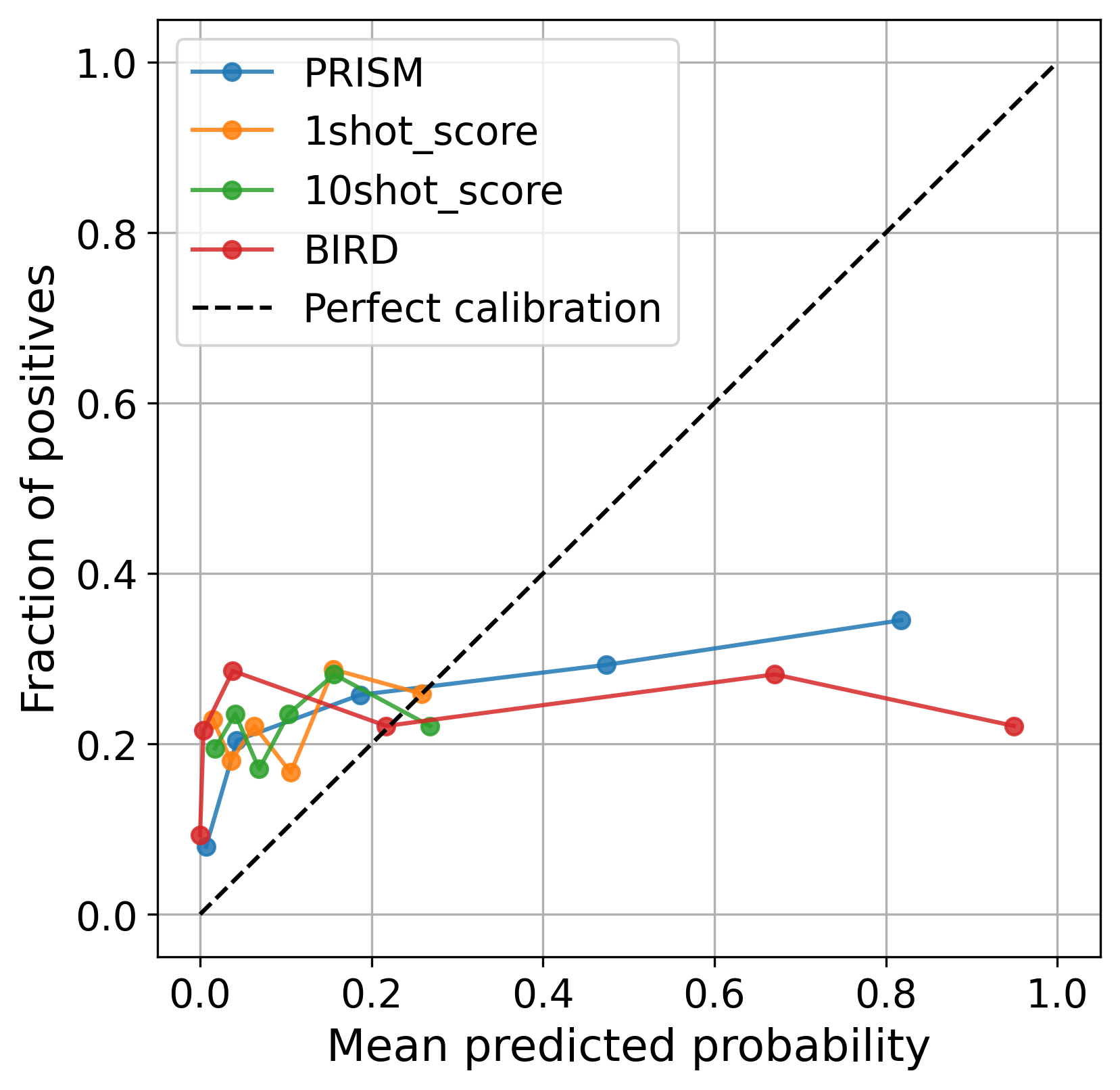}
        \caption{Loan (Gemini)}
    \end{subfigure}
    
    \caption{Calibration (reliability) curves comparing PRISM, 1shot\_score, 10shot\_score, and BIRD on four datasets under GPT-4.1-mini and Gemini-2.5-Pro.}
    \label{fig:eceplots}
\end{figure}

\textbf{Setup.}
We assess calibration by plotting the weighted reliability curve.
Given predictions $\{\hat p_i\}$ and labels $\{y_i\in\{0,1\}\}$, we bin by \emph{equal-count} quantiles of $\hat p$ and compute in-bin weighted means
\begin{equation}
    \hat{p}_m=\frac{1}{W_m}\sum_{i\in B_m} w_i\hat p_i,\quad
\hat{y}_m=\frac{1}{W_m}\sum_{i\in B_m} w_i y_i,\quad
W_m=\sum_{i\in B_m} w_i.
\end{equation}

Weights correct the gap between the \emph{evaluation} split and the \emph{deployment/population} class mix:
\begin{equation}
w_i=
\begin{cases}
\dfrac{\pi}{\hat\pi}, & y_i=1,\\[4pt]
\dfrac{1-\pi}{1-\hat\pi}, & y_i=0,
\end{cases}
\qquad
\mathrm{ECE}=\sum_{m=1}^M \frac{W_m}{\sum_j W_j}\,|\hat p_m-\hat y_m|.
\end{equation}
Here, $\hat\pi$ is the positive rate \emph{in our evaluation split}, which is class-balanced ($\hat\pi=0.5$). 
$\pi$ is the \emph{population} positive rate we aim to evaluate against. 
Since the true deployment prevalence is unknown, we use the \emph{original dataset prevalence before balancing} as a proxy for $\pi$:
\textit{Stroke} $4.87\%$, \textit{Adult Census} $24.88\%$, \textit{Heart Disease} $55.28\%$.
(If a practitioner knows their deployment prevalence, they can plug it in for $\pi$.) 
We visualize calibration by plotting $(\hat p_m,\hat y_m)$ against the identity line $y=x$.

\textbf{Why weighting matters.}
Our test splits are $1{:}1$ balanced, but real-world prevalence is typically skewed. 
Without weighting ($w_i\equiv 1$), bin positive fractions reflect the artificial $50\%$ mix rather than the population, biasing the curve upward on imbalanced tasks (e.g., \textit{Stroke}). 
Importance weights re-create the population mix \emph{within each bin}, so the reliability curve answers the practical question: “given this score in deployment, what fraction will be positive?” Therefore, Reweighting restores the population class mix in each bin, yielding reliability curves that reflect real–world deployment rather than the artificial test mix.

\textbf{Results.} Across all four datasets in Fig.~\ref{fig:eceplots}, \textit{PRISM} maintains strong calibration, remaining close to the $y{=}x$ line on \textit{Adult} and \textit{Heart}, and staying closest to the diagonal on \textit{Stroke} and \textit{Loan} despite small deviations, while \emph{consistently yielding a monotonically increasing reliability curve}. This monotonicity ensures that higher predicted probabilities always correspond to higher empirical event rates (no local reversals), preventing rank inconsistencies and threshold instability that appear in the baselines when the fraction of positives decreases as the mean predicted probability increases, and it further shows that our method is better calibrated than the baselines.

\section{Prompts}
\label{sec:prompts}
\begin{figure}[h]
  \centering
  \captionsetup[sub]{justification=centering}

  \begin{subfigure}[t]{\textwidth}
    \centering
    \includegraphics[width=\linewidth]{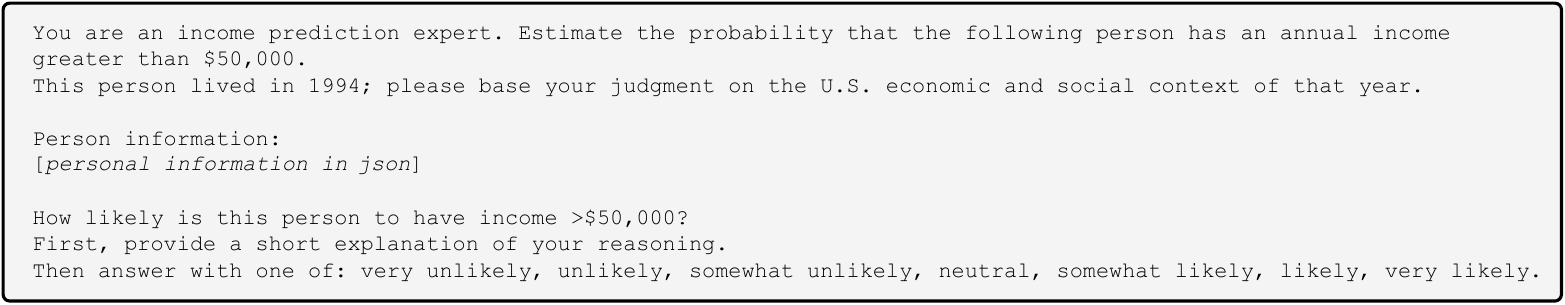}
    \caption{1/5/10shot\_level on Adult Census}
    \label{fig:prompt-nshot_level-adult}
  \end{subfigure}\hfill
  \begin{subfigure}[t]{\textwidth}
    \centering
    \includegraphics[width=\linewidth]{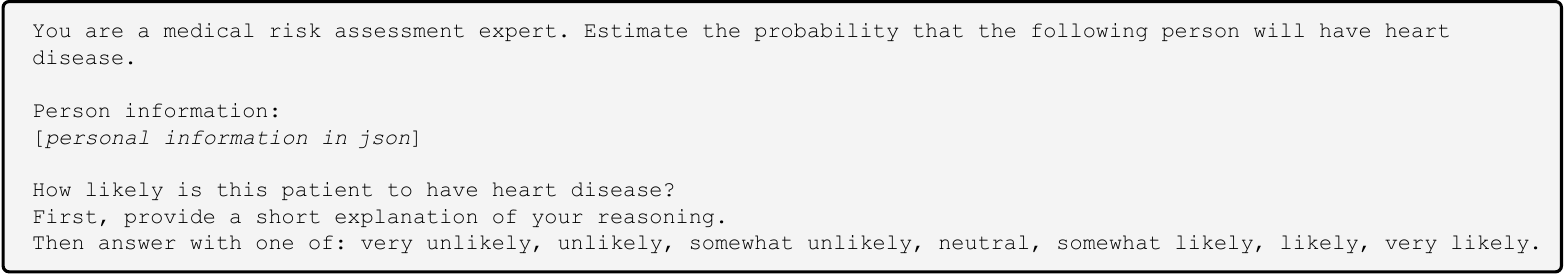}
    \caption{1/5/10shot\_level on Heart Disease dataset}
    \label{fig:prompt-nshot_level-heart}
  \end{subfigure}\hfill
  \begin{subfigure}[t]{\textwidth}
    \centering
    \includegraphics[width=\linewidth]{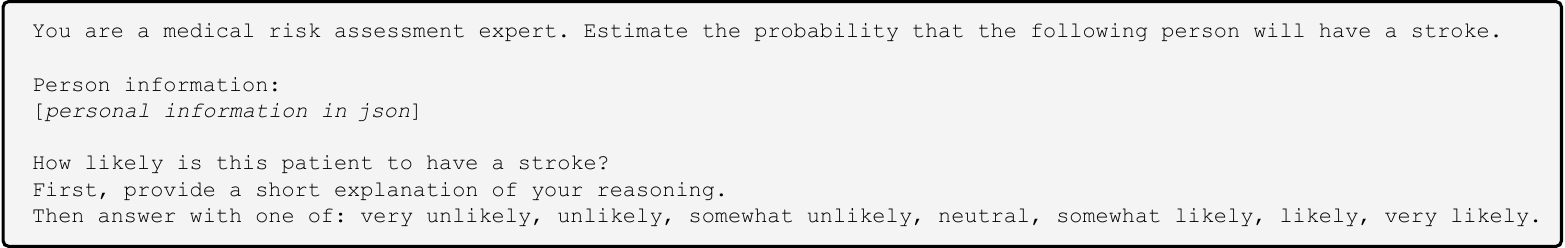}
    \caption{1/5/10shot\_level on Stroke dataset}
    \label{fig:prompt-nshot_level-stroke}
  \end{subfigure}\hfill
  \begin{subfigure}[t]{\textwidth}
    \centering
\includegraphics[width=\linewidth]{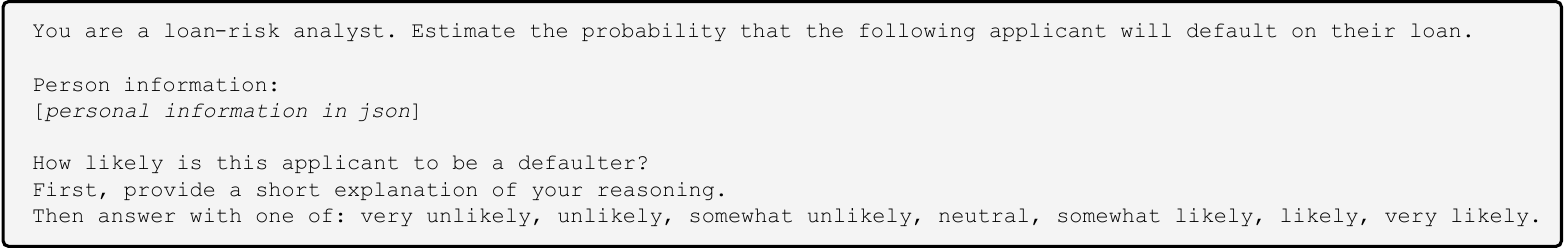}
    \caption{1/5/10shot\_level on Lending dataset}
    \label{fig:prompt-nshot_level-loan}
  \end{subfigure}

  \caption{Prompts for baseline method 1shot\_level, 5shot\_level, 10shot\_level. Where [personal information in json] is replaced by the actual data in json format.}
  \vspace{-0.3cm}
  \label{fig:prompt-nshot_level}
\end{figure}

\begin{figure}[h]
  \centering
  \captionsetup[sub]{justification=centering}

  \begin{subfigure}[t]{\textwidth}
    \centering
    \includegraphics[width=\linewidth]{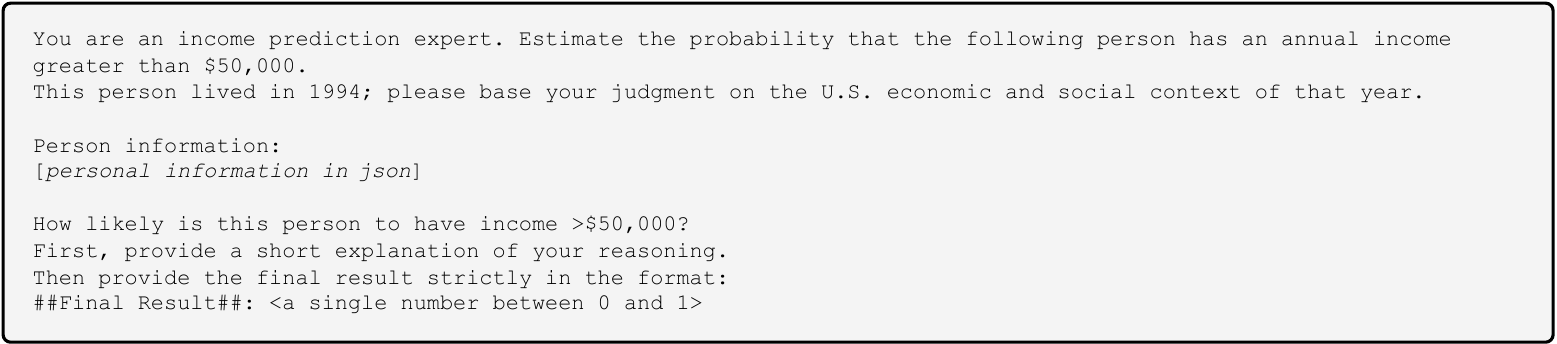}
    \caption{1/5/10shot\_score on Adult Census}
    \label{fig:prompt-nshot_score-adult}
  \end{subfigure}\hfill
  \begin{subfigure}[t]{\textwidth}
    \centering
    \includegraphics[width=\linewidth]{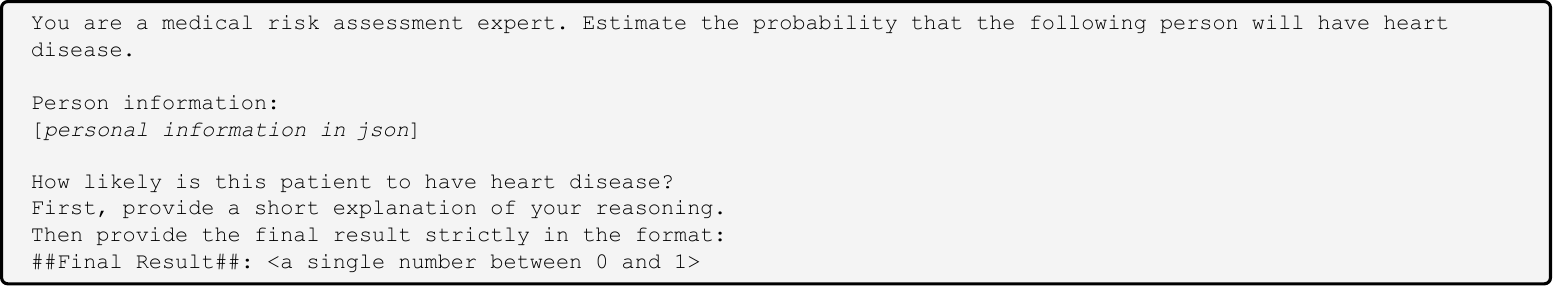}
    \caption{1/5/10shot\_score on Heart Disease dataset}
    \label{fig:prompt-nshot_score-heart}
  \end{subfigure}\hfill
  \begin{subfigure}[t]{\textwidth}
    \centering
    \includegraphics[width=\linewidth]{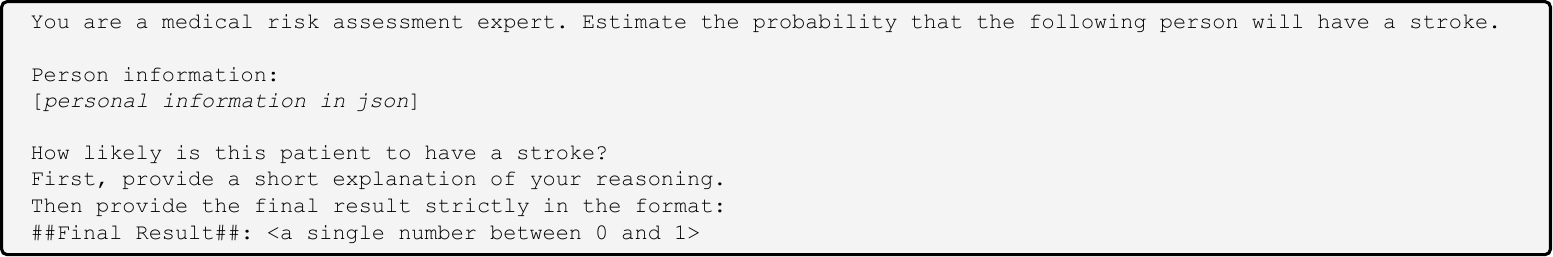}
    \caption{1/5/10shot\_score on Stroke dataset}
    \label{fig:prompt-nshot_score-stroke}
  \end{subfigure}\hfill
  \begin{subfigure}[t]{\textwidth}
    \centering
\includegraphics[width=\linewidth]{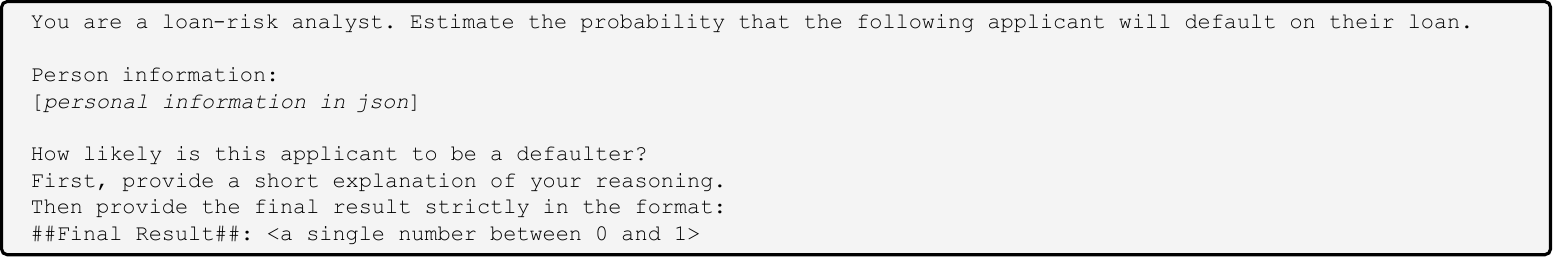}
    \caption{1/5/10shot\_score on Lending dataset}
    \label{fig:prompt-nshot_score-loan}
  \end{subfigure}

  \caption{Prompts for baseline method 1shot\_score, 5shot\_score, 10shot\_score. Where [personal information in json] is replaced by the actual data in json format, and the last line of each prompt is the format guideline.}
  \vspace{-0.3cm}
  \label{fig:prompt-nshot_score}
\end{figure}

\begin{figure}[h]
  \centering
  \captionsetup[sub]{justification=centering}

  \begin{subfigure}[t]{\textwidth}
    \centering
    \includegraphics[width=\linewidth]{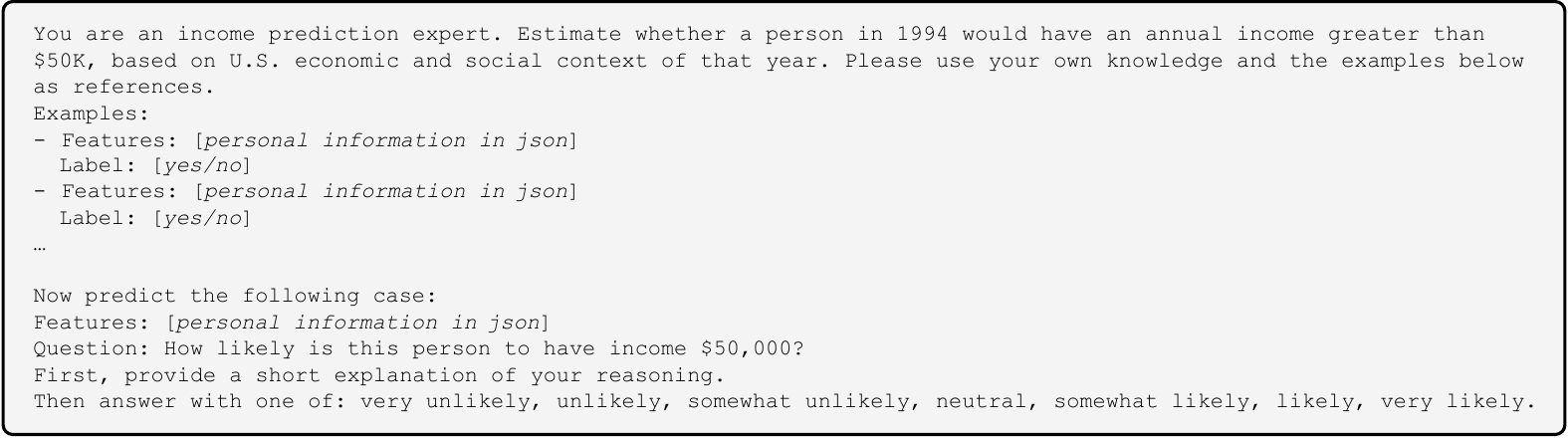}
    \caption{ICL on Adult Census}
    \label{fig:prompt-ICL-adult}
  \end{subfigure}\hfill
  \begin{subfigure}[t]{\textwidth}
    \centering
    \includegraphics[width=\linewidth]{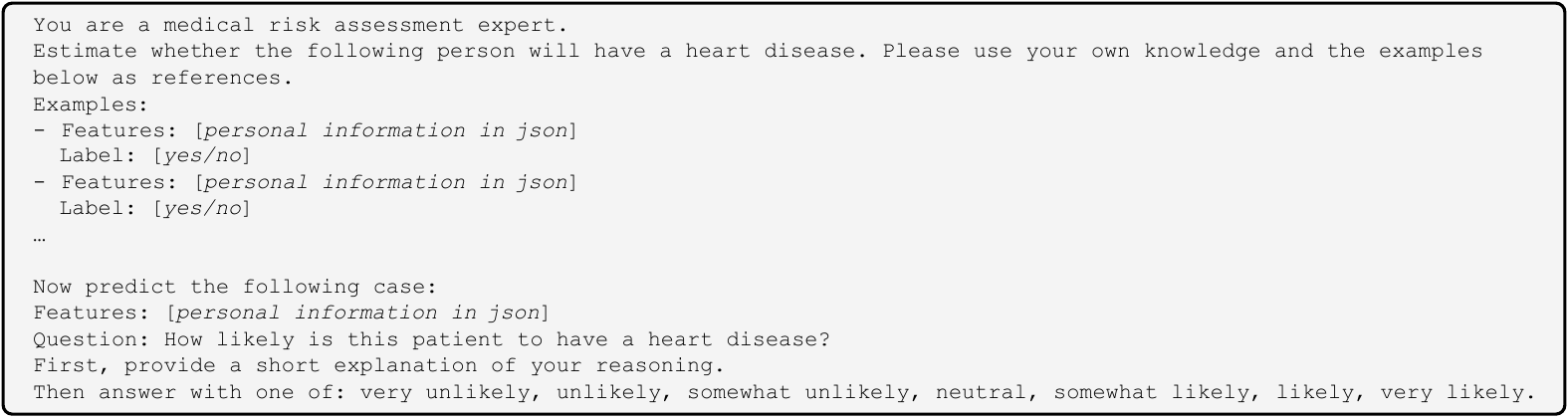}
    \caption{ICL on Heart Disease dataset}
    \label{fig:prompt-ICL-heart}
  \end{subfigure}\hfill
  \begin{subfigure}[t]{\textwidth}
    \centering
    \includegraphics[width=\linewidth]{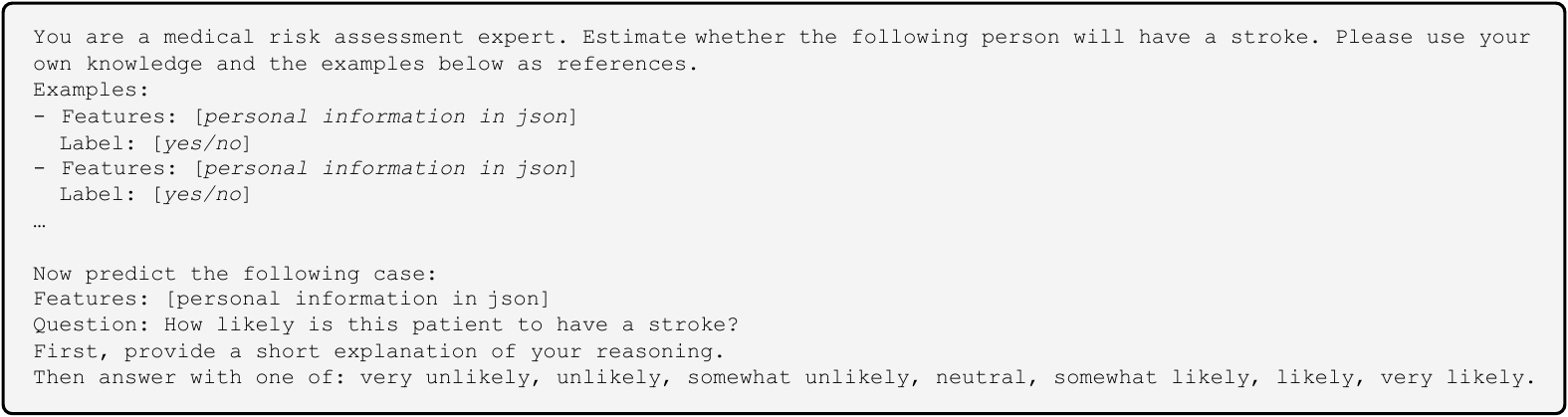}
    \caption{ICL on Stroke dataset}
    \label{fig:prompt-ICL-stroke}
  \end{subfigure}\hfill
  \begin{subfigure}[t]{\textwidth}
    \centering
\includegraphics[width=\linewidth]{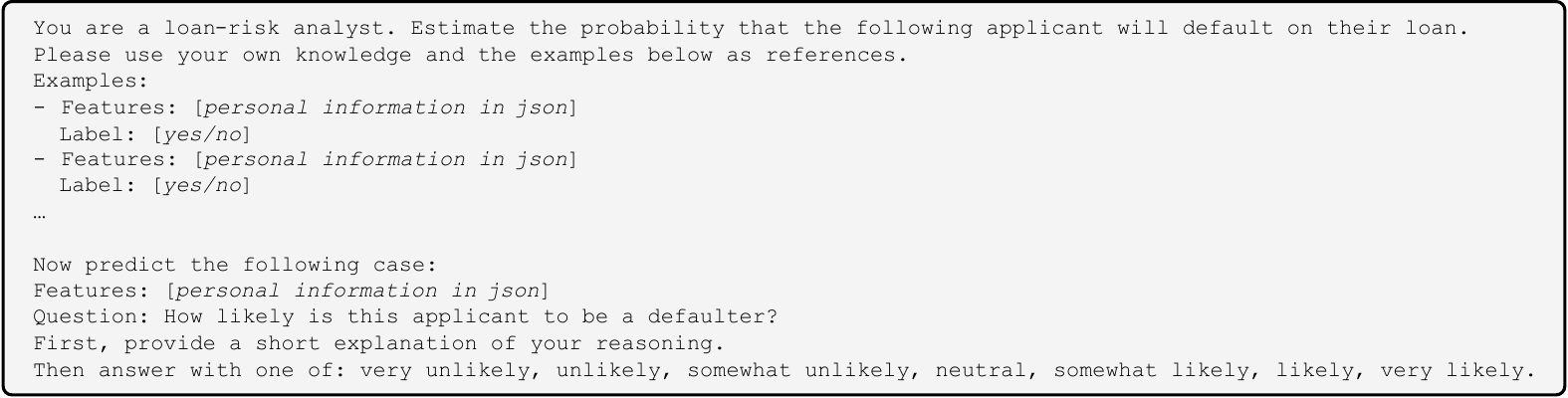}
    \caption{ICL on Lending dataset}
    \label{fig:prompt-ICL-loan}
  \end{subfigure}

  \caption{Prompts for baseline method ICL-5+5, ICL-10+10. Where [personal information in json] is replaced by the actual data in json format. For ICL-5+5, we have 5 positive and 5 negative training data randomly ordered as Features and Values in the prompt, and similar for ICL-10+10 except we have 10 positive and 10 negative training data.}
  \vspace{-0.3cm}
  \label{fig:prompt-ICL}
\end{figure}

\begin{figure}[h]
  \centering
  \captionsetup[sub]{justification=centering}

  \begin{subfigure}[t]{\textwidth}
    \centering
    \includegraphics[width=\linewidth]{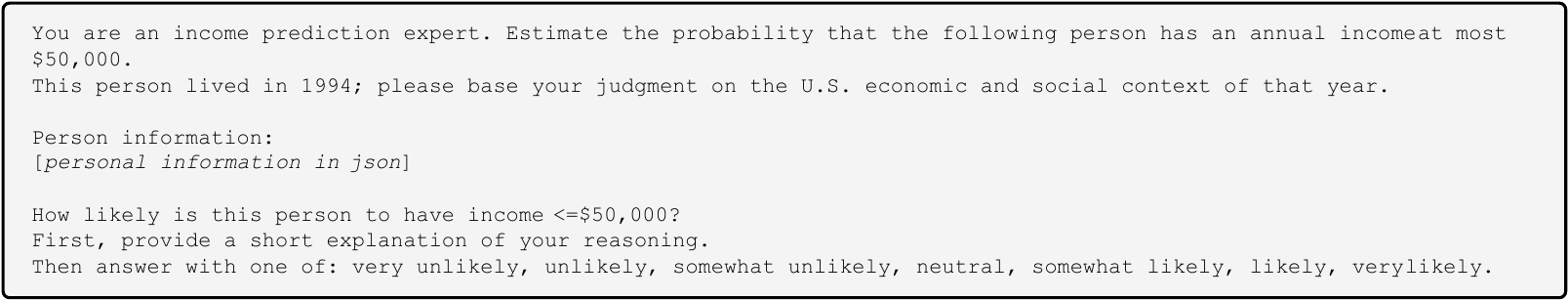}
    \caption{Contrast on Adult Census}
    \label{fig:prompt-contrast-adult}
  \end{subfigure}\hfill
  \begin{subfigure}[t]{\textwidth}
    \centering
    \includegraphics[width=\linewidth]{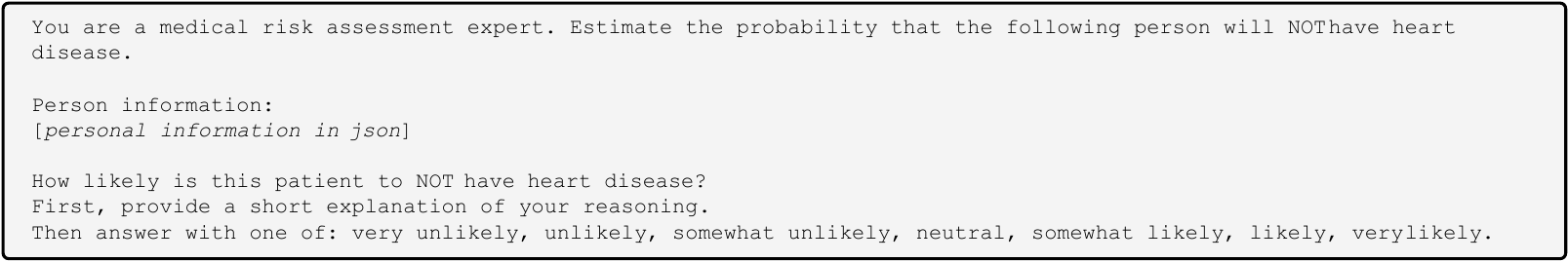}
    \caption{Contrast on Heart Disease dataset}
    \label{fig:prompt-contrast-heart}
  \end{subfigure}\hfill
  \begin{subfigure}[t]{\textwidth}
    \centering
    \includegraphics[width=\linewidth]{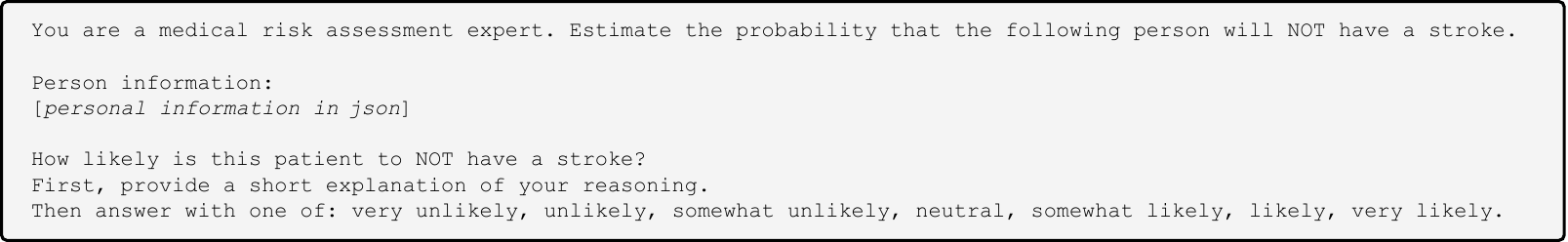}
    \caption{Contrast on Stroke dataset}
    \label{fig:prompt-contrast-stroke}
  \end{subfigure}\hfill
  \begin{subfigure}[t]{\textwidth}
    \centering
\includegraphics[width=\linewidth]{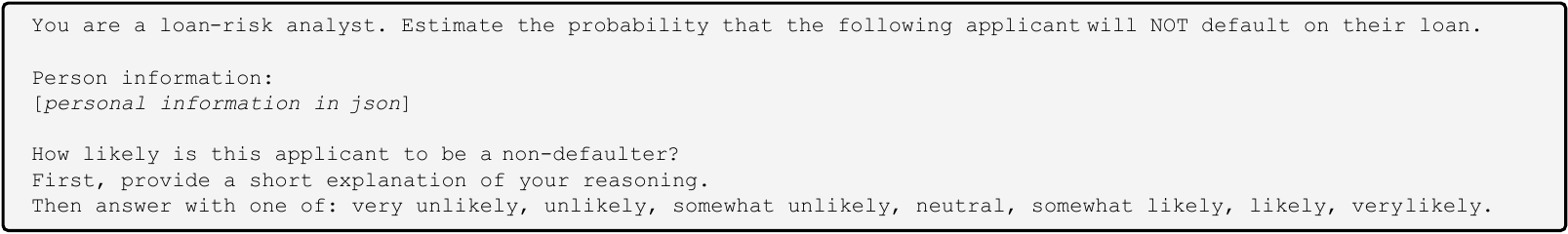}
    \caption{Contrast on Lending dataset}
    \label{fig:prompt-contrast-loan}
  \end{subfigure}

  \caption{Negative prompts for baseline method Contrast. Where [personal information in json] is replaced by the actual data in json format. The positive prompt is the same as in 1shot\_label prompt\ref{fig:prompt-nshot_level}.}
  \vspace{-0.3cm}
  \label{fig:prompt-contrast}
\end{figure}

\begin{figure}[h]
  \centering
  \captionsetup[sub]{justification=centering}

  \begin{subfigure}[t]{\textwidth}
    \centering
    \includegraphics[width=\linewidth]{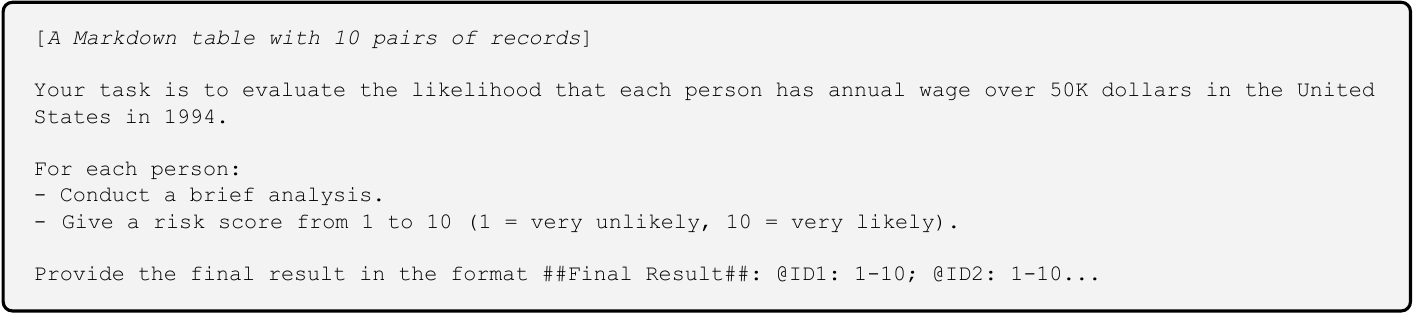}
    \caption{Tabular-PRISM on Adult Census}
    \label{fig:prompt-prism-adult}
  \end{subfigure}\hfill
  \begin{subfigure}[t]{\textwidth}
    \centering
    \includegraphics[width=\linewidth]{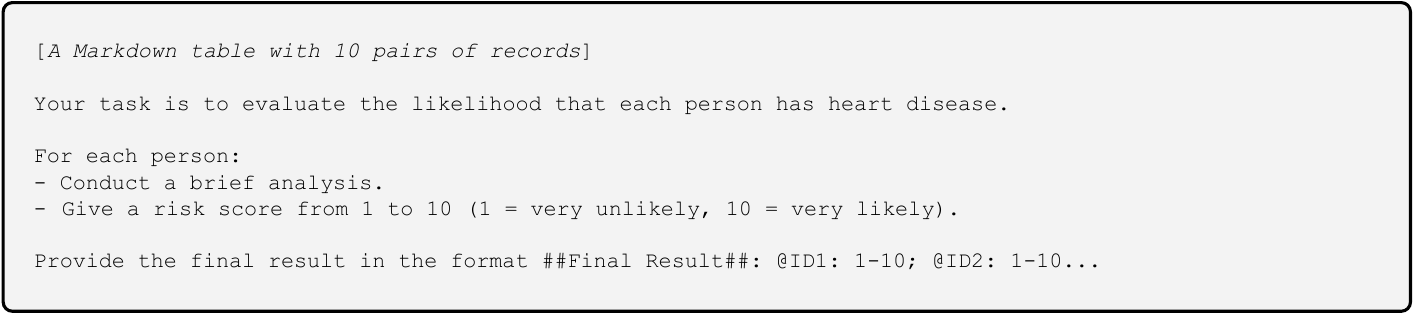}
    \caption{Tabular-PRISM on Heart Disease dataset}
    \label{fig:prompt-prism-heart}
  \end{subfigure}\hfill
  \begin{subfigure}[t]{\textwidth}
    \centering
    \includegraphics[width=\linewidth]{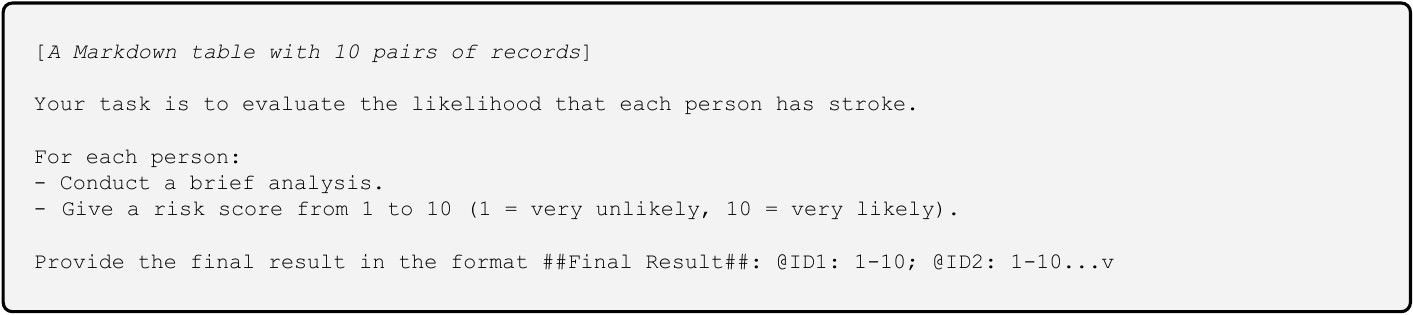}
    \caption{Tabular-PRISM on Stroke dataset}
    \label{fig:prompt-prism-stroke}
  \end{subfigure}\hfill
  \begin{subfigure}[t]{\textwidth}
    \centering
\includegraphics[width=\linewidth]{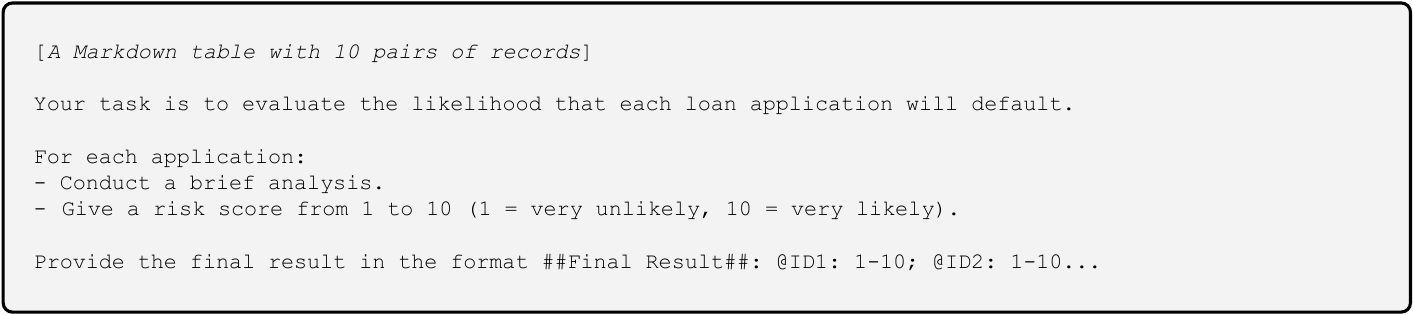}
    \caption{Tabular-PRISM on Lending dataset}
    \label{fig:prompt-prism-loan}
  \end{subfigure}

  \caption{Prompt templates for Tabular\textendash PRISM across four datasets. Each panel shows the standardized instruction used to score instances from a Markdown table with 10 pairs of records; each pair comprises a target case and its matched baseline case.}

  \vspace{-0.3cm}
  \label{fig:prompt-prism-tabular}
\end{figure}

\begin{figure}[h]
  \centering
  \captionsetup[sub]{justification=centering}

  \begin{subfigure}[t]{\textwidth}
    \centering
    \includegraphics[width=\linewidth]{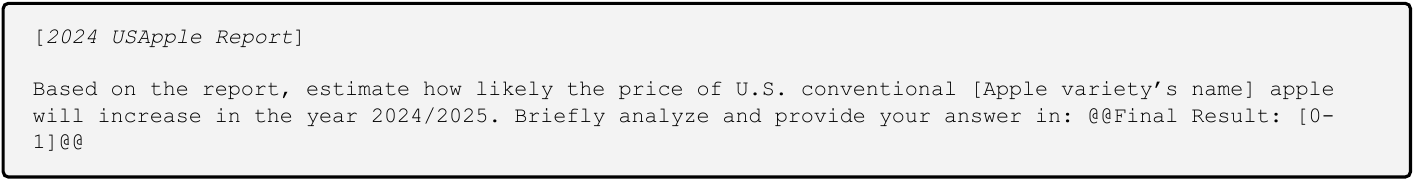}
    \caption{Raw}
    \label{fig:prompt-apple-raw}
  \end{subfigure}\hfill

  \begin{subfigure}[t]{\textwidth}
    \centering
    \includegraphics[width=\linewidth]{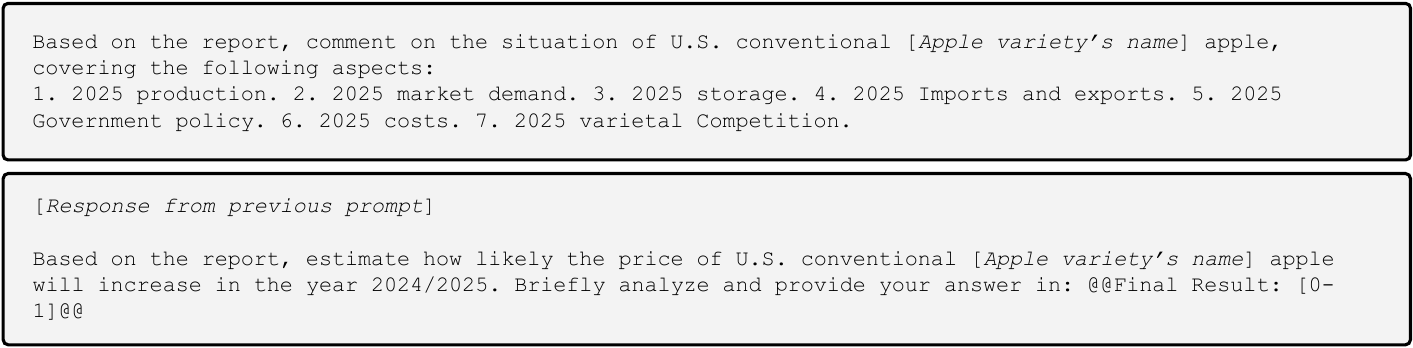}
    \caption{Extracted}
    \label{fig:prompt-apple-extracted}
  \end{subfigure}\hfill

  \begin{subfigure}[t]{\textwidth}
    \centering
    \includegraphics[width=\linewidth]{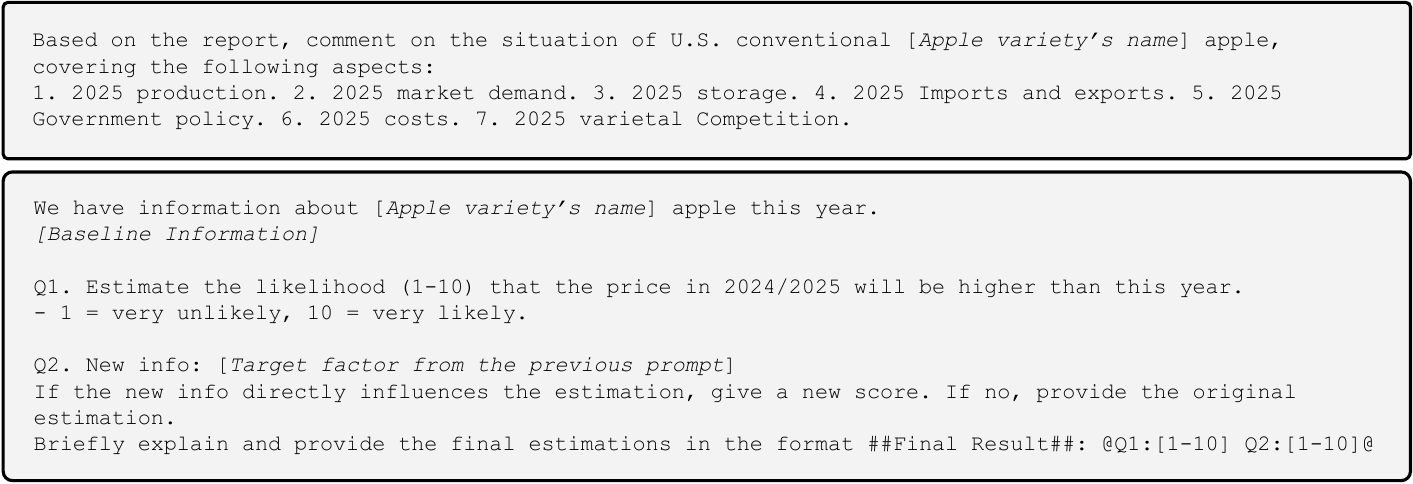}
    \caption{PRISM}
    \label{fig:prompt-apple-prism}
  \end{subfigure}

  \caption{Apple prompts of three experiments: \emph{Raw} input, \emph{Extracted} factors, and \emph{PRISM}.}
  \vspace{-0.3cm}
  \label{fig:prompt-apple}
\end{figure}

\begin{figure}[h]
  \centering
  \captionsetup[sub]{justification=centering}

  \begin{subfigure}[t]{\textwidth}
    \centering
    \includegraphics[width=\linewidth]{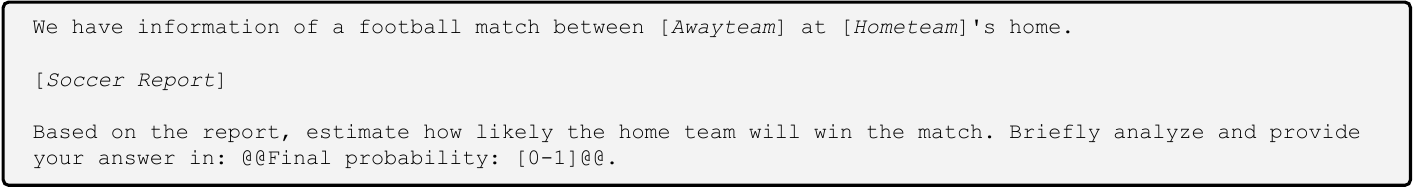}
    \caption{Raw}
    \label{fig:prompt-soccer-raw}
  \end{subfigure}\hfill

  \begin{subfigure}[t]{\textwidth}
    \centering
    \includegraphics[width=\linewidth]{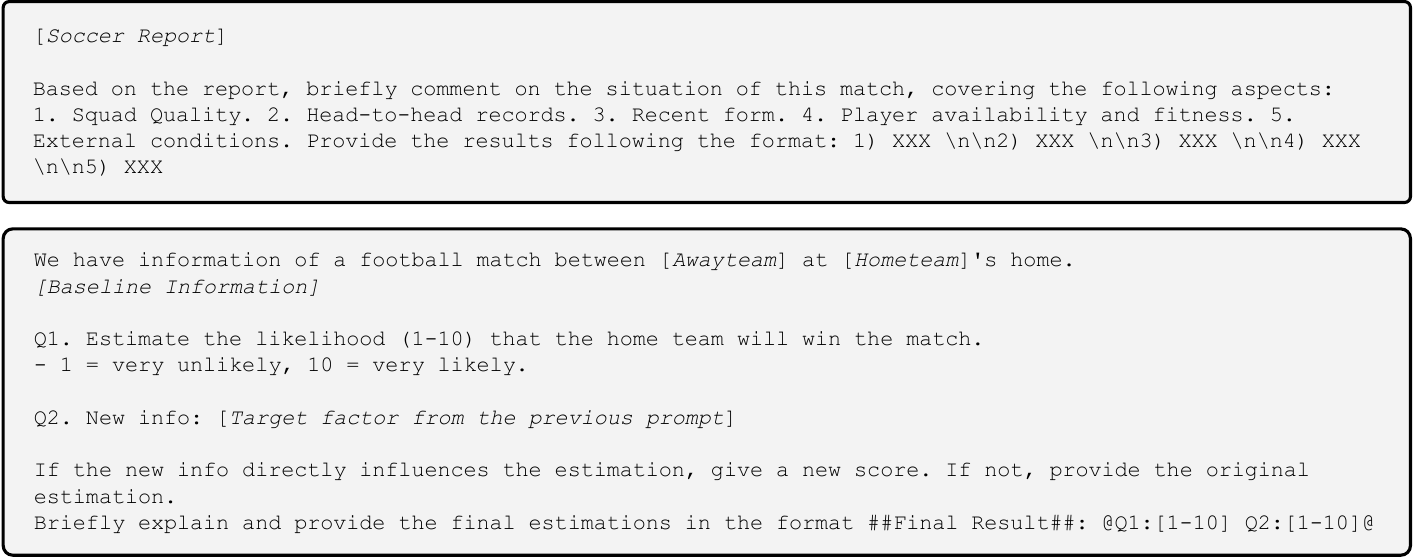}
    \caption{PRISM}
    \label{fig:prompt-soccer-PRISM}
  \end{subfigure}
  \vspace{-0.3cm}
  \caption{Soccer prompts of two experiments: \emph{Raw} input and \emph{PRISM}.}
  \label{fig:prompt-soccer}
\end{figure}
\end{document}